\documentclass[journal]{IEEEtran}
\usepackage{amsmath,amsfonts,amssymb}
\usepackage{amsthm}
\usepackage{mathtools}
\allowdisplaybreaks
\usepackage{algorithm}
\usepackage{algpseudocode}
\usepackage{array}
\usepackage[caption=false,font=normalsize,labelfont=sf,textfont=sf]{subfig}
\usepackage{textcomp}
\usepackage{stfloats}
\usepackage{url}
\usepackage{verbatim}
\usepackage{graphicx}
\usepackage{cite}
\usepackage{epstopdf}

\newcommand{\defeq}{\vcentcolon=}

\newcommand\numberthis{\addtocounter{equation}{1}\tag{\theequation}}

\usepackage{nicematrix}

\newtheorem{remark}{Remark}

\newtheorem{theorem}{Theorem}
\newtheorem{corollary}{Corollary}[theorem]
\newtheorem{lemma}{Lemma}

\newtheorem{assumption}{Assumption}
\newtheorem{example}{Example}

\usepackage{booktabs}
\usepackage{multicol}
 \usepackage{multirow}
\usepackage[table]{xcolor}
\usepackage{makecell}
\usepackage{transparent}

\begin{document}

\title{Heterogeneity-aware Client Sampling: A Unified Solution for Consistent Federated Learning}

\author{Shudi Weng,~\IEEEmembership{Graduate Student Member,~IEEE,}
Chao Ren,~\IEEEmembership{Member,~IEEE,}\\
Ming Xiao,~\IEEEmembership{Senior Member,~IEEE,}
and Mikael Skoglund,~\IEEEmembership{Fellow,~IEEE}\vspace{-1em}
\thanks{Shudi Weng, Chao Ren, Ming Xiao, and Mikael Skoglund are with the Division of Information Science and Engineering, School of Electrical Engineering and Computer Science (EECS), KTH Royal Institute of Technology, 11428 Stockholm, Sweden, Email: \{shudiw, chaor, mingx, skoglund\}@kth.se.

\textit{Corresponding Author: Shudi Weng}.
}
}
\markboth{Journal of \LaTeX\ Class Files,~Vol.~14, No.~8, August~2021}%
{Shell \MakeLowercase{\textit{et al.}}: A Sample Article Using IEEEtran.cls for IEEE Journals}

\IEEEpubid{0000--0000/00\$00.00~\copyright~2021 IEEE}

\maketitle

\begin{abstract}
Federated learning (FL) has emerged as a promising paradigm for edge learning. In such scenarios, clients typically exhibit heterogeneous communication and computational capabilities. 
Such heterogeneity can significantly distort the optimization dynamics and lead to objective inconsistency, where the global model converges to an incorrect stationary point, potentially far from the global optimum.
Despite its critical impact, the joint impact of communication and computation heterogeneity remains unexplored. 
In this paper, we reveal the fundamentally distinct mechanisms through which heterogeneous communication and computation jointly drive inconsistency in FL and conduct a unified theoretical analysis of general heterogeneous FL, offering a principled understanding of how these two forms of heterogeneity jointly distort the optimization trajectory under arbitrary choices of local solvers. 
{Motivated by these insights, firstly, we investigate the consistency-preserving joint communication and computation design principle that ensures correct convergence of the global model;} secondly, we propose \underline{\textbf{Fed}}erated Heterogeneity-\underline{\textbf{A}}ware \underline{\textbf{C}}lient \underline{\textbf{S}}ampling (\texttt{FedACS}), a universal FL framework that can tackle optimization inconsistency arising system heterogeneity. 
The theoretical results show that FedACS is guaranteed to converge towards the correct optimum with a linear rate, even in dynamically heterogeneous environments. 
Extensive experiments across multiple datasets show that the proposed \texttt{FedACS} outperforms state-of-the-art and category-specific accuracy baselines by $4.3\%$-$36\%$, while saving computation resources by $14\%$-$105\%$ to achieve comparable performance.  
\end{abstract}


\section{Introduction}
\subsection{Backgrounds}
Federated learning (FL) is an emerging subfield in distributed optimization. 
By prohibiting raw data sharing, FL alleviates communication bottlenecks and enhances privacy protection \cite{9770266,li2020federatedtsp}.  However, its performance heavily depends on \textit{i)} \textit{statistical heterogeneity}, which refers to as the non-independent and identically distributed (non-IID) data across clients, and \textit{ii)} \textit{system heterogeneity}, which refers to as diverse communication and computation capabilities among clients, as shown in Fig. \ref{fig:heter_comm_comp}, arising from physical factors, e.g., resource constraints and hardware processing speed \cite{wangTackling, 9796935}. 
The statistical heterogeneity, can cause \textit{client drift}, where local models deviate significantly from the global model in each round. The client drift can slow and destabilize the convergence of the global model because the local bias leads to increased variance in the estimated global model per round \cite{zhang2021fedpd}. 
On top of this, system heterogeneity may further intensify the deviation of the global model computation, leading to severe suboptimality in the training outcome and, consequently, significant performance degradation.
\subsection{Related Works}
\subsubsection{System Heterogeneity}
In practice, edge devices often exhibit \textit{heterogeneous computation} capabilities. Heterogeneity in local dataset sizes and hardware capabilities can result in varying numbers of local iterations across clients in each training round. Such variation leads to unequal accumulation of local model updates, thereby causing a discrepancy between the aggregated global model and the ground-truth global model.
This phenomenon, termed objective inconsistency, was first identified and concretely analyzed in \cite{wangTackling,9521822}.
To mitigate objective inconsistency arising from computational heterogeneity, \cite{wangTackling,9521822} further introduced \texttt{FedNova}, which normalizes local updates prior to aggregation to achieve update alignment within each training round. Building upon this intuition, \texttt{FedAU} \cite{wang2024a} and \texttt{FedAWE} \cite{xiang2024efficient} extend the alignment from individual rounds to the entire training process by leveraging historical client participation and update statistics.


Apart from this, the \textit{communication} between edge devices and the parameter server (PS) may also be unreliable and heterogeneous due to physical factors, e.g., fading, shadowing, and the random geometrical distribution of edge devices. 
Due to the unpredictable outcomes of communication links, the successful delivery of local models varies across clients and training rounds, causing both their delivery frequencies and actual aggregation weights to deviate, in expectation, from their intended values. Such deviations can render the resulting optimization process strictly suboptimal \cite{wang2021quantized, ye2022decentralized}. 
To mitigate this issue, \cite{wang2021quantized} explores strategic communication-resource allocation to improve update delivery. Alternatively, \cite{zheng2023federated, perazzone2022communication} develop unbiased aggregation schemes to compensate for missing updates, while \cite{yemini2023robust,10705313} enhance communication reliability through inter-user cooperation. {Another class of approaches \cite{xiang2024efficient,wang2024a,ying2026exact} track historical information from clients' previous participation to mitigate the bias induced by missing updates and preserve convergence to the desired optimum.} Coding-based strategies have also been proposed to reduce the occurrence of missing updates \cite{weng2024cooperative,11175173,10802992}, albeit at the cost of additional communication overhead. Other approaches, such as retransmission \cite{zang2023general} and reusing the latest available local updates \cite{jhunjhunwala2022fedvarp}, may prolong the overall training time or inadvertently overrepresent certain clients' data, respectively.

Despite the rich body of literature addressing communication or computation heterogeneity in FL, a unified analysis that jointly characterizes their interplay is still lacking.

\subsubsection{Communication-Computation Co-Design}
Most existing studies on communication-computation co-design in FL focus on time-domain optimal scheduling of communication and computation. For instance, \cite{9475121} propose to optimize communication and computation parameters jointly to minimize the overall energy consumption, {while \cite{jiang2026non} develops cost-aware client-selection strategies to improve communication and local computation efficiency.}
However, the time-domain system heterogeneity addressed in these works is orthogonal to the dimension considered in this work and, by itself, is insufficient to guarantee optimal learning performance.

A valid theoretical foundation for communication-computation co-design beyond time-domain scheduling is still lacking.
\subsubsection{Client Sampling Strategies}

Client sampling was first introduced to alleviate communication bottlenecks while preserving unbiased optimization \cite{Li2020}, with early approaches including importance sampling (IS) and uniform sampling (US). 
Subsequently, optimal sampling (OS) was proposed to accelerate global model convergence by assigning sampling weights proportional to gradient norms \cite{chen2020optimal}, thereby reducing model variance induced by data heterogeneity. 
However, OS may favor highly similar large-norm gradients, resulting in insufficient sampling diversity. 
To address this limitation, DELTA \cite{wang2023delta} incorporates gradient diversity into client selection. Alternatively, clustered client sampling groups and selects clients according to their data distributions or model diversity, effectively mitigating heterogeneity and improving convergence \cite{balakrishnan2022diverse, fraboni2021clustered}. This idea is further extended by heterogeneity-guided sampling \cite{chen_heterogeneity-guided_2024}, which selects client clusters based on data heterogeneity estimated via Shannon entropy. Beyond mitigating diversity, OS has also been extended to optimal scheduling to reduce the overall training time while preserving unbiasedness in heterogeneous FL \cite{perazzone2022communication,9796935}.

Despite these extensive efforts to ensure unbiased client sampling, whether and how biased sampling can benefit FL remains largely unexplored.

\subsection{Our Contributions}
The significance of our work is twofold. On the one hand, heterogeneous computation manifests itself through communication, leading to intricate interplay rather than a simple additive effect.
However, no prior work investigates this fundamental coupling. A unified theoretical framework not only deepens the understanding of these fundamental issues but can also be seamlessly integrated into existing popular frameworks to enhance their robustness in general heterogeneous environments.
On the other hand, the mechanisms through which computational and communication heterogeneity induce objective inconsistency fundamentally differ, rendering existing approaches tailored to either form insufficient to tackle both. A general framework that can jointly accommodate both forms of heterogeneity is critically needed.
Our main contributions are summarized as follows. 
\begin{itemize}
    \item \textit{Unified Analysis of General Heterogeneous FL:} 
    We provide the first theoretical analysis of FL under jointly heterogeneous communication and computation beyond the time domain. Our results reveal the fundamentally distinct mechanisms through which these two forms of heterogeneity induce objective inconsistency. Under an arbitrary choice of local solver, we rigorously characterize the convergence behavior with respect to both the surrogate and true global objectives. Specifically, we establish convergence guarantees for the surrogate objective in non-convex FL settings and derive an achievable convergence bound for the true global objective. Together, these results provide a comprehensive characterization of the distortion induced by system heterogeneity and lay a theoretical foundation for robust FL under general heterogeneous conditions.

    \item \textit{{Consistent Communication-Computation Co-Design:}}
    {We quantify the non-diminishing optimization bias induced by jointly heterogeneous communication and computation, from which we derive a communication-computation consistency condition that enables the global model to consistently converge to the pursued global optimum and ensures the optimality of the training process. 
    The theoretical results provide a theoretical foundation for future practical system-level designs that jointly account for objective consistency and system constraints.}

    \item \textit{Universal framework to Avoid Objective Inconsistency:} We propose \texttt{FedACS}, a unified framework designed to ensure optimality in heterogeneous FL with partial participation. The core idea is to employ a biased sampling strategy to counteract the bias introduced by system heterogeneity. \texttt{FedACS} can effectively address all types of objective inconsistency and ensure efficient convergence even in dynamic heterogeneous settings. 
    Extensive numerical experiments demonstrate its strong feasibility and substantial performance gains over state-of-the-art benchmarks.  
\end{itemize}
Beyond the listed main contributions, we examine the relation among the widely adopted bounded dissimilarity assumptions in non-convex FL analysis, providing new insights into its applicability and limitations. 


\section{System Model}\label{sec: system model}
In FL, a total of $M$ clients aim to jointly solve the following optimization problem: 
\begin{align}
    \min_{\boldsymbol{X}\in \mathbb{R}^d} \left[ F(\boldsymbol{X})\defeq \sum_{m=1}^M \omega_m F_m(\boldsymbol{X}) \right].
    \label{eq: goal}
\end{align}
In \eqref{eq: goal}, the global objective function (GOF) $F(\cdot)$ is defined by a weighted sum of local objective functions (LOFs), which is defined by the averaged loss $\mathcal{L}(\cdot)$ evaluated on the entire local dataset $\mathcal{D}_m$, i.e., $F_m(\boldsymbol{X})=\frac{1}{\lvert \mathcal{D}_m \rvert}\sum_{\xi\in \mathcal{D}_m}\mathcal{L}(\boldsymbol{X}\vert \xi)$. The importance of clients $\{\omega_m\}_{m=1}^{M}$ is typically determined by the relative size of local datasets.

\subsubsection*{Client Sampling} \label{sec: Client Sampling} 
Given parameters $\boldsymbol{p}\triangleq[p_1, p_2, \cdots, p_M]\in (0,1)^{M\time 1}$ with $\sum_{m=1}^M p_m=1$, {PS samples $K$ clients with replacement to form a random set $\mathcal{S}_r(\boldsymbol{p})\subseteq[M]$ at the $r$-th training round, i.e., $\lvert\mathcal{S}_r(\boldsymbol{p})\rvert=K$.} Each client $m$ is included in $\mathcal{S}_r(\boldsymbol{p})$ with probability $p_m$ during each sampling. A client may appear multiple times in $\mathcal{S}_r(\boldsymbol{p})$, and the aggregation weight for each client $m$ corresponds to the frequency of its appearance in $\mathcal{S}_r(\boldsymbol{p})$. This is also called \textit{independent sampling}, as the event $\{k\in \mathcal{S}_r(\boldsymbol{p})\}$ is independent of the event $\{m\in \mathcal{S}_r(\boldsymbol{p})\}$ for $\forall k\neq m$.

\subsubsection*{Generalized Local Update} 
At the $r$-th training round, the selected client $m\in \mathcal{S}_r(\boldsymbol{p})$ independently runs $T_m$ iterations of local solvers, initialized with the current global model $\boldsymbol{X}_{r}$. For an arbitrary choice of the local solver, the cumulative stochastic gradient resulting from the local training can be summarized to
\begin{align}
    \boldsymbol{\Delta}_{m}^{r,\boldsymbol{\xi}}= \nabla\boldsymbol{F}_m^{r,\boldsymbol{\xi}} \cdot\boldsymbol{a}_m,  
    \label{eq: local_accum}
\end{align}
where $\nabla\boldsymbol{F}_m^{r,\boldsymbol{\xi}}=[\nabla\boldsymbol{F}_m(\boldsymbol{X}_{m,r}^{0}\vert \xi_{m,r}^{1}), \cdots, \nabla\boldsymbol{F}_m(\boldsymbol{X}_{m,r}^{ T_m-1}\vert \xi_{m,r}^{ T_m})]$ stores all local stochastic gradients computed at iteration $t\in [T_m]$ and $\boldsymbol{X}_{m,r}^{0}=\boldsymbol{X}_r$. The $\boldsymbol{a}_m$ characterizes how the stochastic gradients are accumulated by different optimizers in $T_m$ local training iterations. The forms of $\boldsymbol{a}_m$ for the commonly used optimizers are given in  \cite{wangTackling}. Similarly, we define $\nabla\boldsymbol{F}_m^{r}=[\nabla\boldsymbol{F}_m(\boldsymbol{X}_{m,r}^{0}), \cdots, \nabla\boldsymbol{F}_m(\boldsymbol{X}_{m,r}^{ T_m-1})]$ and that $\boldsymbol{\boldsymbol{\Delta}}_{m}^r= \nabla\boldsymbol{F}_m^{r} \cdot\boldsymbol{a}_m$ as cumulative local gradients. Consequently, each local model update is $-\eta \boldsymbol{\Delta}_{m}^{r,\boldsymbol{\xi}}$, where $\eta$ is the learning rate.

\subsubsection*{Unreliable Communication}
The selected client $m\in \mathcal{S}_r(\boldsymbol{p})$ sends its local model update $-\eta \boldsymbol{\Delta}_{m}^{r,\boldsymbol{\xi}}$ to PS over an intermittent communication link. The link is modeled as a Bernoulli r.v. $Z_m^r\sim \mathrm{Ber}(1-q_m)$, where $q_m$ is failure probability.  Under this model, $Z_m^r=1$ signifies a successful transmission, allowing PS to receive $-\eta \boldsymbol{\Delta}_{m}^{r,\boldsymbol{\xi}}$ perfectly, while $Z_m^r=0$ indicates a complete communication failure. Furthermore, the transmissions are assumed to be orthogonal and independent. 

\subsubsection*{Anonymous  Aggregation}
Through unreliable communication, PS receives local model updates from the clients in $\Tilde{\mathcal{S}}_r(\boldsymbol{p},\boldsymbol{q})$, where $\Tilde{\mathcal{S}}_r(\boldsymbol{p},\boldsymbol{q})=\{m: m\in \mathcal{S}_r(\boldsymbol{p}), Z_m^r=1 \}$. 
The aggregation rule is given by
\begin{align}
    \boldsymbol{X}_{r+1}=\boldsymbol{X}_{r} +\frac{1}{K} \sum_{m\in \Tilde{\mathcal{S}}_r(\boldsymbol{p},\boldsymbol{q})} (-\eta\boldsymbol{\Delta}_{m}^{r,\boldsymbol{\xi}}).
    \label{eq: agg_rule1}
\end{align}
Unlike \cite{zang2023general,zheng2023federated}, \eqref{eq: agg_rule1} does not count or identify clients, making it particularly suited for over-the-air (OTA) FL and privacy-preserving FL framework. 

\begin{figure}
    \centering
    \includegraphics[width=0.8\linewidth]{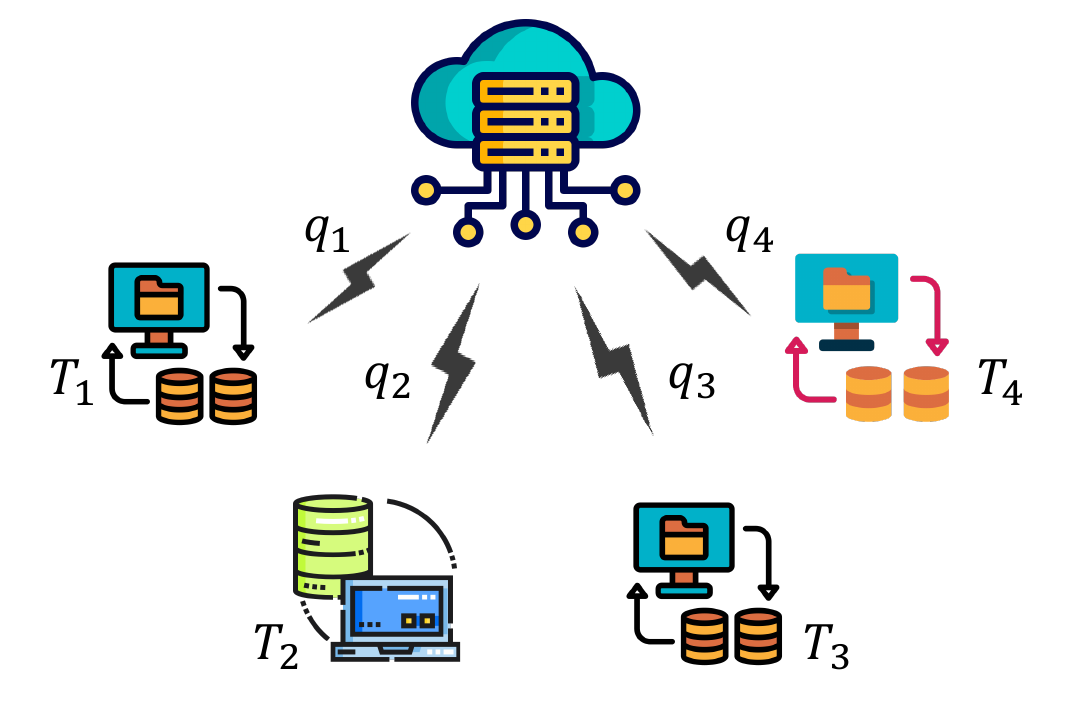}
    \caption{\small An FL scenario at the network edge with heterogeneous communication probabilities $q_m$ and computation epochs $T_m$.}
    \label{fig:heter_comm_comp}
    \vspace{-5mm}
\end{figure}

\vspace{-0.5em}
\section{ Phenomenon of Interest } \label{sec: Phenomenon of Interest}
This section illustrates the negative impacts of heterogeneous communication and computation. An example and various simulations are used to showcase the significant bias. 
\begin{example}[Mismatched Objective Functions]\label{example: 1}
Let $F_m(\boldsymbol{X})\triangleq\frac{1}{2}\lVert \boldsymbol{X}-\boldsymbol{E}_m \rVert^2$, where $\boldsymbol{X}, \boldsymbol{E}_m\in \mathbb{R}^d$. Suppose there are $M$ clients in total, the global objective in \eqref{eq: goal} is defined as $F(\boldsymbol{X})=\frac{1}{M}\sum_{m=1}^M F_m(\boldsymbol{X})$ with the unique minimizer $\boldsymbol{X}^\star=\frac{1}{M}\sum_{m=1}^M \boldsymbol{E}_m$. For any set of weights $\{\Omega_m\}_{m=1}^M: \sum_{m=1}^M \Omega_m=1$, the surrogate function is defined as $\Tilde{F}(\boldsymbol{X})=\sum_{m=1}^M \Omega_m F_m(\boldsymbol{X})$ with unique minimizer $\Tilde{\boldsymbol{X}}^\star=\sum_{m=1}^M \Omega_m \boldsymbol{E}_m$. The $\Tilde{\boldsymbol{X}}^\star$ can diverge arbitrarily far from $\boldsymbol{X}^\star$.   
\end{example}
\begin{example}[Mismatched Output Caused by Heterogeneous Communication and Computation]\label{example: 2}
Following Example \ref{example: 1}, in each round, $K$ clients are independently sampled with probability $p_m = \omega_m = \frac{1}{M}$. Suppose each selected client performs $T_m$ steps of local training using a sufficiently small learning rate $\eta$, and successfully uploads its local model update to the PS with probability $1-q_m$. {For a sufficiently small learning rate $\eta$ and a sufficiently large number of training rounds $R$, under the aggregation rule in \eqref{eq: agg_rule1}, \texttt{FedAvg} converges to}
\footnote{{The complete derivation of \eqref{eq: converge_point_sgd} is provided in the supplementary file \cite[Appendix F]{weng2025heterogeneityawareclientsamplingunified}, including the finite-learning-rate case and its simplification in the small-learning-rate regime.} }
\begin{align}
    &\Tilde{\boldsymbol{X}}^\star=\frac{(1-q_m)T_m}{\sum_{m=1}^M(1-q_m)T_m}\boldsymbol{E}_m,
    \label{eq: converge_point_sgd}
\end{align}
in expectation, which minimizes
\begin{align}
& \Tilde{F}(\boldsymbol{X})=\sum_{m=1}^M\frac{(1-q_m)T_m}{\sum_{m=1}^M(1-q_m)T_m}F_m(\boldsymbol{X}), 
\end{align}
instead of the defined global objective function $F(\boldsymbol{X})$. With this toy example, Figure \ref{fig:examples} shows the distance between the pursued true global minimum $\boldsymbol{X}^\star$ and the convergence point $\Tilde{\boldsymbol{X}}^\star$ achieved by various local solvers and our proposed \texttt{FedACS} under heterogeneous settings. The results imply that $\Tilde{\boldsymbol{X}}^\star$ can deviate largely from $\boldsymbol{X}^\star$ if the system heterogeneity is not properly handled. 
\begin{figure}[ht]
    \centering
    \begin{minipage}[b]{0.45\textwidth}
    \centering
    \includegraphics[width=0.8\linewidth]{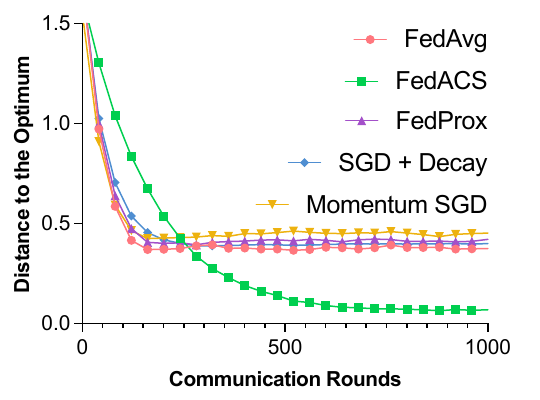}
    \end{minipage}
    \vspace{-2mm}
    \caption{\small Comparison of distance to the optimum in \texttt{FedAvg} with SGD \cite{mcmahan2023communicationefficientlearningdeepnetworks}, proximal SGD ($\mu=1$), SGD with decayed learning rate (decay rate $0.005$), momentum SGD (momentum $0.3$), and our proposed \texttt{FedACS} in Example \ref{example: 1}, where $\boldsymbol{E}_m\sim \mathcal{N}(0,\mathbf{I}_{10\times 10})$, $M = 30$, $K = 15$, and $\eta = 0.001$. 
    Heterogeneous setting where $T_m$ is uniformly distributed from $1$ to $30$ and $q_m$ from $0.01$ to $0.3$ across clients.
    }
    \vspace{-2mm}
    \label{fig:examples}
\end{figure}
\end{example}


\section{New Unified Theoretical Analysis for General Heterogeneous FL} \label{sec: analysis_heter_fl}
This section presents the fundamentals of the optimization dynamics and multi-view convergence behaviors in general heterogeneous FL. It introduces novel perspectives on mechanisms behind objective inconsistency and presents a non-convex analysis broadly applicable to real-world FL with modern deep neural networks (DNNs). By setting $p_m = \omega_m$, we quantify the surrogate objective and optimization gap induced by system heterogeneity.
\vspace{-0.5em}
\subsection{Impact of Heterogeneous Communication and Computation} \label{sec: Objective Inconsistency}
In each round, heterogeneous computation is conveyed through unreliable communication. 
Consider the general heterogeneous FL in Section \ref{sec: system model}. The expectation of the aggregated local model updates is 
\begin{align}
    &\mathbb{E}_{\boldsymbol{p}, \boldsymbol{q}}\left[ \frac{1}{K} \sum_{m\in \Tilde{\mathcal{S}}_r(\boldsymbol{p},\boldsymbol{q})} (-\eta\boldsymbol{\Delta}_{m}^{r,\boldsymbol{\xi}}) \right] \notag\\
    &=-\underbrace{\eta\sum_{m=1}^M p_m(1-q_m)}_{\eta_{\mathrm{eff}}}\cdot \sum_{m=1}^M \underbrace{\frac{p_m(1-q_m)}{\sum_{m=1}^M p_m(1-q_m)}}_{\gamma_m}\boldsymbol{\Delta}_{m}^{r,\boldsymbol{\xi}}
    \label{eq: E_agg}.
\end{align}
On the one hand, the heterogeneous communication reduces the chances of PS seeing each local model update proportionally with non-identical scales, making the aggregation weights $\{\gamma_m\}_{m=1}^M$ statistically deviate from the pursued target $\{\omega_m\}_{m=1}^M$. On the other hand, the effective learning rate $\eta_{\mathrm{eff}}$ that takes effect in the actual learning process is reduced by communication. 
Next, let us expand the computation component $\boldsymbol{\Delta}_{m}^{r,\boldsymbol{\xi}}$. The \eqref{eq: E_agg} can be further rewritten to 
\begin{align}
    &-\eta_{\mathrm{eff}} \cdot \underbrace{\sum_{m=1}^M \gamma_m \lVert \boldsymbol{a}_m\rVert_1}_{T_{\mathrm{eff}}} \cdot \sum_{m=1}^M \underbrace{\frac{\gamma_m \lVert \boldsymbol{a}_m\rVert_1}{\sum_{m=1}^M \gamma_m \lVert \boldsymbol{a}_m\rVert_1 }}_{\Omega_m} \cdot\underbrace{\frac{\nabla\boldsymbol{F}_m^{r,\boldsymbol{\xi}}\cdot\boldsymbol{a}_m}{\lVert \boldsymbol{a}_m\rVert_1}}_{\nabla\Bar{\boldsymbol{F}}_m^{r,\boldsymbol{\xi}}}. 
    \label{eq: obj_inconsist_comm_comp}
\end{align}
From \eqref{eq: obj_inconsist_comm_comp}, the deviation introduced by heterogeneous computation is structural, rather than statistical. In \eqref{eq: obj_inconsist_comm_comp}, the accumulated local stochastic gradients, $\nabla\boldsymbol{F}_m^{r,\boldsymbol{\xi}}\cdot\boldsymbol{a}_m$, is normalized, denoted by $\nabla\Bar{\boldsymbol{F}}_m^{r,\boldsymbol{\xi}}$, and common effective number of local iterations $T_{\mathrm{eff}}$ is extracted while maintaining $\sum_{m=1}^M \Omega_m=1$. 
Both communication and computation impact on $\{\Omega_m\}_{m=1}^{M}$, indicating that objective inconsistency in general heterogeneous FL is both structural and statistical in nature. 
Existing works \cite{xiang2024efficient,wang2024a,perazzone2022communication,wang2021quantized,jhunjhunwala2022fedvarp} leverage statistical properties to eliminate objective inconsistency are ineffective against structural inconsistency. In contrast, \cite{wangTackling} resolves structural inconsistency but is invalid to statistical inconsistency. Our proposed \texttt{FedACS} framework, however, is capable of addressing both types of inconsistencies.

\begin{assumption}[Unbiased Gradient and Bounded Variance]\label{assump: unbiased}
   For $\forall i \in [M]$, the local stochastic gradient is an unbiased estimator of the true local gradient, i.e., $\mathbb{E}_\xi[\nabla F_m(x\vert \xi)]=\nabla F_m(x)$, where $\nabla F_m(x)$ has bounded data variance $\mathbb{E}_\xi[\lVert \nabla F_m(x\vert \xi)-\nabla F_m(x)\rVert^2]\leq \sigma^2 $, $\sigma^2>0$. 
\end{assumption}
Under Assumption \ref{assump: unbiased}, the results in \eqref{eq: obj_inconsist_comm_comp} imply 
\begin{align}
    \mathbb{E}_{\xi, \boldsymbol{p}, \boldsymbol{q}}\left[ \frac{1}{K} \sum_{m\in \Tilde{\mathcal{S}}_r(\boldsymbol{p},\boldsymbol{q})} (-\eta\boldsymbol{\Delta}_{m}^{r,\boldsymbol{\xi}}) \right]=
    -\eta_{\mathrm{eff}}T_{\mathrm{eff}}\sum_{m=1}^M \Omega_m \nabla\Bar{\boldsymbol{F}}_m^{r}. 
    \label{eq: actual opt}
\end{align}
The term $\nabla\Bar{\boldsymbol{F}}_m^{r}$ is the normalized true gradient, defined similarly to $\nabla\Bar{\boldsymbol{F}}_m^{r,\boldsymbol{\xi}}$, except that $\nabla\boldsymbol{F}_m^{r,\boldsymbol{\xi}}$ is replaced by $\nabla\boldsymbol{F}_m^{r}$. From \eqref{eq: actual opt}, if $\eta$ is chosen small enough, the function being optimized is actually the surrogate function $\Tilde{F}(\boldsymbol{X})= \sum_{m=1}^M \Omega_m F_m(\boldsymbol{X})$
rather than the true objective function $F(\boldsymbol{X})$ in \eqref{eq: goal}. 
Similar results hold for full participation, US, and OS.
\subsection{Non-Convex Convergence Analysis with Heterogeneous Communication and Computation with Arbitrary Local Solver} \label{sec: convergence_heterFL}
To begin with, several widely adopted assumptions are made concerning the local objective functions. 

\begin{assumption}[{$L$-Smoothness}]\label{assump: Smooth}
    For $\forall i \in [M]$, each local objective function is lower bounded by $F_m(x)\geq F^\star$ and is Lipschitz differentiable and its gradient $\nabla F_m(x)$ is L-smooth, i.e., $\lVert \nabla F_m(x)-\nabla F_m(y) \rVert\leq L\lVert x-y \rVert$, $\forall x, y\in \mathbb{R}^d$.
\end{assumption}
\begin{assumption}[Bounded Dissimilarity]\label{assump: Dissimilarity}
    For  $\forall\{\omega_m\}_{m=1}^{M}: \sum_{m=1}^{M} \omega_m=1$, the dissimilarity between the local objective functions $\nabla F_m(x)$ and the global objective function $\nabla F(x)$ is bounded by $\sum_{m=1}^M \omega_m\lVert \nabla F_m(x) \rVert^2\leq \beta^2\lVert  \nabla F(x) \rVert^2+\kappa^2$, $\beta^2\geq 1, \kappa^2\geq 0$.
\end{assumption}
\begin{remark}[Relation Among the Bounded Dissimilarity Assumptions] 
To clarify the landscape, we compare the most commonly adopted assumptions, listed in Table \ref{table: assumps}, Appendix \ref{appx: Equivalent Assumptions}, and establish the following relationship: Assump. 8 $\succeq$ Assump. 7 $\succeq$ Assump. 4 $\Leftrightarrow$ Assump. 5 $\succeq$ Assump. 3 $\Leftrightarrow$ Assump. 6, where $\succeq$ indicates stricter constraints and $\Leftrightarrow$ indicates equivalence. The primary distinction among these assumptions lies in whether a bound on the true local gradients is imposed. Assump. 7 enforces bounded local gradients in both IID and non-IID settings. In contrast, Assump. 4 and Assump. 5 relax this requirement partially—only in the IID setting or partially in the non-IID setting—while Assump. 3 and Assump. 6 do not impose such bounds in either case. Furthermore, we find that Assump. 8 is logically equivalent to the conjunction of Assumptions 7 and 2, i.e., \{Assump. 7\} $\cap$ \{Assump. 2\} $\Leftrightarrow$ Assump. 8.
Importantly, in the analysis of strongly convex FL, Assump. 7 and Assump. 8 should be avoided, as they conflict with the strong convexity assumption \cite{nguyen2019newconvergenceaspectsstochastic}. Additionally, when the $F_m(x)$ are identical and data distributions are homogeneous, it holds that $D_m^2=0$, $D^2=0$, $\beta^2=1$, $\kappa^2=0$. However, the interpretations in \cite{wangTackling,wang2023delta} overlook the fact that $F_m(x)$ is defined on local datasets. 
\end{remark}

Next, several necessary lemmas are presented.
\begin{lemma}\label{lemma:3}
Under Assumptions 1--3, it holds that
\begin{align*}
    &\mathbb{E}\left[ \left\langle\nabla\Tilde{F}(\boldsymbol{X}_{r-1}), \boldsymbol{X}_{r}-\boldsymbol{X}_{r-1}\right\rangle \right]\\
    &\leq -\frac{1}{2}\eta_{\mathrm{eff}}T_{\mathrm{eff}}\mathbb{E}\left[ \left\lVert \nabla\Tilde{F}(\boldsymbol{X}_{r-1}) \right\rVert^2 \right]\\
    &+\frac{1}{2}\eta_{\mathrm{eff}}T_{\mathrm{eff}}\sum_{m=1}^M\Omega_m \frac{L^2}{\lVert \boldsymbol{a}_m \rVert_1} \sum_{t=1}^{T_m} a_{m,t} \mathbb{E}\left[ \left\lVert \boldsymbol{X}_{r-1}-  \boldsymbol{X}_{m,r}^{t-1}  
    \right\rVert^2 \right] \notag\\
    &-\frac{1}{2}\eta_{\mathrm{eff}}T_{\mathrm{eff}}\mathbb{E}\left[ \left\lVert \sum_{m=1}^M\Omega_m \nabla\Bar{\boldsymbol{F}}_m^{r} \right\rVert^2 \right]. 
    \numberthis
\end{align*}
\end{lemma}
\begin{proof}
The proof is provided in Appendix \ref{appx:lemma1}.
\end{proof}

\begin{lemma}\label{lemma:4}
Under Assumptions 1--3, it holds that
\begin{align*}
    &\mathbb{E}\left[ \left\lVert \frac{1}{K}\sum_{m\in \Tilde{\mathcal{S}}_r(\boldsymbol{p},\boldsymbol{q})} \left(\boldsymbol{\Delta}_{m,r}^{\boldsymbol{\xi}}
    - \boldsymbol{\Delta}_{m,r}\right)\right\rVert^2 \right]\\
    &\leq \sigma^2 \sum_{m=1}^M p_m
\lVert \boldsymbol{a}_m \rVert_1\sum_{t=1}^{T_m} a_{m,t}(1-q_m).
    \numberthis
\end{align*}
\end{lemma}
\begin{proof}
The proof is provided in Appendix \ref{appx:lemma2}.
\end{proof}

\begin{lemma}
\label{lemma: lemma 1}
     {For an arbitrary set of non-negative functions $\{f_m(x)\}_{m=1}^M$, it holds that}
    \begin{align}
    &{\sum_{m=1}^M p_m(1-q_m) \lVert \boldsymbol{a}_m \rVert_1^2 f_m(x)} \notag\\
    &{\leq A\sum_{m=1}^Mp_m(1-q_m)T_{\mathrm{eff}}\sum_{m=1}^M \Omega_m f_m(x),}
    \label{eq: lemma}
    \end{align}
    {where $A\triangleq\max_m\{\lVert \boldsymbol{a}_m \rVert_1\}$.  The nonnegativity condition is satisfied in our subsequent analysis, where $f_m(x)$ corresponds to $\ell_2$-norms.} 
\end{lemma}
\begin{proof}[Proof]
The proof is provided in Appendix \ref{appx:lemma3}.
\end{proof}

\begin{lemma}\label{lemma:5}
Under Assumptions 1--3, it holds that
\begin{align*}
    &\mathbb{E}\left[ \left\lVert \frac{1}{K}\sum_{m\in \Tilde{\mathcal{S}}_r(\boldsymbol{p},\boldsymbol{q})} \boldsymbol{\Delta}_{m,r}\right\rVert^2 \right]\leq\\
    & 2L^2 \sum_{m=1}^M p_m (1-q_m) \lVert \boldsymbol{a}_m \rVert_1\sum_{t=1}^{T_m} a_{m,t} \mathbb{E}\left[\left\lVert \boldsymbol{X}_{m,r}^{t-1}- \boldsymbol{X}_{r-1}  \right\rVert^2 \right]\\
    &+ 2A\sum_{m=1}^Mp_m(1-q_m)T_{\mathrm{eff}} \left(\beta^2\mathbb{E}\left[ \left\lVert  \nabla \Tilde{F}(\boldsymbol{X}_{r-1}) \right\rVert^2 \right]+\kappa^2\right).
    \numberthis
\end{align*}
\end{lemma}
\begin{proof}[Proof]
The proof is provided in Appendix \ref{appx:lemma4}.
\end{proof}

\begin{lemma}[Accummulation of Local Variance]
\label{lemma: lemma 2}
    For the consecutive local optimization performed on each client, the accumulated local variance is bounded by   
    \begin{align}
        &\sum_{t=1}^{T_m} a_{m,t}\mathbb{E}\left[ \left\lVert  \boldsymbol{X}_{r-1}-\boldsymbol{X}_{m,r}^{t-1}  \right\rVert^2 \right] \notag\\
        &\leq \frac{2\eta^2\lVert \boldsymbol{a}_m \rVert_1^3}{1-2\eta^2L^2\lVert \boldsymbol{a}_m \rVert_1^2} \mathbb{E}\left[\lVert \nabla F_m\left( \boldsymbol{X}_{r-1}\right) \rVert^2\right].
        \label{eq:local_accum}
    \end{align}
\end{lemma}
\begin{proof}[Proof]
The proof is provided in Appendix \ref{appx:lemma5}.
\end{proof}

Our main results are stated below. With the help of Lemma \ref{lemma:3}--\ref{lemma: lemma 2}, Lemma \ref{Lemma: decent_surrogate} and \ref{Lemma: decent_true} characterize per-round gradient descent for the surrogate and true objective functions. Theorem \ref{theo: Convergence of the Surrogate} provides a rigorous justification for our analysis in Section \ref{sec: Objective Inconsistency}, characterizing the convergence behavior of the surrogate objective. Building on this, Theorem~\ref{theo:converge_bound_true} establishes a quantitative convergence analysis for the true objective function.
\begin{lemma}[Decent Lemma of the Surrogate Objective Function]\label{Lemma: decent_surrogate}
Conditioning on the $\sigma$-algebra generated by the randomness up to round $r$, and Assumptions \ref{assump: unbiased}--\ref{assump: Dissimilarity}, the the expected squared $\ell_2$-norm of the gradient of the surrogate objective function $\Tilde{F}(\boldsymbol{X})$ at the $r$-th round is bounded by
\begin{align}
    &\mathbb{E}\left[ \Tilde{F}(\boldsymbol{X}_{r}) \right]
    -\mathbb{E}\left[ \Tilde{F}(\boldsymbol{X}_{r-1}) \right] \notag\\
    &\leq \eta_{\mathrm{eff}} T_{\mathrm{eff}} \beta^2 \left(-1+\rho(\eta,L,A)
    \right)\mathbb{E}\left[ \left\lVert  \nabla \Tilde{F}(\boldsymbol{X}_{r-1}) \right\rVert^2 \right]\notag\\
    &\hspace{5mm}+\eta_{\mathrm{eff}} T_{\mathrm{eff}} \kappa^2 \left(-\frac{1}{2}+
    \rho(\eta,L,A)
    \right)
    +\frac{1}{2}\eta_{\mathrm{eff}} \eta LA^2 \sigma^2,
    \label{eq:one_round_theorem_1}
\end{align}
where $\rho(\eta,L,A)= (2\eta LA+1)\left(\frac{\eta^2L^2 A^2}{1-2\eta^2L^2A^2}+\frac{1}{2}\right)$, and $\lim_{\eta\rightarrow0}\rho(\eta,L,A)=\frac{1}{2}$.
\end{lemma}
\begin{proof}
    The proof is provided in Appendix \ref{appx:lemma6}.
\end{proof}
\begin{lemma}[Decent Lemma of the True Objective Function]\label{Lemma: decent_true}
Under Assumptions \ref{assump: unbiased}--\ref{assump: Dissimilarity}, the $\ell_2$-norm of the gradient evaluated by the true objective function $F(\boldsymbol{X})$ at the $r$-th round is bounded by
\begin{align}
    &\left \lVert \nabla F(\boldsymbol{X}_{r-1}) \right\rVert\notag \\
    &\leq \left(\beta\sqrt{\chi_{\boldsymbol{\omega}\Vert\boldsymbol{\Omega} }^2} +1\right) \left\lVert  \nabla \Tilde{F}(\boldsymbol{X}_{r-1}) \right\rVert  +\kappa \sqrt{\chi_{\boldsymbol{\omega}\Vert\boldsymbol{\Omega} }^2}, 
\end{align}
\end{lemma}
\begin{proof}
    The proof is provided in Appendix \ref{appx:lemma7}.
\end{proof}
\begin{theorem}[Convergence of the Surrogate Objective Function]\label{theo: Convergence of the Surrogate}
Under assumption \ref{assump: unbiased}--\ref{assump: Dissimilarity}, any federated optimization algorithm, that aims to minimize \eqref{eq: goal}  with computational capabilities $\{T_m\}_{m=1}^M$ and intermittent connectivity $\{Z_m^r\sim \mathrm{Ber}(1-q_m)\}_{m=1}^M$ and follows the update rule in \eqref{eq: agg_rule1}, will converge to a stationary point that minimizes the surrogate objective function $\Tilde{F}(\boldsymbol{X})$. If $\eta$ is sufficiently small, the optimization error after $R$ rounds of training is bounded by 
\begin{align*}
    &\frac{1}{R}\sum_{r=1}^R \mathbb{E}\left[ \left\lVert \nabla \Tilde{F}(\boldsymbol{X}_{r-1}) \right\rVert^2 \right]\\
    &\leq \frac{4(\Tilde{F}(\boldsymbol{X}_R)-\Tilde{F}^\star)}{\eta_{\mathrm{eff}}T_{\mathrm{eff}}R}
    +\eta_{\mathrm{eff}}T_{\mathrm{eff}}\frac{\kappa^2}{\beta^2}+\frac{2\eta LA^2}{T_{\mathrm{eff}}}\sigma^2\triangleq \epsilon_{\mathrm{opt}}.
    \numberthis
    \label{eq: theo1}
\end{align*}
where $A\triangleq \max\left\{\lVert\boldsymbol{a}_m\rVert_1\right\}$. If $\eta \propto \frac{1}{\sqrt{R}}$, $\epsilon_{\mathrm{opt}}\rightarrow0$ with convergence rate $\mathcal{O}\left(\frac{1}{\sqrt{R}}\right)$.    
\end{theorem}
\begin{proof}
    The proof is provided in Appendix \ref{appx: theo1}.
\end{proof}
\begin{theorem}[Convergence of the True Objective Function]\label{theo:converge_bound_true}
Under the same condition in Theorem \ref{theo: Convergence of the Surrogate}, if $\eta\propto \frac{1}{\sqrt{R}}$, the optimization error evaluated by the true objective function $F(\boldsymbol{X})$ is bounded by
\begin{align*}
    \lim_{R\rightarrow +\infty}\frac{1}{R}\sum_{r=1}^R \left\lVert \nabla F(\boldsymbol{X}_{r-1}) \right\rVert^2
    \leq \chi_{\boldsymbol{\omega}\Vert\boldsymbol{\Omega}}^2 \kappa^2.
    \numberthis
    \label{eq:converge_bound_true}
\end{align*}
Furthermore, the expected optimization error of $F(\boldsymbol{X})$ evaluated over data samples is bounded by
\begin{align*}
    \lim_{R\rightarrow +\infty}\frac{1}{R}\sum_{r=1}^R \mathbb{E}_\xi\left[\left\lVert \nabla F(\boldsymbol{X}_{r-1}) \right\rVert^2\right]
    \leq \chi_{\boldsymbol{\omega}\Vert\boldsymbol{\Omega}}^2 \kappa^2+\sigma^2.
    \numberthis
    \label{eq:converge_bound_true2}
\end{align*}
\end{theorem}
\begin{proof}
    The proof of is provided in Appendix \ref{appx: theo2}.
\end{proof}
\begin{remark}[Achievability of \eqref{eq:converge_bound_true}] Theorem \ref{theo:converge_bound_true} implies that $F(\boldsymbol{X})$ may converge to a persisting point 
or never converge. Compared to \cite[Theorem 2]{wangTackling}, our constructed convergence bound in \eqref{eq:converge_bound_true} is tight and achievable, 
implying the strict sub-optimality of regular FL algorithms in heterogeneous settings. 
The achievability proof can be found in Appendix \ref{Appx: achieve}
\end{remark}
\begin{corollary}[Distance Between the Inconsistent Solution and Consistent Solution] 
    The Euclidean distance between the resulting inconsistent convergence point $\Tilde{\boldsymbol{X}}^\star$ from heterogeneous FL and the true target $\boldsymbol{X}^\star$ in \eqref{eq: goal} is bounded by
    $\lVert \Tilde{\boldsymbol{X}}^\star-\boldsymbol{X}^\star \rVert\geq 
    \frac{1}{L}\lVert \nabla F(\Tilde{\boldsymbol{X}}^\star)\rVert$,  
where $\lVert \nabla F(\Tilde{\boldsymbol{X}}^\star)\rVert\in \left[0,\sqrt{\chi_{\boldsymbol{\omega}\Vert\boldsymbol{\Omega} }^2} \kappa\right]$. This indicates that  $\Tilde{\boldsymbol{X}}^\star$ can diverge arbitrarily far from $\boldsymbol{X}^\star$. 
\end{corollary}\label{corollary2.1}
\begin{remark}[No Convergence Guarantee in Dynamic Heterogeneous FL]
Under the same condition in Theorem \ref{theo: Convergence of the Surrogate}, but with time-varying $\{T_m^{(r)}\}_{m=1}^M$ and $\{Z_m^r\sim \mathrm{Ber}(1-q_m^{(r)})\}_{m=1}^M$, the surrogate functions $\Tilde{F}_r(\boldsymbol{X})$ in each round are statistically non-identical for $r\in[R]$, leading to a potential divergence.
\end{remark}

\section{{Consistent Communication-Computation Co-Design}} \label{sec:comm_comp_design}
{According to Theorem \ref{theo:converge_bound_true}, the optimization dynamics converge to the correct global optimum if $\chi_{\boldsymbol{\omega}\Vert\boldsymbol{\Omega} }^2 = 0$, i.e., $\forall m: \Omega_m = \omega_m$.} This further gives 
\begin{align}
    (1-q_m)\lVert \boldsymbol{a}_m \rVert_1=\sum_{m=1}^M \omega_m (1-q_m) \lVert \boldsymbol{a}_m \rVert_1.
\end{align}
In matrix form, this becomes
\begin{align}
    \boldsymbol{W} \Tilde{\boldsymbol{q}}=\mathbf{0}, 
    \label{eq: consistent_consition_0}
\end{align}
where $\Tilde{\boldsymbol{q}}=[
        1-q_1,
        1-q_2,
        \cdots,
        1-q_M]^\top$, 
and
\begin{align}
    &\boldsymbol{W}=\notag\\
    &\begin{bNiceMatrix}
        (\omega_1-1) \lVert \boldsymbol{a}_1 \rVert_1 &\omega_2\lVert \boldsymbol{a}_2 \rVert_1
        &\cdots &\omega_M \lVert \boldsymbol{a}_M \rVert_1\\
        \omega_1 \lVert \boldsymbol{a}_1 \rVert_1 &(\omega_2-1) \lVert \boldsymbol{a}_2 \rVert_1
        &\cdots &\omega_M \lVert \boldsymbol{a}_M \rVert_1\\
        \vdots & \vdots  & \ddots &\vdots\\
        \omega_1 \lVert \boldsymbol{a}_1 \rVert_1 &\omega_2  \lVert \boldsymbol{a}_2 \rVert_1
        &\cdots &(\omega_M-1) \lVert \boldsymbol{a}_M \rVert_1\\
    \end{bNiceMatrix}  \notag  
\end{align}
with $\sum_{m=1}^M \omega_m=1$. Let $\boldsymbol{w}_m$ the $m$-th row in $\boldsymbol{W}$, we have
\begin{align}
    \boldsymbol{w}_1+\sum_{m=2}^M \omega_m\cdot\boldsymbol{w}_m=\mathbf{0},
    \label{eq: equation_contraint}
\end{align}
indicating that the coefficient matrix $\boldsymbol{W}$ is rank-deficient. 
Thus, \eqref{eq: consistent_consition_0} admits many non-trivial solutions. 
Solving \eqref{eq: consistent_consition_0} leads to the following theorem.
\begin{theorem}[{Consistent Communication-Computation Co-Design}]\label{theo: comm-comp}
{For a heterogeneous FL system described in Section~\ref{sec: system model}, 
the following condition is sufficient to guarantee objective consistency under 
regular participation schemes, e.g., IS, US, OS, and full participation\footnote{The universality of the proposed condition is verified in \cite[Appendix Q]{weng2025heterogeneityawareclientsamplingunified}}, as 
well as subsequent frameworks evolved from them:}
\begin{align}
    \forall m\in [M]: \lVert \boldsymbol{a}_m \rVert_1=\frac{1-q_1}{1-q_m}\lVert \boldsymbol{a}_1 \rVert_1. 
    \label{eq: comm-comp design}
\end{align}
{By substituting \eqref{eq: comm-comp design} into \eqref{eq: E_agg} and \eqref{eq: obj_inconsist_comm_comp}, the effective learning rate and the common number of epochs are 
obtained}
\begin{align}
    &\eta_{\mathrm{eff}}=\eta\sum_{m=1}^M \omega_m(1-q_m),\\
    &T_{\mathrm{eff}}=\frac{(1-q_1)\lVert  \boldsymbol{a}_1 \rVert_1}{\sum_{m=1}^M \omega_m(1-q_m)}.
    \label{eq: eta_comm-comp design}
\end{align}
\end{theorem}
It can be verified that homogeneous communication and computation, i.e., 
$\forall m\in[M]: q_m=q$ and 
$\lVert \boldsymbol{a}_m \rVert_1=\lVert \boldsymbol{a} \rVert_1$, 
satisfy \eqref{eq: consistent_consition_0}. Theorem~\ref{theo: comm-comp} further reveals that such homogeneity constitutes a sufficient, but not necessary, 
condition for maintaining the optimality of the FL training process.
\section{\texttt{FedACS}: Biased Sampling for Debiasing Optimization}\label{sec: FedACS}
\begin{figure*}[!htb]
    \begin{minipage}[b]{1\textwidth}
    \vspace{-10pt}
    \centering
    \includegraphics[width=0.85\linewidth]{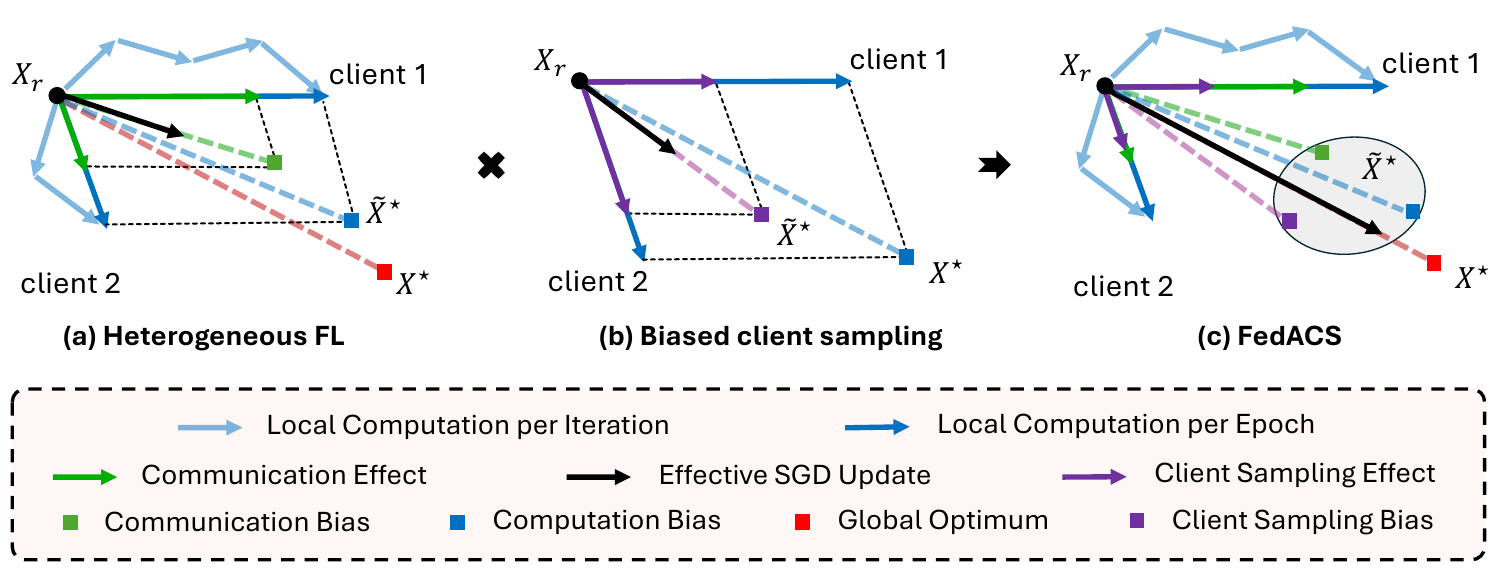}
    \vspace{-3pt} 
    \caption{\small {Illustration of the proposed biased client sampling scheme \texttt{FedACS} to counterbalance the objective inconsistency induced by heterogeneous communication and computation.} }
    \vspace{-12pt}
   \label{fig:adap_CS}
   \end{minipage}
\end{figure*}
\subsection{Proposed Method}
The rationale behind the proposed method is to design a biased client sampling scheme to counteract the biased optimization dynamics arising from heterogeneous communication $q_m^{(r)}$ and computation $\boldsymbol{a}_m^{(r)}$. Theorem \ref{theo:converge_bound_true} suggests if the clients are sampled independently with probability $p_m^{(r)}$ such that the non-vanishing term $\chi_{\boldsymbol{\omega}\Vert\boldsymbol{\Omega}^{(r)}}^2=0$ in each round, the FL process will converge to the pursued global optimum in \eqref{eq: goal}, as shown in Figure  \ref{fig:adap_CS}, i.e., $\forall m\in [M]$,
\begin{align}
    (1-q_m^{(r)})\lVert\boldsymbol{a}_m^{(r)}
    \rVert_1\cdot  p_m^{(r)}
    &=\sum_{m=1}^M \omega_m(1-q_m^{(r)})\lVert\boldsymbol{a}_m^{(r)}
    \rVert_1 \cdot p_m^{(r)},\notag\\
    &\hspace{-10mm}\mathrm{s.t.}\; p_m^{(r)}\in (0,1), \sum_{m=1}^Mp_m^{(r)}=1.
    \label{eq: adaptive_CS}
\end{align}
The unique solution is of the following form, 
\begin{align}
    p_m^{(r)}=\frac{\frac{ \omega_m}{(1-q_m^{(r)})\lVert\boldsymbol{a}_m^{(r)}
    \rVert_1}}{\sum_{m=1}^M \frac{ \omega_m}{(1-q_m^{(r)})\lVert\boldsymbol{a}_m^{(r)}
    \rVert_1}}.
    \label{eq: adap_prob}
\end{align}
In \texttt{FedACS}, $K$ clients are sampled adaptively with replacement according to \eqref{eq: adap_prob}, at server aggregation, it yields the following new algorithm: 
\begin{align}
    &[\texttt{FedACS}]\hspace{5mm} \boldsymbol{X}_{r}-\boldsymbol{X}_{r-1}\notag\\
    &=-\underbrace{\eta\frac{\sum_{m=1}^M\frac{\omega_m}{\lVert\boldsymbol{a}_m^{(r)}
    \rVert_1}}{\sum_{m=1}^M\frac{\omega_m}{(1-q_m^{(r)})\lVert\boldsymbol{a}_m^{(r)}
    \rVert_1}}}_{ \eta_{\mathrm{eff}}^{(r)}} \cdot \underbrace{\frac{1}{\sum_{m=1}^M\frac{\omega_m}{\lVert\boldsymbol{a}_m^{(r)}
    \rVert_1}}}_{T_{\mathrm{eff}}^{(r)}}
   \sum_{m=1}^M \omega_m \nabla\Bar{\boldsymbol{F}}_m^{r}.
   \label{eq: fedACS_update}
\end{align}
\begin{remark}[Impact on Step Length]
Compared to \texttt{FedAvg} where $p_m^{(r)}=w_m$ in \eqref{eq: obj_inconsist_comm_comp}, 
equation \eqref{eq: fedACS_update} shows 
a variation in the effective step length given by $\eta_{\mathrm{eff}}^{(r)}T_{\mathrm{eff}}^{(r)} = \frac{\eta}{\sum_{m=1}^M \frac{ \omega_m}{(1-q_m^{(r)})\lVert\boldsymbol{a}_m^{(r)}\rVert_1}}$, which can either increase or decrease. This phenomenon is common in state-of-art works \cite{wang2024a,xiang2024efficient,wangTackling,jhunjhunwala2022fedvarp}. 
While such variability may impact convergence speed, the optimality of algorithms remains unaffected. 
\end{remark}
\vspace{-0.5em}
\subsection{Convergence Analysis} 
In dynamic settings, the optimized function $F_r(\boldsymbol{X})$  per round is identical to $F(\boldsymbol{X})$ in \texttt{FedACS}, thus it is guaranteed to converge towards the pursued global optimum in \eqref{eq: goal}. 
\begin{theorem}[Convergence of \texttt{FedACS}]\label{theo: FedACS}
In \texttt{FedACS}, clients are sampled independently according to \eqref{eq: adap_prob}, thus the aggregation is unbiased in each round. Therefore, \texttt{FedACS} will converge to the true target defined in \eqref{eq: goal}, and it yields that 
    \begin{align}
        &\frac{1}{R}\sum_{r=1}^R \mathbb{E}\left[ \left\lVert \nabla F(\boldsymbol{X}_{r-1}) \right\rVert^2 \right] \notag\\
        &\leq \frac{4 B}{R\eta}\left( \mathbb{E}\left[ F(\boldsymbol{X}_0) \right]-F^\star \right)+\eta\Bar{C}\frac{\kappa^2}{\beta^2}+2\eta L \Bar{D}\sigma^2,
    \end{align}
where $B$, $\Bar{C}$, $\Bar{D}$ are constants in \eqref{eq: defineABC}. If $\eta  \propto \frac{1}{\sqrt{R}}$, \texttt{FedACS} converges with rate $\mathcal{O}\left(\frac{B}{\sqrt{R}}\right)+\mathcal{O}\left(\frac{\Bar{C}\kappa^2}{\sqrt{R}\beta^2} \right)+\mathcal{O}\left( \frac{\Bar{D}\sigma^2}{\sqrt{R}}\right)$. 
\end{theorem}
\begin{proof}
    The proof of Theorem \ref{theo: FedACS} is provided in Appendix \ref{appx: theo 4}.
\end{proof}
\begin{remark}[{Acquisition of prior knowledge}]
{The proposed biased client sampling scheme \texttt{FedACS} requires prior information about local solvers and communication outage statistics, i.e., $\lVert\boldsymbol{a}_m^{(r)}\rVert_1$ and $q_m^{(r)}$. The $\lVert\boldsymbol{a}_m^{(r)}\rVert_1$ is known by default in algorithm design, and $q_m^{(r)}$ can be obtained by exact closed-form computation \cite{5595117,wang2021quantized} or via Monte Carlo simulations with sufficient precision \cite{6042306,9409740} based on the communication system parameters.}\footnote{{ In practical deployments, the estimation accuracy of $q_m^{(r)}$ may be affected by hardware impairments. Future work may investigate the impact of imperfect knowledge of $q_m^{(r)}$ on the performance of \texttt{FedACS}, with particular emphasis on characterizing the resulting aggregation-weight bias and convergence behavior in the presence of hardware impairments.
}} 
\end{remark}
\vspace{-0.5em}
\section{Numerical Experiments}\label{sec: simulation}
\subsection{System Setups} 
We evaluate the performance of the proposed \texttt{FedACS} algorithm by comparing it against six state-of-the-art benchmark frameworks, each representing a distinct class of frameworks based on similar principles: \texttt{FedAvg}, and the communication-aware variant \texttt{c-a-FedAvg}~\cite{perazzone2022communication}, \texttt{FedVarp} \cite{jhunjhunwala2022fedvarp}, \texttt{FedNova} \cite{wangTackling},  \texttt{FedAU} \cite{wang2024a}, OS \cite{chen2020optimal}, {\texttt{FOCUS} \cite{ying2026exact}}. In each communication round, a fraction of $0.3$ of the clients is randomly selected out of $M=20$ clients to participate in training.  
The clients are divided into two groups. For the first half, local iterations $T_m^{(r)}\sim\mathcal{U}(1, 10)$ and communication probabilities $q_m^{(r)}\sim\mathcal{U}(0.4, 0.5)$; for the second half, $T_m^{(r)}\sim\mathcal{U}(20, 30)$ and $q_m^{(r)}\sim\mathcal{U}(0.0, 0.1)$. This models a heterogeneous system where clients exhibit varying levels of computation and communication capabilities.    

The experiments are performed on image classification tasks on the MNIST, CIFAR-10, and CINIC-10 datasets. 
The data distribution is non-IID across clients. For the MNIST dataset, an extreme scenario is adopted by assigning only one class to each client. 
For CIFAR-10 and CINIC-10 datasets, we use the Dirichlet distribution to partition the datasets. For a given concentration parameter $\alpha$, the Dirichlet distribution controls the degree of skewness of the label distribution among the clients. A lower $\alpha$ indicates more imbalanced and disjoint class distributions. The $\alpha$ is set to $0.1$. 
The employed NNs and hyperparameters are specified in Table \ref{tbl: cnn structures}. 
\vspace{-2mm}
\begin{table}[!ht]
\caption{
\footnotesize
NN architecture, loss function, learning rate, and batch size specifications}
\label{tbl: cnn structures}
\resizebox{\linewidth}{!}{
\begin{tabular}{cccc}
\toprule
{\bf Datasets}& 
{\bf MNIST} & 
{\bf CIFAR-10} & 
{\bf CINIC-10} \\
\toprule
Neural network &
CNN &
CNN &
CNN \\
Model architecture$^*$ & 
\begin{tabular}{p{.18\textwidth}}
\centering
{\bf C}(1,10)
-- {\bf C}(10,20)
-- {\bf D}
-- {\bf L}(50)
-- {\bf L}(10)
\end{tabular}
&
\begin{tabular}{p{.18\textwidth}}
\centering
{\bf C}(3,32)
-- {\bf R}
-- {\bf M}
-- {\bf C}(32,32)
-- {\bf R}
-- {\bf M}
-- {\bf L}(256)
-- {\bf R}
-- {\bf L}(64)
-- {\bf R}
-- {\bf L}(10)
\end{tabular}
&
\begin{tabular}{p{.18\textwidth}}
\centering
{\bf C}(3,32)
-- {\bf R}
-- {\bf M}
-- {\bf C}(32,32)
-- {\bf R}
-- {\bf M}
-- {\bf D}
-- {\bf L}(512)
-- {\bf R}
-- {\bf D}
-- {\bf L}(256)
-- {\bf R}
-- {\bf D}
-- {\bf L}(10)
\end{tabular} \\
\midrule
Loss function &
\multicolumn{2}{c}{Negative log likelihood loss} &\multicolumn{1}{c}{Cross-entropy loss} \\
\addlinespace[1ex]
Learning rate $\eta_1$
&
\multicolumn{2}{c}{$\eta_1=0.02$} & \multicolumn{1}{c}{$\eta_1=0.03$} \\ 
\addlinespace[1ex] 
Number of global rounds $T$ &
\multicolumn{3}{c}{200} 
\\
\midrule
Batch size &
\multicolumn{2}{c}{512} &\multicolumn{1}{c}{1024 } \\
\bottomrule
\end{tabular}}
\vskip.2\baselineskip
\begin{tabular}{p{.47\textwidth}}
$^*$
\begin{footnotesize}    
{\bf C}(\# in-channel, \# out-channel): a 2D convolution layer (kernel size 3, stride 1, padding 1);
{\bf R}: ReLU activation function;
{\bf M}: a 2D max-pool layer (kernel size 2, stride 2);
{\bf L}: (\# outputs): a fully-connected linear layer;
{\bf D}: a dropout layer (probability 0.2).
\end{footnotesize}
\end{tabular}
\vspace{-2mm}
\end{table}
\vspace{-0.5em}
\subsection{{Consistent Communication-Computation Co-Design}}  
The learning rate is set to $\eta=0.001$ in FedAvg.  In the optimal communication-computation design, $q_1=0.01$ remains the same, and $q_m$, $m=2, \cdots, M,$ are re-designed according to \eqref{eq: eta_comm-comp design}. To ensure a fair comparison, the effective step size $\eta_{\mathrm{eff}}T_{\mathrm{eff}}$ is ensured to be equal to that in FedAvg.

Fig. \ref{fig:comm-comp co-design} compares the convergence behavior of different optimizers and the proposed optimal communication-computation co-design in terms of the log-distance to the true optimum versus communication rounds. 
The proposed communication–computation design consistently achieves the most accurate convergence, constantly approaching the true optimum compared with FedAvg, FedProx, SGD with learning-rate decay, and Momentum SGD. 
In contrast, the baseline frameworks exhibit slower or plateauing convergence due to biased optimization dynamics, whereas the proposed design continues to reduce the optimality gap as communication rounds increase, demonstrating the optimality of the design rule under heterogeneous settings.

\begin{figure}
    \centering
    \includegraphics[width=0.8\linewidth]{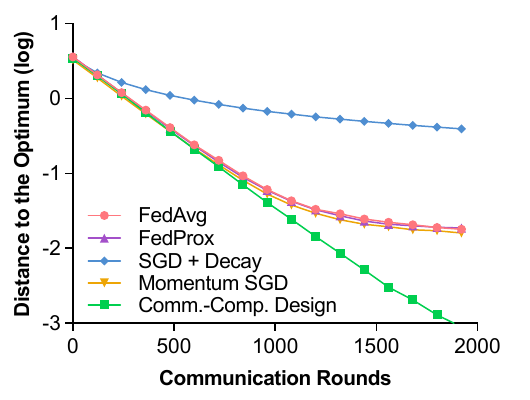}
    \vspace{-3mm}
    \caption{{Verification of consistent communication-computation design principle}. In a heterogeneous setting, $T_m$ takes fixed values in $\{2,3\}$, where the first $20$ clients have $T_m = 2$ and the last 10 clients have $T_m = 3$, and $q_m$ from $0.01$ to $0.3$ across clients. Other parameters remain the same as in Fig. \ref{fig:examples}.
    }
    \vspace{-3mm}
    \label{fig:comm-comp co-design}
\end{figure}

\begin{figure*}[!htb]
\vspace{-2mm}
    \centering
    \begin{minipage}[b]{0.3\textwidth}
    \centering
    \includegraphics[width=1\linewidth]{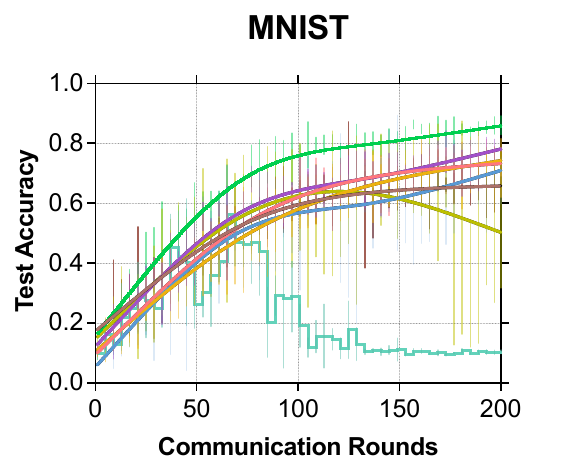}
    \end{minipage}
    \hfill
    \begin{minipage}[b]{0.3\textwidth}
    \centering
    \includegraphics[width=1\linewidth]{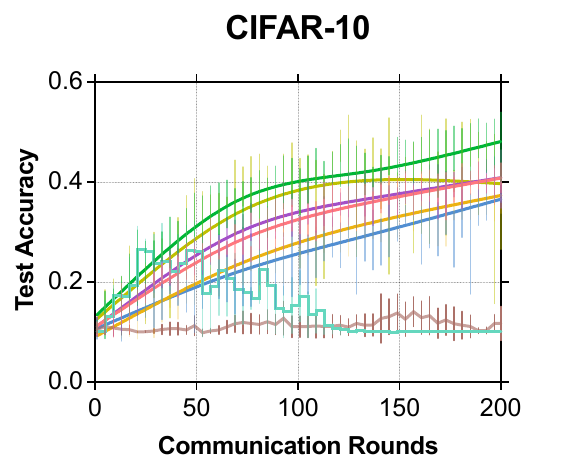}
    \end{minipage}
    \hfill
    \begin{minipage}[b]{0.3\textwidth}
    \centering
    \includegraphics[width=1\linewidth]{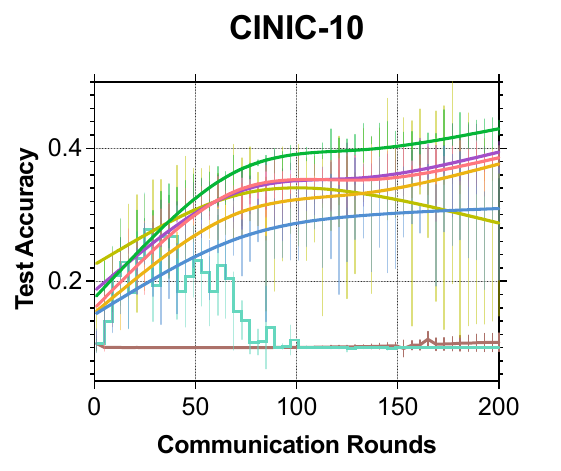}
    \end{minipage}
    \vspace{-0mm}
    \begin{minipage}[b]{1\textwidth}
    \centering
    \vspace{0mm}
    \includegraphics[width=0.9\linewidth]{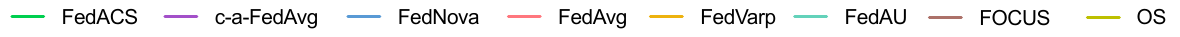}
    \end{minipage}
    \vspace{-7mm}
    \caption{\small \textbf{Test Accuracy of algorithms over various datasets in dynamic heterogeneous FL.} The average test accuracy over multiple runs is smoothed and plotted for each algorithm. Vertical bars indicate the variance at the corresponding communication round, reflecting the stability of each algorithm.
    } \label{fig:dynamic_exp}
    \vspace{-4mm}
\end{figure*}
\vspace{-0.5em}
\subsection{Superiority of FedACS}   
To ensure a fair comparison of each algorithm’s ability to reach the correct optimum, all benchmark frameworks are calibrated to operate with the same effective step size by aligning the learning rate $\eta$. This eliminates confounding factors, allowing performance differences to reflect each algorithm’s ability to handle objective inconsistency precisely. {Yet, \texttt{FOCUS} relies on a special gradient accumulation mechanism and is difficult to align. Therefore, for \texttt{FOCUS}, we set the best-performing learning rate after testing, which is $0.0005$ for MNIST and $0.0002$ for CIFAR-10 and CINIC-10.}

\textbf{Convergence to the Consistent Solution.} In Figure \ref{fig:dynamic_exp}, the results show that \texttt{FedACS} outperforms all benchmark frameworks achieving improvements in test accuracy of $7.5\%$-$36\%$ on MNIST, $7\%$-$11\%$ on CIFAR-10, and $4.3\%$-$15\%$ on CINIC-10. This benefits from its ability to address both statistical and structural objective inconsistencies, effectively handling heterogeneous communication and computation and allowing it to converge to the correct global optimum, while other benchmarks cannot.

The \texttt{FedAvg} fails to address any form of objective inconsistency. In contrast, \texttt{c-a-FedAvg} accounts for communication heterogeneity and yields slight performance gains, albeit with increased variance in the global model. The \texttt{FedVarp} reuses previously observed local updates; however, due to non-stationary statistics across training rounds, it fails to resolve any form of objective inconsistency as well. Worse, it can introduce misleading statistics, further destabilizing the training process. The \texttt{FedNova} addresses computational heterogeneity, but its performance degrades under heterogeneous communication due to the resulting increase in variance.
The OS reflected the performance of a class of frameworks, including \texttt{DELTA}, etc., \cite{chen2020optimal,9796935}. These frameworks determine sampling probabilities without considering system heterogeneity, relying on alternative factors; as a result, their performance tends to be opportunistic and sub-optimal in heterogeneous FL. The \texttt{FedAU} is representative of the group taking account for prior connectivity realization of clients, e.g., \texttt{FedAWE} \cite{xiang2024efficient}. While \texttt{FedAU} aims to mitigate statistical inconsistency, its performance is less stable, as the accumulation of gradients after consecutive transmission failures is too large and leads to divergence. {Similarly, despite the pull-push gradient strategy adopted by \texttt{FOCUS} \cite{ying2026exact} addressing arbitrary client participation, its delayed gradient updates may increase sensitivity to computation heterogeneity and data noise.
}

\begin{table*}[!htb]
 \centering
 \vspace{0pt}
 \caption{\textbf{Training efficiency of algorithms over various datasets in dynamic heterogeneous FL.} We report the average number of communication rounds and total computation time required to reach the predefined accuracy thresholds for the first time.} \label{table: dynamic_sota}
 \vspace{-.em}
    \resizebox{1.\textwidth}{!}{
  \begin{tabular}{l c c c c c c c c c c}
   \toprule
   \multirow{2}{*}{Algorithm} & \multicolumn{2}{c}{MNIST }& \multicolumn{2}{c}{CIFAR-10} & \multicolumn{2}{c}{CINIC-10} \\
   \cmidrule(lr){2-3} \cmidrule(lr){4-5}  \cmidrule(lr){6-7} 
            & Rounds for 70\% & Time (s) for 70\% 
            & Rounds for 40\% & Time (s) for 40\%
            & Rounds for 35\% & Time (s) for 35\%     \\
   \midrule
   \rowcolor{lightgray} \texttt{FedACS}     & \textbf{68\small {\transparent{0.5} (1.0$\times$)}}   &\textbf{3548.6 \small {\transparent{0.5} (1.0$\times$)}}  & \textbf{100 \small {\transparent{0.5} (1.0$\times$)}}& \textbf{6928.9 \small {\transparent{0.5} (1.0$\times$)}} & \textbf{55 \small {\transparent{0.5} (1.0$\times$)}} & \textbf{11452.3 \small {\transparent{0.5} (1.0$\times$)}} \\
   \texttt{FedAvg}     & 93\small {\transparent{0.5} (1.37$\times$)}  & 5470.4 \small {\transparent{0.5} (1.38$\times$)} &  143 \small {\transparent{0.5} (1.43$\times$)} & 12482.0 \small {\transparent{0.5} (1.80$\times$)} & 74 \small {\transparent{0.5} (1.34$\times$)} &14550.0 \small {\transparent{0.5} (1.27$\times$)} \\
\texttt{c-a-FedAvg}    & 84\small {\transparent{0.5} (1.24$\times$)}  & 4884.4 \small {\transparent{0.5} (1.51$\times$)}&  130 \small {\transparent{0.5} (1.30$\times$)} & 7917.0\small {\transparent{0.5} (1.14$\times$)} &67 \small {\transparent{0.5} (1.22$\times$)} &13256.8\small {\transparent{0.5} (1.16$\times$)} \\ 
\texttt{FedVarp}         & 115\small {\transparent{0.5} (1.69$\times$)}  & 7289.9 \small {\transparent{0.5} (2.05$\times$)} &  165 \small {\transparent{0.5} (1.65$\times$)} & 13479.8\small {\transparent{0.5} (1.95$\times$)} & 104 \small {\transparent{0.5} (1.89$\times$)} &22200.5\small {\transparent{0.5} (1.94$\times$)}  \\
\texttt{FedNova}   & 99\small {\transparent{0.5} (1.46$\times$)}  & 6043.4 \small {\transparent{0.5} (1.70$\times$)} &  161 \small {\transparent{0.5} (1.61$\times$)} & 13314.8\small {\transparent{0.5} (1.92$\times$)} & 100 \small {\transparent{0.5} (1.81$\times$)} &21802.2 \small {\transparent{0.5} (1.90$\times$)} \\ 
\texttt{FedAU}         & --  & --  &  --  & -- & -- &--  \\ 
OS        & 72\small {\transparent{0.5} (1.06$\times$)}  & 4391.9 \small {\transparent{0.5} (1.24$\times$)} &  80 \small {\transparent{0.5} (0.8$\times$)} & 5095.3 \small {\transparent{0.5} (0.74$\times$)} & 65 \small {\transparent{0.5} (1.18$\times$)} &14248.9\small {\transparent{0.5} (1.24$\times$)} \\
\texttt{FOCUS}        & --  & -- &  -- & -- & -- & -- \\
   \bottomrule
  \end{tabular}
  }
\vskip.2\baselineskip
\begin{tabular}{p{.95\textwidth}}
\begin{footnotesize}    
A dash (–) indicates a failure to reach the specified threshold within the entire training process.
\end{footnotesize}
\end{tabular}
  \vspace{-8pt}
\end{table*}

\textbf{Communication and Computation Efficiency.}
{We evaluate training efficiency using two complementary metrics: (i) the number of communication rounds required to first reach a predefined target accuracy, which measures communication efficiency, and (ii) the cumulative computation time required to reach the same target, which measures computation efficiency. Table II reports both the absolute values and their ratios relative to FedACS.The reduction ratio achieved by \texttt{FedACS} relative to a baseline method is defined as
$\frac{C_{\mathrm{baseline}}-C_{\texttt{FedACS}}}
{C_{\texttt{FedACS}}}\times 100\%$, where $C_{\mathrm{baseline}}$ and
$C_{\texttt{FedACS}}$ denote the resource costs of the baseline method and
\texttt{FedACS}, respectively.
Table~\ref{table: dynamic_sota} shows that \texttt{FedACS} reaches the target accuracy thresholds in fewer communication rounds than \texttt{FedAvg}, \texttt{c-a-FedAvg}, \texttt{FedVarp}, \texttt{FedNova}, and \texttt{FedAU}, reducing the required communication rounds by $24\%$--$69\%$ on MNIST, $30\%$--$65\%$ on CIFAR-10, and $22\%$--$89\%$ on CINIC-10.
The communication efficiency stem from \texttt{FedACS}'s ability to correct the objective inconsistency induced by joint communication--computation heterogeneity, thereby guiding the global model toward the true optimum. In contrast, these baseline methods cannot fully address the two intertwined forms of heterogeneity and may converge to biased stationary points, requiring more communication rounds to reach the same target accuracy, and OS is not stable in achieving the threshold accuracy and is strictly sub-optimal. It should be noted that their relative performance depends on the degree of communication and computation heterogeneity.
In addition, \texttt{FedACS} requires substantially less computation time to reach the target accuracy, with reductions of $38\%$--$105\%$ on MNIST, $14\%$--$95\%$ on CIFAR-10, and $16\%$--$94\%$ on CINIC-10. Beyond improved convergence direction,  this gain also stems from its adaptive client sampling, which tends to prioritize clients with fewer local training epochs if the communication reliability is not very poor. This selection preference reduces the overall computational workload without compromising convergence quality. 
}
\hspace{-0.5em}
\subsection{Ablation Study}
{Ablation studies are conducted to validate FedACS in dynamic heterogeneous FL under varying data heterogeneity (Table \ref{table: vary_data}) and under varying client sampling ratios (Table \ref{table: vary_frac}) on CIFAR-10. Moreover, we also validate FedACS in static heterogeneous FL on varying datasets. Full results are reported in Supplementary file \cite[Appendix P]{weng2025heterogeneityawareclientsamplingunified}. } 
{The results show that \texttt{FedACS} outperforms \texttt{FedAvg}, \texttt{c-a-FedAvg}, \texttt{FedVarp}, \texttt{FedNova}, and OS, achieving improvements in test accuracy of $6\%$-$13.8\%$, $2\%$-$12\%$, $6\%$-$14\%$ on CIFAR-10 under varying Client Sampling Ratio, where $\mathrm{frac}=0.3$, $\mathrm{frac}=0.5$, and $\mathrm{frac}=0.4$, respectively, achieving improvements in test accuracy of $3.5\%$-$13.4\%$, $4\%$-$14.3\%$, $3\%$-$12\%$ on CIFAR-10 under varying levels of data imbalance, where $\alpha=0.05$, $\alpha=0.2$ and $\alpha=0.35$, respectively.}
{
As shown in Table~\ref{table: vary_data}, \texttt{FedACS} reaches target accuracy thresholds in fewer communication rounds compared to \texttt{FedAvg}, \texttt{c-a-FedAvg}, \texttt{FedVarp}, and \texttt{FedNova}, achieving reductions of $33\%$-$39\%$, $72\%$-$206\%$, and $144\%$-$342\%$, and achieving reduction in computation time of $45\%$-$44\%$, and $87\%$-$273\%$, $162\%$-$422\%$ on CIFAR-10 under varying levels of data imbalance, where $\alpha=0.05$, $\alpha=0.2$ and $\alpha=0.35$, respectively.}
{
As shown in Table~\ref{table: vary_frac}, \texttt{FedACS} reaches target accuracy thresholds in fewer communication rounds compared to \texttt{FedAvg}, \texttt{c-a-FedAvg}, \texttt{FedVarp}, and \texttt{FedNova}, achieving reductions of $33\%$-$39\%$, $29\%$-$34\%$, and $100\%$-$220\%$, and achieving reduction in computation time of $44\%$-$45\%$, and $89\%$-$91\%$, $118\%$-$256\%$ on CIFAR-10 under varying Client Sampling Ratio, where $\mathrm{frac}=0.3$, $\mathrm{frac}=0.5$, and $\mathrm{frac}=0.4$, respectively. 
}

\begin{table*}[!t]
 \centering
 \vspace{-0.4em}
 \caption{\textbf{Training Efficiency of algorithms on CIFAR-10 under varying levels of data imbalance in dynamic heterogeneous FL.} We report the average number of communication rounds and total computation time required to reach the predefined accuracy thresholds for the first time across different datasets. 
 } 
 \label{table: vary_data}
 \vspace{-0em}
    \resizebox{1.\textwidth}{!}{
  \begin{tabular}{l c c c c c c c c c c}
   \toprule
   \multirow{2}{*}{Algorithm} & \multicolumn{2}{c}{$\alpha=0.05$}& \multicolumn{2}{c}{$\alpha=0.2$} & \multicolumn{2}{c}{$\alpha=0.35$} \\
   \cmidrule(lr){2-3} \cmidrule(lr){4-5}  \cmidrule(lr){6-7} 
            & Rounds for 35\% & Time (s) for 35\% 
            & Rounds for 35\% & Time (s) for 35\%
            & Rounds for 40\% & Time (s) for 40\%     \\
   \midrule
  \rowcolor{lightgray} \texttt{FedACS}     & \textbf{107\small {\transparent{0.5} (1.0$\times$)}} &\textbf{6186.2\small {\transparent{0.5} (1.0$\times$)}}  & \textbf{51 \small {\transparent{0.5} (1.0$\times$)}}& \textbf{4747.3 \small {\transparent{0.5} (1.0$\times$)}} & \textbf{36}\small {\transparent{0.5} (1.0$\times$)} & \textbf{3391.5\small {\transparent{0.5} (1.0$\times$)}}\\
   \texttt{FedAvg}     & 149\small {\transparent{0.5} (1.39$\times$)}  & 9591.7 \small {\transparent{0.5} (1.55$\times$)} &  88 \small {\transparent{0.5} (1.72$\times$)} & 8930.0 \small {\transparent{0.5} (1.88$\times$)} & 88 \small {\transparent{0.5} (2.44$\times$)} & 8918.3 \small {\transparent{0.5} (2.62$\times$)} \\
    \texttt{c-a-FedAvg}    &142 \small {\transparent{0.5} (1.33$\times$)}  & 8982.4 \small {\transparent{0.5} (1.45$\times$)} &  88\small {\transparent{0.5} (1.72$\times$)} & 8873.4 \small {\transparent{0.5} (1.87$\times$)} &90 \small {\transparent{0.5} (2.5$\times$)} & 9293.9 \small {\transparent{0.5} (2.74$\times$)}\\ 
    \texttt{FedVarp}         & --  & -- & 156\small {\transparent{0.5} (3.06$\times$)} & 17728.6 \small {\transparent{0.5} (3.73$\times$)} & 159 \small {\transparent{0.5} (4.42$\times$)} &17733.9 \small {\transparent{0.5} (5.22$\times$)}  \\
    \texttt{FedNova}         & --  & -- &  125 \small {\transparent{0.5} (2.45$\times$)} & 13690.4 \small {\transparent{0.5} (2.88$\times$)} & 94 \small {\transparent{0.5} (2.61$\times$)} &9901.3\small {\transparent{0.5} (2.92$\times$)} \\ 
    \texttt{FedAU}         & --  & -- &  -- & -- & -- & -- \\ 
    OS        & 110\small {\transparent{0.5} (1.02$\times$)} & 6310.7\small {\transparent{0.5} (1.02$\times$)} &   54\small {\transparent{0.5} (1.06$\times$)} & 5305.3 \small {\transparent{0.5} (1.18$\times$)} & 32 \small {\transparent{0.5} (0.89$\times$)} &3100.7 \small {\transparent{0.5} (0.94$\times$)} \\
   \bottomrule
  \end{tabular}
  }
  \vskip.2\baselineskip
\begin{tabular}{p{.95\textwidth}}
\begin{footnotesize}    
A dash (–) indicates a failure to reach the specified threshold within the entire training process.
\end{footnotesize}
\end{tabular}
  \vspace{-0.6em}
\end{table*}

\begin{table*}[!t]
 \centering
 \vspace{-0.4em}
 \caption{\textbf{Training Efficiency of algorithms on CIFAR-10 under varying client participation ratios in dynamic heterogeneous FL.} We report the average number of communication rounds and total computation time required to reach the predefined accuracy thresholds for the first time across different datasets. 
 } 
 \label{table: vary_frac}
 \vspace{-0em}
    \resizebox{1.\textwidth}{!}{
  \begin{tabular}{l c c c c c c c c c c}
   \toprule
   \multirow{2}{*}{Algorithm} & \multicolumn{2}{c}{$\mathrm{frac}=0.3,\alpha=0.05$}& \multicolumn{2}{c}{$\mathrm{frac}=0.5,\alpha=0.05$} & \multicolumn{2}{c}{$\mathrm{frac}=0.4,\alpha=0.2$} \\
   \cmidrule(lr){2-3} \cmidrule(lr){4-5}  \cmidrule(lr){6-7} 
            & Rounds for 35\% & Time (s) for 35\% 
            & Rounds for 35\% & Time (s) for 35\%
            & Rounds for 35\% & Time (s) for 40\%     \\
   \midrule
  \rowcolor{lightgray} \texttt{FedACS}     & \textbf{107\small {\transparent{0.5} (1.0$\times$)}} &\textbf{6186.2\small {\transparent{0.5} (1.0$\times$)}}  & \textbf{103 \small {\transparent{0.5} (1.0$\times$)}}& \textbf{1912.3 \small {\transparent{0.5} (1.0$\times$)}} & \textbf{40}\small {\transparent{0.5} (1.0$\times$)} & \textbf{1881.0\small {\transparent{0.5} (1.0$\times$)}}\\
   \texttt{FedAvg}     & 149\small {\transparent{0.5} (1.39$\times$)}  & 9591.7 \small {\transparent{0.5} (1.55$\times$)} &  138 \small {\transparent{0.5} (1.34$\times$)} & 3609.2 \small {\transparent{0.5} (1.89$\times$)} & 80 \small {\transparent{0.5} (2$\times$)} & 4326.2 \small {\transparent{0.5} (2.30$\times$)} \\
    \texttt{c-a-FedAvg}    &142 \small {\transparent{0.5} (1.33$\times$)}  & 8982.4 \small {\transparent{0.5} (1.45$\times$)} &  133\small {\transparent{0.5} (1.29$\times$)} & 3647.9 \small {\transparent{0.5} (1.91$\times$)} & 81 \small {\transparent{0.5} (2.03$\times$)} & 4108.4 \small {\transparent{0.5} (2.18$\times$)}\\ 
    \texttt{FedVarp}         & --  & -- & -- & -- & 128 \small {\transparent{0.5} (3.2$\times$)} & 6704.2 \small {\transparent{0.5} (3.56$\times$)}  \\
    \texttt{FedNova}         & --  & -- &  --  & -- & 98 \small {\transparent{0.5} (2.45$\times$)} &5021.8\small {\transparent{0.5} (2.67$\times$)} \\ 
    \texttt{FedAU}         & --  & -- &  -- & -- & -- & -- \\ 
    OS        & 110\small {\transparent{0.5} (1.02$\times$)} & 6310.7\small {\transparent{0.5} (1.02$\times$)} &   83\small {\transparent{0.5} (0.81$\times$)} & 2340.7 \small {\transparent{0.5} (1.22$\times$)} & 43 \small {\transparent{0.5} (1.08$\times$)} &1834.9 \small {\transparent{0.5} (0.98$\times$)} \\
   \bottomrule
  \end{tabular}
  }
  \vskip.2\baselineskip
\begin{tabular}{p{.95\textwidth}}
\begin{footnotesize}    
A dash (–) indicates a failure to reach the specified threshold within the entire training process.
\end{footnotesize}
\end{tabular}
  \vspace{-0.6em}
\end{table*}
\hspace{-0.5em}
\section{Conclusion}\label{sec: conclusion}
System heterogeneity can significantly degrade the FL performance. Despite its practical prevalence, this issue remains largely underexplored in general heterogeneous settings. To bridge this gap, we conduct a rigorous analysis of the optimization dynamics under heterogeneous conditions, revealing distinct mechanisms through which communication and computation heterogeneity induce objective inconsistency.
Our findings show that existing state-of-the-art frameworks are insufficient when both forms of heterogeneity are present. In contrast, we proposed \texttt{FedACS}, a universal framework designed to handle all types of inconsistencies. Both theoretical analysis and empirical results confirm the effectiveness and robustness of \texttt{FedACS} in heterogeneous FL.
Future research may explore how our findings can be integrated into existing frameworks in broader FL scenarios.

\vspace{-0.5em}
\bibliographystyle{IEEEtran}
\bibliography{IEEEabrv,reference}

\begin{appendices}
\section{Proof of Lemma \ref{lemma:3}}\label{appx:lemma1}
First, we bound the first term at the right-hand side of \eqref{eq: smooth_lossF}. During the global model update process, the communication is unreliable. In \eqref{eq: bound_first_term}, we directly applied \eqref{eq: obj_inconsist_comm_comp} to the expectation of $\boldsymbol{X}_{r}-\boldsymbol{X}_{r-1}$, where $\eta_{\mathrm{eff}}$, $T_{\mathrm{eff}}$, and $\Omega_m$ already reflected the impact of unreliable communication.  
\begin{subequations}
    \begin{align}
        &\mathbb{E}\left[ \left\langle\nabla\Tilde{F}(\boldsymbol{X}_{r-1}), \boldsymbol{X}_{r}-\boldsymbol{X}_{r-1}\right\rangle \right]\\
        &=\mathbb{E}\left[ \left\langle\nabla\Tilde{F}(\boldsymbol{X}_{r-1}), -\eta_{\mathrm{eff}}T_{\mathrm{eff}}\sum_{m=1}^M\Omega_m \nabla\Bar{\boldsymbol{F}}_m^{r,\boldsymbol{\xi}}\right\rangle \right]\\
        &=-\eta_{\mathrm{eff}}T_{\mathrm{eff}}\mathbb{E}\left[ \left\langle\nabla\Tilde{F}(\boldsymbol{X}_{r-1}), \sum_{m=1}^M\Omega_m \nabla\Bar{\boldsymbol{F}}_m^{r}\right\rangle \right]\\
        &=-\frac{1}{2}\eta_{\mathrm{eff}}T_{\mathrm{eff}}\mathbb{E}\left[ \left\lVert \nabla\Tilde{F}(\boldsymbol{X}_{r-1}) \right\rVert^2 \right]\notag\\
        &\hspace{5mm}-\frac{1}{2}\eta_{\mathrm{eff}}T_{\mathrm{eff}}\mathbb{E}\left[ \left\lVert \sum_{m=1}^M\Omega_m \nabla\Bar{\boldsymbol{F}}_m^{r} \right\rVert^2 \right]\notag\\
        &\hspace{5mm}+\frac{1}{2}\eta_{\mathrm{eff}}T_{\mathrm{eff}}\mathbb{E}\left[ \left\lVert \nabla\Tilde{F}(\boldsymbol{X}_{r-1})-\sum_{m=1}^M\Omega_m \nabla\Bar{\boldsymbol{F}}_m^{r}
        \right\rVert^2 \right]\label{eq: 48d}\\ 
        &\leq -\frac{1}{2}\eta_{\mathrm{eff}}T_{\mathrm{eff}}\mathbb{E}\left[ \left\lVert \nabla\Tilde{F}(\boldsymbol{X}_{r-1}) \right\rVert^2 \right]\notag\\
        &\hspace{5mm}-\frac{1}{2}\eta_{\mathrm{eff}}T_{\mathrm{eff}}\mathbb{E}\left[ \left\lVert \sum_{m=1}^M\Omega_m \nabla\Bar{\boldsymbol{F}}_m^{r} \right\rVert^2 \right] \notag\\ 
        &\hspace{5mm}+\frac{1}{2}\eta_{\mathrm{eff}}T_{\mathrm{eff}}\sum_{m=1}^M\Omega_m \mathbb{E}\left[ \left\lVert \nabla\Tilde{F}(\boldsymbol{X}_{r-1})- \nabla\Bar{\boldsymbol{F}}_m^{r}
        \right\rVert^2 \right]\label{eq: 48e}\\ 
        &=-\frac{1}{2}\eta_{\mathrm{eff}}T_{\mathrm{eff}}\mathbb{E}\left[ \left\lVert \nabla\Tilde{F}(\boldsymbol{X}_{r-1}) \right\rVert^2 \right]\notag\\
        &\hspace{5mm}-\frac{1}{2}\eta_{\mathrm{eff}}T_{\mathrm{eff}}\mathbb{E}\left[ \left\lVert \sum_{m=1}^M\Omega_m \nabla\Bar{\boldsymbol{F}}_m^{r} \right\rVert^2 \right] +\frac{1}{2}\eta_{\mathrm{eff}}T_{\mathrm{eff}}\sum_{m=1}^M\Omega_m\cdot\notag\\
        &\mathbb{E}\left[ \left\lVert \nabla\Tilde{F}(\boldsymbol{X}_{r-1})- \frac{1}{\lVert \boldsymbol{a}_m \rVert_1} \sum_{t=1}^{T_m} a_{m,t} \nabla F_m(\boldsymbol{X}_{m,r}^{t-1}) 
        \right\rVert^2 \right]\\
        &\leq -\frac{1}{2}\eta_{\mathrm{eff}}T_{\mathrm{eff}}\mathbb{E}\left[ \left\lVert \nabla\Tilde{F}(\boldsymbol{X}_{r-1}) \right\rVert^2 \right]\notag\\
        &\hspace{5mm}-\frac{1}{2}\eta_{\mathrm{eff}}T_{\mathrm{eff}}\mathbb{E}\left[ \left\lVert \sum_{m=1}^M\Omega_m \nabla\Bar{\boldsymbol{F}}_m^{r} \right\rVert^2 \right] \notag\\
        &\hspace{5mm}+\frac{1}{2}\eta_{\mathrm{eff}}T_{\mathrm{eff}}\sum_{m=1}^M\Omega_m \frac{1}{\lVert \boldsymbol{a}_m \rVert_1} \sum_{t=1}^{T_m} a_{m,t}\cdot \notag\\
        &\hspace{20mm}\mathbb{E}\left[ \left\lVert \nabla\Tilde{F}(\boldsymbol{X}_{r-1})-  \nabla F_m(\boldsymbol{X}_{m,r}^{t-1}) 
        \right\rVert^2 \right]\label{eq: 48g} \\
        &\leq -\frac{1}{2}\eta_{\mathrm{eff}}T_{\mathrm{eff}}\mathbb{E}\left[ \left\lVert \nabla\Tilde{F}(\boldsymbol{X}_{r-1}) \right\rVert^2 \right]\notag\\ 
        &\hspace{5mm}-\frac{1}{2}\eta_{\mathrm{eff}}T_{\mathrm{eff}}\mathbb{E}\left[ \left\lVert \sum_{m=1}^M\Omega_m \nabla\Bar{\boldsymbol{F}}_m^{r} \right\rVert^2 \right] \notag\\ 
        &+\frac{1}{2}\eta_{\mathrm{eff}}T_{\mathrm{eff}}\sum_{m=1}^M\Omega_m \frac{L^2}{\lVert \boldsymbol{a}_m \rVert_1} \sum_{t=1}^{T_m} a_{m,t} \mathbb{E}\left[ \left\lVert \boldsymbol{X}_{r-1}-  \boldsymbol{X}_{m,r}^{t-1} 
        \right\rVert^2 \right]\label{eq: 48h}
    \end{align}
    \label{eq: bound_first_term}
\end{subequations}
In \eqref{eq: bound_first_term}, \eqref{eq: 48d} is due to property of $\ell_2$-norm, \eqref{eq: 48e} and \eqref{eq: 48g} owe to Jensen's inequality, and \eqref{eq: 48h} is due to Assumption \ref{assump: Smooth}.

\section{Proof of Lemma \ref{lemma:4}}\label{appx:lemma2}
\begin{subequations}
\begin{align}
&\mathbb{E}\left[ \left\lVert \frac{1}{K}\sum_{m\in \Tilde{\mathcal{S}}_r(\boldsymbol{p},\boldsymbol{q})} \left(\boldsymbol{\Delta}_{m,r}^{\boldsymbol{\xi}}
    - \boldsymbol{\Delta}_{m,r}\right)\right\rVert^2 \right]\\
&=\mathbb{E}\left[ \left\lVert \frac{1}{K}\sum_{m\in\mathcal{S}_r(\boldsymbol{p})} \left(\nabla\boldsymbol{F}_m^{r,\boldsymbol{\xi}}
    - \nabla\boldsymbol{F}_m^{r}\right)\boldsymbol{a}_m\cdot Z_m^r\right\rVert^2 \right] \\
&\leq \mathbb{E}\left[ \frac{1}{K}\sum_{m\in\mathcal{S}_r(\boldsymbol{p})} \left\lVert \left(\nabla\boldsymbol{F}_m^{r,\boldsymbol{\xi}}
    - \nabla\boldsymbol{F}_m^{r}\right)\boldsymbol{a}_m\cdot Z_m^r\right\rVert^2 \right]\label{eq: 50b} \\ 
&=\mathbb{E}\Bigg[ \frac{1}{K}\sum_{m\in\mathcal{S}_r(\boldsymbol{p})} \Bigg\lVert 
\sum_{t=1}^{T_m} a_{m,t} \big(\nabla F_m(\boldsymbol{X}_{m,r}^{t-1},\boldsymbol{\xi}_{m,r}^{t})\notag\\
&\hspace{35mm}-\nabla F_m(\boldsymbol{X}_{m,r}^{t-1}) \big)
\Bigg\rVert^2 \cdot Z_m^r \Bigg]\\
&\leq\mathbb{E}\Bigg[ \frac{1}{K}\sum_{m\in\mathcal{S}_r(\boldsymbol{p})} 
\lVert \boldsymbol{a}_m \rVert_1\sum_{t=1}^{T_m} a_{m,t}\cdot\notag\\
&\hspace{8mm}\Bigg\lVert 
 \big(\nabla F_m(\boldsymbol{X}_{m,r}^{t-1},\boldsymbol{\xi}_{m,r}^{t})-\nabla F_m(\boldsymbol{X}_{m,r}^{t-1}) \big)
\Bigg\rVert^2 \cdot Z_m^r \Bigg] \label{eq: 50d}\\ 
&\leq\mathbb{E}\left[ \frac{1}{K}\sum_{m\in\mathcal{S}_r(\boldsymbol{p})} 
\lVert \boldsymbol{a}_m \rVert_1\sum_{t=1}^{T_m} a_{m,t} \sigma^2 \cdot Z_m^r \right]\label{eq: 50e}\\ 
&=\mathbb{E}\left[ \frac{1}{K}\sum_{m\in\mathcal{S}_r(\boldsymbol{p})} 
\lVert \boldsymbol{a}_m \rVert_1\sum_{t=1}^{T_m} a_{m,t} \sigma^2 (1-q_m) \right]\\
&=K\cdot \frac{1}{K} \sum_{m=1}^M p_m
\lVert \boldsymbol{a}_m \rVert_1\sum_{t=1}^{T_m} a_{m,t} \sigma^2 (1-q_m)\label{eq: 50g}\\ 
&= \sigma^2 \sum_{m=1}^M p_m
\lVert \boldsymbol{a}_m \rVert_1\sum_{t=1}^{T_m} a_{m,t}(1-q_m),
\end{align}   
\end{subequations}
where \eqref{eq: 50b} and \eqref{eq: 50d} are due to Jensen's inequality, \eqref{eq: 50e} is due to Assumption \ref{assump: unbiased}, and \eqref{eq: 50g} is due to independent sampling process. 

\section{Proof of Lemma \ref{lemma: lemma 1}} \label{appx:lemma3}
{For ease of use, we upper bound the aggregated arbitrary non-negative functions with the same structure,}
    \begin{subequations}
    \begin{align}
        &\sum_{m=1}^M p_m(1-q_m) \lVert \boldsymbol{a}_m \rVert_1^2 f_m(x) \notag\\
        &=\sum_{m=1}^Mp_m(1-q_m)\sum_{m=1}^M \gamma_m \lVert \boldsymbol{a}_m \rVert_1^2 f_m(x) \label{eq: agg_f_1}\\
        &=\sum_{m=1}^Mp_m(1-q_m)T_{\mathrm{eff}} \sum_{m=1}^M \Omega_m \lVert \boldsymbol{a}_m \rVert_1 f_m(x)  \label{eq: agg_f_2}\\
        &\leq A\sum_{m=1}^Mp_m(1-q_m)T_{\mathrm{eff}}\sum_{m=1}^M \Omega_m f_m(x),
    \end{align}
    \end{subequations}
     where \eqref{eq: agg_f_1} and \eqref{eq: agg_f_2} are similar to \eqref{eq: E_agg} and \eqref{eq: obj_inconsist_comm_comp}, and $A\triangleq\max_m\{\lVert \boldsymbol{a}_m \rVert_1\}$.

\section{Proof of Lemma \ref{lemma:5}}\label{appx:lemma4}
Let $\mathbf{1}$ denote all-one vector, then $\mathbf{1}\cdot \nabla F_m(\boldsymbol{X}_{r-1})$ is a vector formed by identical entries $ \nabla F_m(\boldsymbol{X}_{r-1})$. The first term in \eqref{eq: second_term_pre} can be bounded by 
\begin{align}
    &\mathbb{E}\left[ \left\lVert \frac{1}{K}\sum_{m\in \Tilde{\mathcal{S}}_r(\boldsymbol{p},\boldsymbol{q})} \boldsymbol{\Delta}_{m,r}\right\rVert^2 \right]
    \hspace{-1mm}=\mathbb{E}\left[ \left\lVert \frac{1}{K}\sum_{m\in \Tilde{\mathcal{S}}_r(\boldsymbol{p},\boldsymbol{q})} \nabla\boldsymbol{F}_m^{r} \boldsymbol{a}_m  \right\rVert^2 \right] \notag\\
    &\leq \underbrace{2 \mathbb{E}\left[ \left\lVert \frac{1}{K}\sum_{m\in \Tilde{\mathcal{S}}_r(\boldsymbol{p},\boldsymbol{q})} \left(\nabla\boldsymbol{F}_m^{r}-\mathbf{1}\cdot \nabla F_m(\boldsymbol{X}_{r-1})\right) \cdot  \boldsymbol{a}_m  \right\rVert^2 \right] }_{T_1} \notag\\
    &+\underbrace{ 2 \mathbb{E}\left[ \left\lVert \frac{1}{K}\sum_{m\in \Tilde{\mathcal{S}}_r(\boldsymbol{p},\boldsymbol{q})} \mathbf{1}\cdot \nabla F_m(\boldsymbol{X}_{r-1}) \cdot \boldsymbol{a}_m  \right\rVert^2 \right] }_{T_2} \label{eq: 51b} 
\end{align}
Furthermore, $T_1$ can be upper bounded by 
\begin{subequations}
    \begin{align}
        &T_1 \notag\\
        &\leq 2 \mathbb{E}\left[ \frac{1}{K}\hspace{-1mm} \sum_{m\in \mathcal{S}_r(\boldsymbol{p})}\hspace{-2mm}\left\lVert \left(\nabla\boldsymbol{F}_m^{r}-\mathbf{1}\cdot \nabla F_m(\boldsymbol{X}_{r-1})\right) \cdot \boldsymbol{a}_m  \right\rVert^2 \cdot Z_m^r\right] \label{eq:52a}\\ 
        &\leq 2 \mathbb{E}\Bigg[   \frac{1}{K}\sum_{m\in \mathcal{S}_r(\boldsymbol{p})} \lVert \boldsymbol{a}_m \rVert_1\sum_{t=1}^{T_m} a_{m,t} \Big\lVert \nabla F_m(\boldsymbol{X}_{m,r}^{t-1}) \notag\\
        &\hspace{4cm} - \nabla F_m(\boldsymbol{X}_{r-1})  \Big\rVert^2 \cdot Z_m^r\Bigg] \label{eq:52b}\\ 
        &\leq 2 \mathbb{E}\left[ \frac{1}{K}\hspace{-1mm}\sum_{m\in \mathcal{S}_r(\boldsymbol{p})} \hspace{-2mm} \lVert \boldsymbol{a}_m \rVert_1\sum_{t=1}^{T_m} a_{m,t} L^2 \left\lVert \boldsymbol{X}_{m,r}^{t-1}- \boldsymbol{X}_{r-1}  \right\rVert^2 \hspace{-0.5mm} \cdot \hspace{-1mm} Z_m^r\hspace{-0.5mm}\right] \label{eq:52c}\\ 
        &= 2L^2 \sum_{m=1}^M p_m (1-q_m) \lVert \boldsymbol{a}_m \rVert_1 \notag\\ 
        &\hspace{2.5cm}\cdot \sum_{t=1}^{T_m} a_{m,t} \mathbb{E}\left[\left\lVert \boldsymbol{X}_{m,r}^{t-1}- \boldsymbol{X}_{r-1}  \right\rVert^2 \right],  \label{eq:52d}
    \end{align}
    \label{eq: bound_T1}
\end{subequations}
where \eqref{eq:52a} and \eqref{eq:52b} are due to Jensen's Inequality,  \eqref{eq:52c} is due to Assumption \ref{assump: Smooth},  \eqref{eq:52d} is due to the independent client sampling and that $\mathbb{E}[Z_m^r]=1-q_m$.  

Moreover, $T_2$ can be upper bounded by 
\begin{subequations}
    \begin{align}
        &T_2 
        = 2 \mathbb{E}\left[ \left\lVert \frac{1}{K}\sum_{m\in \Tilde{\mathcal{S}}_r(\boldsymbol{p},\boldsymbol{q})} \mathbf{1}\cdot \nabla F_m(\boldsymbol{X}_{r-1}) \cdot \boldsymbol{a}_m  \right\rVert^2 \right]\\
        &= 2 \mathbb{E}\left[ \left\lVert \frac{1}{K}\sum_{m\in \mathcal{S}_r(\boldsymbol{p})} \mathbf{1}\cdot \nabla F_m(\boldsymbol{X}_{r-1}) \cdot \boldsymbol{a}_m  \cdot Z_m^r \right\rVert^2 \right]\\
        &\leq 2 \mathbb{E}\left[ \frac{1}{K}\sum_{m\in \mathcal{S}_r(\boldsymbol{p})} \left\lVert  \mathbf{1}\cdot \nabla F_m(\boldsymbol{X}_{r-1}) \cdot \boldsymbol{a}_m  \right\rVert^2 \cdot Z_m^r \right]\label{eq:53c}\\
        &= 2 \mathbb{E}\hspace{-1mm}\left[ \frac{1}{K}\hspace{-1mm}\sum_{m\in \mathcal{S}_r(\boldsymbol{p})} \hspace{-1mm}\hspace{-1mm}\left\lVert  \mathbf{1}\cdot \nabla F_m(\boldsymbol{X}_{r-1}) \cdot \boldsymbol{a}_m  \right\rVert^2 (1-q_m) \right]\\
        &= 2 \mathbb{E}\hspace{-1mm}\left[ \frac{1}{K}\hspace{-1mm}\sum_{m\in \mathcal{S}_r(\boldsymbol{p})} \left\lVert \sum_{t=1}^{T_m} a_{m,t}  \nabla F_m(\boldsymbol{X}_{r-1}) \right\rVert^2\hspace{-1mm}\hspace{-1mm} (1-q_m) \right]\\
        &= 2 \mathbb{E}\left[ \frac{1}{K}\hspace{-1mm}\sum_{m\in \mathcal{S}_r(\boldsymbol{p})}\hspace{-1mm} \lVert \boldsymbol{a}_m \rVert_1^2  \cdot \left\lVert  \nabla F_m(\boldsymbol{X}_{r-1}) \right\rVert^2 \hspace{-1mm}(1-q_m) \right] \label{eq:53f} \\
        &= 2 \sum_{m=1}^{M} p_m(1-q_m) \lVert \boldsymbol{a}_m \rVert_1^2 \cdot \mathbb{E}\left[ \left\lVert  \nabla F_m(\boldsymbol{X}_{r-1}) \right\rVert^2 \right] \label{eq:53g}\\
        &\leq 2A\sum_{m=1}^Mp_m(1-q_m)T_{\mathrm{eff}}\sum_{m=1}^M \Omega_m \mathbb{E}\left[ \left\lVert  \nabla F_m(\boldsymbol{X}_{r-1}) \right\rVert^2 \right] \label{eq:53h}\\
        &\leq 2A\sum_{m=1}^Mp_m(1-q_m)T_{\mathrm{eff}} \left(\beta^2\mathbb{E}\left[ \left\lVert  \nabla \Tilde{F}(\boldsymbol{X}_{r-1}) \right\rVert^2 \right]+\kappa^2\right), \label{eq:53i}
    \end{align}\label{eq: bound_T2}
\end{subequations}
where \eqref{eq:53c} follows Jensen's Inequality, \eqref{eq:53f} owes to the fact that $\nabla F_m(\boldsymbol{X}_{r-1})$ is irrelevant to index $t$, \eqref{eq:53h} follows Lemma \ref{lemma: lemma 1}, \eqref{eq:53i} follows Assumption \ref{assump: Dissimilarity}.   

By substituting \eqref{eq: bound_T1} and \eqref{eq: bound_T2} into \eqref{eq: 51b}, Lemma \ref{lemma:5} is derived.

\section{Proof of Lemma \ref{lemma: lemma 2}} \label{appx:lemma5}
On each device, the accumulation of the expected squared $\ell_2$-norm of the local model updates is upper bounded by
    \begin{subequations}
    \begin{align}
        &\sum_{t=1}^{T_m} a_{m,t}\mathbb{E}\left[ \left\lVert  \boldsymbol{X}_{r-1}-\boldsymbol{X}_{m,r}^{t-1}  \right\rVert^2 \right]\\
        &=\eta^2 \sum_{t=1}^{T_m} a_{m,t}\mathbb{E}\left[ \left\lVert  \sum_{s=1}^{t-1} a_{m,s} \nabla F_m(\boldsymbol{X}_{m,r}^{s-1})  \right\rVert^2 \right]\\
        &\leq\eta^2 \sum_{t=1}^{T_m} a_{m,t} \sum_{s=1}^{t-1} a_{m,s} \sum_{s=1}^{t-1} a_{m,s}\mathbb{E}\left[ \left\lVert \nabla F_m(\boldsymbol{X}_{m,r}^{s-1})  \right\rVert^2 \right]\\
        &\leq\eta^2 \left(\sum_{t=1}^{T_m} a_{m,t}\right)^2 \sum_{s=1}^{T_m-1} a_{m,s}\mathbb{E}\left[ \left\lVert \nabla F_m(\boldsymbol{X}_{m,r}^{s-1})  \right\rVert^2 \right]\\
        &\leq\eta^2 \lVert \boldsymbol{a}_m \rVert_1^2 \sum_{s=1}^{T_m-1} a_{m,s}\mathbb{E}\left[ \left\lVert \nabla F_m(\boldsymbol{X}_{m,r}^{s-1})  \right\rVert^2 \right]\\
        &\leq \eta^2 \lVert \boldsymbol{a}_m \rVert_1^2 \sum_{s=1}^{T_m-1} a_{m,s} \bigg(2\mathbb{E}\Big[ \lVert \nabla F_m(\boldsymbol{X}_{m,r}^{s-1})\notag\\
        &\hspace{5mm} -\nabla F_m(\boldsymbol{X}_{r-1}) \rVert^2 \Big]+2\mathbb{E}\Big[ \left\lVert \nabla F_m(\boldsymbol{X}_{r-1})  \right\rVert^2 \Big]\bigg)\\
        &\leq 2\eta^2 \lVert \boldsymbol{a}_m \rVert_1^2 \sum_{s=1}^{T_m-1} a_{m,s} \Big(L^2\mathbb{E}\left[ \left\lVert \boldsymbol{X}_{m,r}^{s-1}-\boldsymbol{X}_{r-1} \right\rVert^2 \right]\notag\\ 
        &\hspace{5mm}+\mathbb{E}\left[ \left\lVert \nabla F_m(\boldsymbol{X}_{r-1})  \right\rVert^2 \right]\Big).
    \end{align}
    \end{subequations}  
    Both sides in the above inequality contains $\sum_{t=1}^{T_m} a_{m,t}\mathbb{E}\left[ \left\lVert  \boldsymbol{X}_{r-1}-\boldsymbol{X}_{m,r}^{t-1}  \right\rVert^2 \right]$, by minor rearrangement, we have
    \begin{align}
        &\left(1-2\eta^2L^2\lVert \boldsymbol{a}_m \rVert_1^2\right)\sum_{t=1}^{T_m} a_{m,t}\mathbb{E}\left[ \left\lVert  \boldsymbol{X}_{r-1}-\boldsymbol{X}_{m,r}^{t-1}  \right\rVert^2 \right]\notag\\
        &\leq 2\eta^2\lVert \boldsymbol{a}_m \rVert_1^3\cdot\mathbb{E}\left[\lVert \nabla F_m\left( \boldsymbol{X}_{r-1}\right) \rVert^2\right].
    \end{align}
    Then, we can easily derive Lemma \ref{lemma: lemma 2}.

\section{Proof of Lemma \ref{Lemma: decent_surrogate}} \label{appx:lemma6}
By Assumption \ref{assump: Smooth},  the following inequality concerning global objective function hold: 
\begin{align}
    &\mathbb{E}\left[ \Tilde{F}(\boldsymbol{X}_{r}) \right]-\mathbb{E}\left[ \Tilde{F}(\boldsymbol{X}_{r-1}) \right] \notag\\
    &\leq \mathbb{E}\left[ \left\langle\nabla\Tilde{F}(\boldsymbol{X}_{r-1}), \boldsymbol{X}_{r}-\boldsymbol{X}_{r-1}\right\rangle \right]
    +\frac{L}{2} \mathbb{E}\left[ \left\lVert \boldsymbol{X}_{r}-\boldsymbol{X}_{r-1}\right\rVert^2 \right].
    \label{eq: smooth_lossF}
\end{align}
Furthermore, it holds that 
\begin{align}
    &\frac{L}{2} \mathbb{E}\left[ \left\lVert \boldsymbol{X}_{r}-\boldsymbol{X}_{r-1}\right\rVert^2 \right]
    = \frac{\eta^2 L}{2} \mathbb{E}\left[ \left\lVert \frac{1}{K}\sum_{m\in \Tilde{\mathcal{S}}_r(\boldsymbol{p},\boldsymbol{q})} \boldsymbol{\Delta}_{m,r}^{\boldsymbol{\xi}}\right\rVert^2 \right] \notag\\
    &= \frac{\eta^2 L}{2} \Bigg\{\mathbb{E}\left[ \left\lVert \frac{1}{K}\sum_{m\in \Tilde{\mathcal{S}}_r(\boldsymbol{p},\boldsymbol{q})} \boldsymbol{\Delta}_{m,r}\right\rVert^2 \right] \notag\\
    &+\mathbb{E}\left[ \left\lVert \frac{1}{K}\sum_{m\in \Tilde{\mathcal{S}}_r(\boldsymbol{p},\boldsymbol{q})} \left(\boldsymbol{\Delta}_{m,r}^{\boldsymbol{\xi}}
    - \boldsymbol{\Delta}_{m,r}\right)\right\rVert^2 \right]
    \Bigg\},\label{eq: second_term_pre}  
\end{align}  
where \eqref{eq: second_term_pre} is due to unbiased estimation.

\noindent\textbf{Immediate Result.}
By substituting Lemma 3--6 into \eqref{eq: smooth_lossF}, the immediate result is obtained as follows, 
\begin{subequations}
\begin{align}
    &\mathbb{E}\left[ \Tilde{F}(\boldsymbol{X}_{r}) \right]
    -\mathbb{E}\left[ \Tilde{F}(\boldsymbol{X}_{r-1}) \right]\\
    &\leq \frac{1}{2}\eta_{\mathrm{eff}} T_{\mathrm{eff}} \sum_{m=1}^M \Omega_m \frac{L^2}{\lVert\boldsymbol{a}_m \rVert_1} \sum_{t=1}^{T_m} a_{m,t} \mathbb{E}\left[ \left\lVert \boldsymbol{X}_{r-1}-\boldsymbol{X}_{m,r}^{t-1} \right\rVert^2 \right] \notag\\
    &-\frac{1}{2}\eta_{\mathrm{eff}} T_{\mathrm{eff}}\mathbb{E}\left[ \left\lVert  \nabla \Tilde{F}(\boldsymbol{X}_{r-1}) \right\rVert^2 \right] \notag\\
    &-\frac{1}{2}\eta_{\mathrm{eff}} T_{\mathrm{eff}}\mathbb{E}\left[ \left\lVert \sum_{m=1}^M\Omega_m \nabla\Bar{\boldsymbol{F}}_m^{r} \right\rVert^2 \right] \notag \\
    &+\frac{1}{2}\eta^2 L\sigma^2 \sum_{m=1}^M p_m(1-q_m)\lVert \boldsymbol{a}_m\rVert_1^2 +\notag\\
    &\eta^2L^3 \sum_{m=1}^M p_m(1-q_m) \lVert \boldsymbol{a}_m\rVert_1 \sum_{t=1}^{T_m} a_{m,t} \mathbb{E}\left[ \left\lVert \boldsymbol{X}_{r-1}-\boldsymbol{X}_{m,r}^{t-1} \right\rVert^2 \right]\notag\\
    &+\eta\eta_{\mathrm{eff}}LAT_{\mathrm{eff}} \left(\beta^2\mathbb{E}\left[ \left\lVert  \nabla \Tilde{F}(\boldsymbol{X}_{r-1}) \right\rVert^2 \right]+\kappa^2\right)\\
    &\leq \frac{1}{2}\eta_{\mathrm{eff}} T_{\mathrm{eff}} \sum_{m=1}^M \Omega_m \frac{L^2}{\lVert\boldsymbol{a}_m \rVert_1} \sum_{t=1}^{T_m} a_{m,t} \mathbb{E}\left[ \left\lVert \boldsymbol{X}_{r-1}-\boldsymbol{X}_{m,r}^{t-1} \right\rVert^2 \right] \notag\\
    &-\frac{1}{2}\eta_{\mathrm{eff}} T_{\mathrm{eff}}\mathbb{E}\left[ \left\lVert  \nabla \Tilde{F}(\boldsymbol{X}_{r-1}) \right\rVert^2 \right] \notag\\
    &+\frac{1}{2}\eta^2 L\sigma^2 \sum_{m=1}^M p_m(1-q_m)\lVert \boldsymbol{a}_m\rVert_1^2+ \notag\\
    &\eta^2L^3 \sum_{m=1}^M p_m(1-q_m) \lVert \boldsymbol{a}_m\rVert_1 \sum_{t=1}^{T_m} a_{m,t} \mathbb{E}\left[ \left\lVert \boldsymbol{X}_{r-1}-\boldsymbol{X}_{m,r}^{t-1} \right\rVert^2 \right]\notag\\
    &+\eta\eta_{\mathrm{eff}}LAT_{\mathrm{eff}} \left(\beta^2\mathbb{E}\left[ \left\lVert  \nabla \Tilde{F}(\boldsymbol{X}_{r-1}) \right\rVert^2 \right]+\kappa^2\right).
\end{align}\label{eq: immediate_result}
\end{subequations}
\noindent\textbf{Final Proof.}
Substitute Lemma \ref{lemma: lemma 2} into \eqref{eq: immediate_result}, we get 
\begin{subequations}
\begin{align}
    &\mathbb{E}\left[ \Tilde{F}(\boldsymbol{X}_{r}) \right]
    -\mathbb{E}\left[ \Tilde{F}(\boldsymbol{X}_{r-1}) \right]\\
    &\leq \eta_{\mathrm{eff}} T_{\mathrm{eff}} \sum_{m=1}^M \Omega_m \frac{\eta^2L^2\lVert \boldsymbol{a}_m \rVert_1^2}{1-2\eta^2L^2\lVert \boldsymbol{a}_m \rVert_1^2} \mathbb{E}\left[\lVert \nabla F_m\left( \boldsymbol{X}_{r-1}\right) \rVert^2\right] \notag\\
    &-\frac{1}{2}\eta_{\mathrm{eff}} T_{\mathrm{eff}}\mathbb{E}\left[ \left\lVert  \nabla \Tilde{F}(\boldsymbol{X}_{r-1}) \right\rVert^2 \right] \notag\\
    &+\frac{1}{2}\eta^2 L\sigma^2 \sum_{m=1}^M p_m(1-q_m)\lVert \boldsymbol{a}_m\rVert_1^2
    +\eta^2L^3 \cdot \notag\\
    &\sum_{m=1}^M p_m(1-q_m) \lVert \boldsymbol{a}_m\rVert_1^2 \frac{2\eta^2\lVert \boldsymbol{a}_m \rVert_1^2}{1-2\eta^2L^2\lVert \boldsymbol{a}_m \rVert_1^2} \mathbb{E}\left[\lVert \nabla F_m\left( \boldsymbol{X}_{r-1}\right) \rVert^2\right]\notag\\
    &+\eta\eta_{\mathrm{eff}}LAT_{\mathrm{eff}} \left(\beta^2\mathbb{E}\left[ \left\lVert  \nabla \Tilde{F}(\boldsymbol{X}_{r-1}) \right\rVert^2 \right]+\kappa^2\right)\\
    &\leq \eta_{\mathrm{eff}} T_{\mathrm{eff}} \frac{\eta^2L^2 A^2}{1-2\eta^2L^2A^2} \sum_{m=1}^M \Omega_m \mathbb{E}\left[\lVert \nabla F_m\left( \boldsymbol{X}_{r-1}\right) \rVert^2\right] \notag\\
    &-\frac{1}{2}\eta_{\mathrm{eff}} T_{\mathrm{eff}}\mathbb{E}\left[ \left\lVert  \nabla \Tilde{F}(\boldsymbol{X}_{r-1}) \right\rVert^2 \right] \notag\\
    &+\frac{1}{2}\eta^2 L\sigma^2 \sum_{m=1}^M p_m(1-q_m)\lVert \boldsymbol{a}_m\rVert_1^2
    +\eta^2L \frac{2\eta^2 L^2A^2}{1-2\eta^2L^2A^2} \cdot \notag\\
    &\sum_{m=1}^M p_m(1-q_m) \lVert \boldsymbol{a}_m\rVert_1^2 \mathbb{E}\left[\lVert \nabla F_m\left( \boldsymbol{X}_{r-1}\right) \rVert^2\right]\notag\\
    &+\eta\eta_{\mathrm{eff}}LAT_{\mathrm{eff}} \left(\beta^2\mathbb{E}\left[ \left\lVert  \nabla \Tilde{F}(\boldsymbol{X}_{r-1}) \right\rVert^2 \right]+\kappa^2\right) \label{eq: 57b}\\
    &\leq \eta_{\mathrm{eff}} T_{\mathrm{eff}} \frac{\eta^2L^2 A^2}{1-2\eta^2L^2A^2} \left(\beta^2\mathbb{E}\left[ \left\lVert  \nabla \Tilde{F}(\boldsymbol{X}_{r-1}) \right\rVert^2 \right]+\kappa^2\right) \notag\\
    &-\frac{1}{2}\eta_{\mathrm{eff}} T_{\mathrm{eff}}\mathbb{E}\left[ \left\lVert  \nabla \Tilde{F}(\boldsymbol{X}_{r-1}) \right\rVert^2 \right]
    +\frac{1}{2}\eta \eta_{\mathrm{eff}} LA^2\sigma^2 \notag\\
    &+\eta\eta_{\mathrm{eff}}LAT_{\mathrm{eff}} \frac{2\eta^2 L^2A^2}{1-2\eta^2L^2A^2} \left(\beta^2\mathbb{E}\left[ \left\lVert  \nabla \Tilde{F}(\boldsymbol{X}_{r-1}) \right\rVert^2 \right]+\kappa^2\right)    \notag\\
    &+\eta\eta_{\mathrm{eff}}LAT_{\mathrm{eff}} \left(\beta^2\mathbb{E}\left[ \left\lVert  \nabla \Tilde{F}(\boldsymbol{X}_{r-1}) \right\rVert^2 \right]+\kappa^2\right) \label{eq: 57c}\\
    &\leq \eta_{\mathrm{eff}} T_{\mathrm{eff}} \frac{\eta^2L^2 A^2}{1-2\eta^2L^2A^2} \left(\beta^2\mathbb{E}\left[ \left\lVert  \nabla \Tilde{F}(\boldsymbol{X}_{r-1}) \right\rVert^2 \right]+\kappa^2\right) \notag\\
    &-\frac{1}{2}\eta_{\mathrm{eff}} T_{\mathrm{eff}}\mathbb{E}\left[ \left\lVert  \nabla \Tilde{F}(\boldsymbol{X}_{r-1}) \right\rVert^2 \right]
    +\frac{1}{2}\eta \eta_{\mathrm{eff}} LA^2\sigma^2 \notag\\
    &+\eta\eta_{\mathrm{eff}}LAT_{\mathrm{eff}} \frac{2\eta^2 L^2A^2}{1-2\eta^2L^2A^2} \left(\beta^2\mathbb{E}\left[ \left\lVert  \nabla \Tilde{F}(\boldsymbol{X}_{r-1}) \right\rVert^2 \right]+\kappa^2\right)    \notag\\
    &+\eta\eta_{\mathrm{eff}}LAT_{\mathrm{eff}} \left(\beta^2\mathbb{E}\left[ \left\lVert  \nabla \Tilde{F}(\boldsymbol{X}_{r-1}) \right\rVert^2 \right]+\kappa^2\right)
    \label{eq: 57d}\\
    &\leq \eta_{\mathrm{eff}} T_{\mathrm{eff}} \beta^2 \Big(-\frac{1}{2} + \frac{\eta^2L^2 A^2}{1-2\eta^2L^2A^2} + 
    \eta LA \frac{2\eta^2 L^2A^2}{1-2\eta^2L^2A^2}\notag\\
    &+\eta LA
    \Big)\mathbb{E}\left[ \left\lVert  \nabla \Tilde{F}(\boldsymbol{X}_{r-1}) \right\rVert^2 \right]+\eta_{\mathrm{eff}}T_{\mathrm{eff}} \frac{\eta LA^2}{2T_{\mathrm{eff}}} \sigma^2 \notag\\
    &+ \eta_{\mathrm{eff}} T_{\mathrm{eff}} \kappa^2 \left( \frac{\eta^2L^2 A^2}{1-2\eta^2L^2A^2} + 
    \eta LA \frac{2\eta^2 L^2A^2}{1-2\eta^2L^2A^2}+\eta LA
    \right) \label{eq: 57e}, 
\end{align}
\end{subequations}
where \eqref{eq: 57b} owes to Lemma \ref{lemma: lemma 1}, \eqref{eq: 57c} follows the definition of $A$, \eqref{eq: 57d} follows Assumption \ref{assump: Dissimilarity}. After minor re-arrangement, we derive Lemma \ref{Lemma: decent_surrogate}.

\vspace{-0.5em}
\section{Proof of Lemma \ref{Lemma: decent_true}}\label{appx:lemma7}
By Cauchy–Schwarz Inequality and Assumption \ref{assump: Dissimilarity}, it holds that
\begin{subequations}
\begin{align}
    &\left \lVert \nabla F(x)- \nabla\Tilde{F}(x) \right\rVert^2
    \hspace{-1mm}= \left \lVert \sum_{m=1}^M \left(\omega_m-\Omega_m\right)
    \nabla F_m(x) \right\rVert^2\\ 
    &=\left \lVert \sum_{m=1}^M \frac{\omega_m-\Omega_m}{\sqrt{\Omega_m}}\cdot\sqrt{\Omega_m}
    \cdot\nabla F_m(x) \right\rVert^2\\
    &\leq \left(\sum_{m=1}^M \frac{\left(\omega_m-\Omega_m\right)^2}{\Omega_m}\right) 
    \cdot \left(  \sum_{m=1}^M \Omega_m \left \lVert \nabla F_m(x) \right\rVert^2 \right) \label{eq: 577c}  \\
    &\leq \chi_{\boldsymbol{\omega}\Vert\boldsymbol{\Omega} }^2  \cdot  \left( \beta^2\left\lVert  \nabla \Tilde{F}(x) \right\rVert^2+\kappa^2 \right) \label{eq: 577d}\\
    &\leq \chi_{\boldsymbol{\omega}\Vert\boldsymbol{\Omega} }^2  \cdot  \left( \beta\left\lVert  \nabla \Tilde{F}(x) \right\rVert+\kappa \right)^2.  \label{eq: 577e}
\end{align}
\end{subequations}
Take the square root of both sides of \eqref{eq: 577e} and let $x=\boldsymbol{X}_{r}$, Lemma \ref{Lemma: decent_true} is acquired.

\vspace{-0.5em}
\section{Proof for Theorem \ref{theo: Convergence of the Surrogate}}\label{appx: theo1}
If the choice of $\eta$ satisfies $ \frac{\eta^2L^2 A^2}{1-2\eta^2L^2A^2} + 
    \eta LA \frac{2\eta^2 L^2A^2}{1-2\eta^2L^2A^2}+\eta LA\leq \frac{1}{4}$, a simple variation of \eqref{eq: 57e} is given by 
\begin{align}
    \mathbb{E}\left[ \left\lVert  \nabla \Tilde{F}(\boldsymbol{X}_{r-1}) \right\rVert^2 \right]
    &\leq \frac{4}{\eta_{\mathrm{eff}}T_{\mathrm{eff}}}\left( \mathbb{E}\left[ \Tilde{F}(\boldsymbol{X}_{r-1}) \right]-\mathbb{E}\left[ \Tilde{F}(\boldsymbol{X}_{r}) \right]
    \right) \notag\\ 
    &+\eta_{\mathrm{eff}}T_{\mathrm{eff}} \frac{\kappa^2}{\beta^2}
    +\frac{2\eta LA^2}{T_{\mathrm{eff}}} \sigma^2.
\end{align}
Take the average over all training rounds, we obtain
\begin{align}
    \frac{1}{R}\sum_{r=1}^R\mathbb{E}\left[ \left\lVert  \nabla \Tilde{F}(\boldsymbol{X}_{r-1}) \right\rVert^2 \right]
    &\leq \frac{4}{R\eta_{\mathrm{eff}}T_{\mathrm{eff}}}\left( \mathbb{E}\left[ \Tilde{F}(\boldsymbol{X}_{0}) \right]-F^\star
    \right) \notag\\
    &+\eta_{\mathrm{eff}}T_{\mathrm{eff}} \frac{\kappa^2}{\beta^2}
    +\frac{2\eta LA^2}{T_{\mathrm{eff}}} \sigma^2.
\end{align}
Now we complete the proof for Theorem \ref{theo: Convergence of the Surrogate}.

\vspace{-0.5em}
\section{Proof for Theorem \ref{theo:converge_bound_true}}\label{appx: theo2}
To begin with, we introduce the construction of the convergence bound of $F(x)$ in terms of $\Tilde{F}(x)$. 

By the property of $\ell_2$-norm, it also holds that 
\begin{align}
    &\left \lVert \nabla F(x)- \nabla\Tilde{F}(x) \right\rVert^2 
    \geq \left(\left\lVert \nabla F(x) \right\rVert-\left\lVert \nabla\Tilde{F}(x) \right\rVert\right)^2 \notag\\
    &= \hspace{-0.5mm}\left\lVert \nabla F(x) \right\rVert^2\hspace{-0.5mm}+\left\lVert \nabla\Tilde{F}(x) \right\rVert^2
    \hspace{-1mm}-\hspace{-0mm}2\left\lVert \nabla F(x) \right\rVert \cdot \left\lVert \nabla\Tilde{F}(x) \right\rVert.\label{eq:56}
\end{align}

Substitute \eqref{eq: 57e} into \eqref{eq:56}, by minor re-arrangement, we get
\begin{align}
     &\left\lVert \nabla F(x) \right\rVert^2
     -2 \left\lVert \nabla F(x) \right\rVert \cdot \left\lVert \nabla\Tilde{F}(x) \right\rVert \notag\\
     &\leq \chi_{\boldsymbol{\omega}\Vert\boldsymbol{\Omega} }^2  \left( \beta^2\left\lVert  \nabla \Tilde{F}(x) \right\rVert^2+\kappa^2 \right)
    - \left\lVert \nabla\Tilde{F}(x) \right\rVert^2.
    \label{eq: F_and_tilde{F}}
\end{align}
Take $x=\boldsymbol{X}_{r}$, we further have 
\begin{align}
     &\frac{1}{R}\sum_{r=1}^R \left\lVert \nabla F(\boldsymbol{X}_{r}) \right\rVert^2
     -2 \cdot \frac{1}{R}\sum_{r=1}^R \left\lVert \nabla F(\boldsymbol{X}_{r}) \right\rVert \cdot \left\lVert \nabla\Tilde{F}(\boldsymbol{X}_{r}) \right\rVert\notag\\
     &\leq \left(\chi_{\boldsymbol{\omega}\Vert\boldsymbol{\Omega} }^2 \beta^2-1 \right) \cdot \frac{1}{R}\sum_{r=1}^R\left\lVert  \nabla \Tilde{F}(\boldsymbol{X}_{r}) \right\rVert^2+ \chi_{\boldsymbol{\omega}\Vert\boldsymbol{\Omega} }^2\kappa^2,
     \label{eq: 60}
\end{align}
The second term on the left-hand side in \eqref{eq: 60} can be further lower bounded by
\begin{subequations}
\begin{align}
     &2 \cdot \frac{1}{R}\sum_{r=1}^R \left\lVert \nabla F(\boldsymbol{X}_{r}) \right\rVert \cdot \left\lVert \nabla\Tilde{F}(\boldsymbol{X}_{r}) \right\rVert \\
     &=2 \cdot \frac{1}{R} \sqrt{\left(\sum_{r=1}^R \left\lVert \nabla F(\boldsymbol{X}_{r}) \right\rVert \cdot \left\lVert \nabla\Tilde{F}(\boldsymbol{X}_{r}) \right\rVert\right)^2} \\
     &\leq  2 \cdot \frac{1}{R}\sqrt{\sum_{r=1}^R \left\lVert \nabla F(\boldsymbol{X}_{r}) \right\rVert^2} \cdot \sqrt{\sum_{r=1}^R\left\lVert \nabla\Tilde{F}(\boldsymbol{X}_{r}) \right\rVert^2}\\
     &= 2 \sqrt{\frac{1}{R}\sum_{r=1}^R \left\lVert \nabla F(\boldsymbol{X}_{r}) \right\rVert^2} \cdot \sqrt{\frac{1}{R}\sum_{r=1}^R\left\lVert \nabla\Tilde{F}(\boldsymbol{X}_{r}) \right\rVert^2}, 
\end{align}     \label{eq: 61}
\end{subequations}
where Cauchy–Schwarz Inequality applies. Combine \eqref{eq: 61} with \eqref{eq: 60} and \eqref{eq: theo1}, we get 
\begin{align}
    &\frac{1}{R}\sum_{r=1}^R \left\lVert \nabla F(\boldsymbol{X}_{r}) \right\rVert^2
     -2 \sqrt{\frac{1}{R}\sum_{r=1}^R \left\lVert \nabla F(\boldsymbol{X}_{r}) \right\rVert^2} \cdot \sqrt{\epsilon_{\mathrm{opt}}} \notag\\
     &\leq \left(\chi_{\boldsymbol{\omega}\Vert\boldsymbol{\Omega} }^2 \beta^2-1 \right) \cdot\epsilon_{\mathrm{opt}}+ \chi_{\boldsymbol{\omega}\Vert\boldsymbol{\Omega} }^2\kappa^2. 
\end{align}
Taking the limit of both sides and using the fact that $\lim_{R\rightarrow \infty}\epsilon_{\mathrm{opt}}=0$, we finally arrive
\begin{align}
    \lim_{R\rightarrow \infty} \frac{1}{R}\sum_{r=1}^R \left\lVert \nabla F(\boldsymbol{X}_{r}) \right\rVert^2
    \leq \chi_{\boldsymbol{\omega}\Vert\boldsymbol{\Omega} }^2\kappa^2. 
\end{align}
The construction of \eqref{eq:converge_bound_true} for the true objective function is thereby completed. 

Due to unbiased estimation in Assumption \ref{assump: unbiased}, we have the following inequality,
\begin{subequations}
\begin{align}
    &\lim_{R\rightarrow +\infty}\frac{1}{R}\sum_{r=1}^R \mathbb{E}_\xi\left[\left\lVert \nabla F(\boldsymbol{X}_{r}) \right\rVert^2\right] \notag\\
    &=\lim_{R\rightarrow +\infty}\frac{1}{R}\sum_{r=1}^R \Big(\left\lVert \nabla F(\boldsymbol{X}_{r}) \right\rVert^2\notag\\
    & \hspace{2cm}+\mathbb{E}_{\xi}\left[\left\lVert \nabla F(\boldsymbol{X}_{r}) - \nabla F(\boldsymbol{X}_{r}\vert \xi) \right\rVert^2\right]\Big) \\
    &=\lim_{R\rightarrow +\infty}\frac{1}{R}\sum_{r=1}^R \Bigg(\left\lVert \nabla F(\boldsymbol{X}_{r}) \right\rVert^2 \notag\\
    & \hspace{0cm}+\mathbb{E}_{\xi}\left[\left\lVert \sum_{m=1}^M\omega_m\left(\nabla F_m(\boldsymbol{X}_{r}) - \nabla F_m(\boldsymbol{X}_{r}\vert \xi)\right) \right\rVert^2\right]\Bigg)\\
    &\leq\lim_{R\rightarrow +\infty}\frac{1}{R}\sum_{r=1}^R \Big(\left\lVert \nabla F(\boldsymbol{X}_{r}) \right\rVert^2 \notag\\
    &+\sum_{m=1}^M\omega_m\mathbb{E}_{\xi}\left[\left\lVert \left(\nabla F_m(\boldsymbol{X}_{r}) - \nabla F_m(\boldsymbol{X}_{r}\vert \xi)\right) \right\rVert^2\right]\Big)\\
    &\leq \chi_{\boldsymbol{\omega}\Vert\boldsymbol{\Omega}}^2 \kappa^2+\sigma^2.
\end{align}
\end{subequations}
Hence, the proof of \eqref{eq:converge_bound_true2} is completed. 
\vspace{-0.5em}
\section{Achievability of Theorem \ref{theo:converge_bound_true}}\label{Appx: achieve}
Next, we show that the acquired convergence bound is tight by showing that it is achievable. Consider the example given in Section \ref{sec: Phenomenon of Interest}, specifically when $M=2$, $d=1$, $e_m=(-1)^m e$. Then, we have
\begin{subequations}
\begin{align}
    \sum_{m=1}^M \Omega_m\lVert \nabla F_m(x) \rVert^2
    &=\Omega_1(x+e)^2+\Omega_2(x-e)^2\\
    &=x^2+2(\Omega_1-\Omega_2)xe +e^2,
    \label{eq: 64}
\end{align}
\end{subequations}
and
\begin{subequations}
\begin{align}
    \lVert  \nabla F(x) \rVert^2
    &=\left\lVert \sum_{m=1}^M \Omega_m\nabla F_m(x) \right\rVert^2\\
    &=\left(\Omega_1(x+e)+\Omega_2(x-e)\right)^2\\
    &=x^2+2(\Omega_1-\Omega_2)xe+(\Omega_1-\Omega_2)^2e^2.
    \label{eq: 65}
\end{align}
\end{subequations}
Compare \eqref{eq: 64} and \eqref{eq: 65} with Assumption \ref{assump: Dissimilarity}, we can derive $\beta^2=1$, $\kappa^2=e^2-(\Omega_1-\Omega_2)^2e^2=(1+\Omega_1-\Omega_2)(1-\Omega_1+\Omega_2)e^2=4\Omega_1\Omega_2e^2$ in this example. Furthermore, $\sigma^2=0$ by its definition.

Equation \eqref{eq: converge_point_sgd} derives the expected convergence point of \texttt{FedAvg} with IS. Combined with Theorem~\ref{theo: Convergence of the Surrogate}, which guarantees its convergence to a stationary point, this implies that the convergence point is exactly
\begin{align}
      \boldsymbol{X}_{R}=\frac{(1-q_2)T_2 e-(1-q_1)T_1 e}{(1-q_1)T_1+(1-q_2)T_2}.
\end{align}
Accordingly, we have
\begin{subequations}
\begin{align}
    &\Omega_1=\frac{(1-q_1)T_1}{(1-q_2)T_2+(1-q_1)T_1}\\
    &\Omega_2=\frac{(1-q_2)T_2}{(1-q_2)T_2+(1-q_1)T_1}.
\end{align}\label{eq: weights_example}
\end{subequations}
As a result, we have
\begin{subequations}
\begin{align}
   \lim_{R\rightarrow \infty}  \left\lVert\nabla F(\boldsymbol{X}_{R})\right\lVert^2
    &=\lim_{R\rightarrow \infty} \left( \frac{1}{2}(\boldsymbol{X}_{R}-e)+\frac{1}{2}(\boldsymbol{X}_{R}+e) \right)^2\\
    =&\lim_{R\rightarrow \infty} \left(\boldsymbol{X}_{R}  \right)^2\\
    =&\left(\frac{(1-q_2)T_2 -(1-q_1)T_1}{(1-q_1)T_1+(1-q_2)T_2}\right)^2e^2\\
    =&\frac{((1-q_2)T_2 -(1-q_1)T_1)^2}{4(1-q_1)(1-q_2)T_1T_2}\kappa^2.
    \label{eq: lim_example}
\end{align}
\end{subequations}
In addition, we have 
\begin{subequations}
\begin{align}
    \chi_{\boldsymbol{\omega}\Vert\boldsymbol{\Omega} }^2
    &=\frac{(\Omega_1-\frac{1}{2})^2}{\Omega_1}+\frac{(\Omega_2-\frac{1}{2})^2}{\Omega_2}\\
    &=-1+\frac{1}{4\Omega_1\Omega_2},
\end{align}\label{eq: chi_square}
\end{subequations}
where $\Omega_1+\Omega_2=1$ is used in the derivation. By plugging \eqref{eq: weights_example} into \eqref{eq: chi_square}, we further get
\begin{align}
     \chi_{\boldsymbol{\omega}\Vert\boldsymbol{\Omega} }^2=\frac{((1-q_2)T_2 -(1-q_1)T_1)^2}{4(1-q_1)(1-q_2)T_1T_2}. 
     \label{eq:chi_square_2}
\end{align}
By comparing \eqref{eq:chi_square_2} with \eqref{eq: lim_example}, we prove
\begin{align}
     \lim_{R\rightarrow \infty}  \left\lVert\nabla F(\boldsymbol{X}_{R})\right\lVert^2= \chi_{\boldsymbol{\omega}\Vert\boldsymbol{\Omega} }^2,
\end{align}
indicating Theorem \ref{theo:converge_bound_true} is achievable. 

\vspace{-0.5em}
\section{Proof for Theorem \ref{theo: FedACS}}\label{appx: theo 4}
In the case of \texttt{FedACS}, the aggregation weights in each round is unbiased, i.e., $\Omega_m^{(r)}=\omega_m^{(r)}$. The surrogate function $\Tilde{F}(\boldsymbol{X})$ is the same as $F(\boldsymbol{X})$, so we can directly reuse the immediate result in \eqref{eq:one_round_theorem_1} with minor adjustment. That gives
\begin{align}
    \mathbb{E}\left[ \left\lVert  \nabla F(\boldsymbol{X}_{r-1}) \right\rVert^2 \right]
    &\leq \frac{4}{\eta_{\mathrm{eff}}^{(r)}T_{\mathrm{eff}}^{(r)}}\left( \mathbb{E}\left[ F(\boldsymbol{X}_{r-1}) \right]-\mathbb{E}\left[ F(\boldsymbol{X}_{r}) \right]
    \right) \notag\\
    &+\eta_{\mathrm{eff}}^{(r)}T_{\mathrm{eff}}^{(r)} \frac{\kappa^2}{\beta^2}
    +\frac{2\eta LA_r^2}{T_{\mathrm{eff}}^{(r)}} \sigma^2.
\end{align}
Average both sides over $R$ rounds, we get 
\begin{align}
    &\frac{1}{R}\sum_{r=1}^R \mathbb{E}\left[ \left\lVert  \nabla F(\boldsymbol{X}_{r-1}) \right\rVert^2 \right]
    \leq   \notag\\ 
    &\frac{4}{R\eta}\sum_{r=1}^R \frac{\mathbb{E}\left[ F(\boldsymbol{X}_{r-1}) \right]-\mathbb{E}\left[ F(\boldsymbol{X}_{r}) \right]}{\sum_{m=1}^M p_m^{(r)}(1-q_m^{(r)})T_{\mathrm{eff}}^{(r)}}  +\frac{2\eta L}{R}\sum_{r=1}^R \frac{A_r^2}{T_{\mathrm{eff}}^{(r)}} \sigma^2 \notag\\ 
    &+\eta\frac{\kappa^2}{\beta^2} \frac{1}{R}\sum_{r=1}^R \sum_{m=1}^M p_m^{(r)}(1-q_m^{(r)})T_{\mathrm{eff}}^{(r)}.
    \label{eq:one_round_fedacs}
\end{align}
Define 
\begin{subequations}
\begin{align}
    B&=\max_{r}\left\{\frac{1}{\sum_{m=1}^M p_m^{(r)}(1-q_m^{(r)})T_{\mathrm{eff}}^{(r)}}\right\},\\
    C_r&=\sum_{m=1}^M p_m^{(r)}(1-q_m^{(r)}),\\
    D_r&=\frac{A_r^2}{T_{\mathrm{eff}}^{(r)}},
\end{align}\label{eq: defineABC}
\end{subequations}
and $\Bar{C}=\frac{1}{R}\sum_{r=1}^R C_r $, $\Bar{D}=\frac{1}{R}\sum_{r=1}^R D_r $, Theorem \ref{theo: FedACS} is derived.

\section{Bounded Dissimilarity Assumptions} \label{appx: Equivalent Assumptions} 
Other commonly adopted bounded dissimilarity assumptions are listed in Table \ref{table: assumps}. These assumptions are compared along with Assump. \ref{assump: Dissimilarity}. If an assumption is said to be stricter than the other, that means the assumption implies the other.
The complete proof can be found at \cite[Appendix O]{weng2025heterogeneityawareclientsamplingunified}
\begin{table}[!h]
    \centering
    \vspace*{-5pt}
    \caption{\small Common assumptions on bounded gradient dissimilarity.}\label{table: assumps}
    \centering
    \resizebox{\linewidth}{!}{
    \begin{tabular}{c  c  c }
    \toprule
    {\bf \centering No.} & {\bf \centering Bounded Dissimilarity Assumptions} &  
    {\bf \centering Ref.}  \\
    \midrule     
    \bf \centering 4 & $\forall m\in [M]: \lVert \nabla F_m(x)-\nabla F(x)\rVert^2\leq D_m^2$
     & \cite{wang2021quantized}
    \\ 
    \midrule 
    \bf \centering 5 & $\sum_{m=1}^M \omega_m \lVert \nabla F_m(x)-\nabla F(x) \rVert^2\leq D^2$
     & \cite{chen2020optimal}
     \\ 
    \midrule
    \bf \centering 6 & $\exists\beta^2\geq 0, \kappa^2\geq 0: \sum_{m=1}^M \omega_m\lVert \nabla F_m(x)-  \nabla F(x) \rVert^2 \leq \beta^2\lVert  \nabla F(x) \rVert^2+\kappa^2$
     & \cite{xiang2024efficient} \\
    \midrule
    \bf \centering 7 & $\forall m\in [M]: \lVert \nabla F_m(x) \rVert^2\leq G^2$
     & \cite{11159129} \\
     \midrule
    \bf \centering 8 & $\forall m\in [M]: \lVert \nabla F_m(x\vert \xi) \rVert^2\leq G^2$
     & \cite{Li2020} 
     \\
    \bottomrule
    \end{tabular}
    }
    \vspace*{-.5\baselineskip}
\end{table}

\vspace{-0.5em}
\section{Step Length Calibration}\label{appx: step_carlibration}
The step size influences the convergence speed, either accelerating or decelerating the optimization process, but does not determine the final convergence point. Although a larger step size may lead to faster initial convergence, the algorithm can still settle at an incorrect stationary point. To ensure a fair comparison of each algorithm’s ability to reach the correct optimum, all benchmark methods are operated with the same effective step size. This calibration is achieved by adjusting the learning rate $\eta$. The index $r$ is dropped below.

In \texttt{FedAvg}, by plugging $p_m=\omega_m$ into \eqref{eq: E_agg} and \eqref{eq: obj_inconsist_comm_comp}, we get 
\begin{subequations}
\begin{align}
    &\eta_{\mathrm{eff},1}=\eta_1\sum_{m=1}^M\omega_m(1-q_m),\\
    &T_{\mathrm{eff},1}=\sum_{m=1}^M \frac{\omega_m(1-q_m)}{\sum_{m=1}^M\omega_m(1-q_m)}\lVert \boldsymbol{a}_m \rVert_1.
\end{align}
\end{subequations}

In \texttt{FedACS}, in \eqref{eq: fedACS_update}, we have
\begin{subequations}
\begin{align}
    &\eta_{\mathrm{eff},2}=\eta_2\frac{\sum_{m=1}^M\frac{\omega_m}{\lVert\boldsymbol{a}_m
    \rVert_1}}{\sum_{m=1}^M\frac{\omega_m}{(1-q_m)\lVert\boldsymbol{a}_m
    \rVert_1}},\\
    &T_{\mathrm{eff},2}=\frac{1}{\sum_{m=1}^M\frac{\omega_m}{\lVert\boldsymbol{a}_m
    \rVert_1}}.
\end{align}
\end{subequations}

In \texttt{c-a-FedAvg}, the expectation of the server aggregation is given as follows, 
\begin{subequations}
\begin{align}
    &\mathbb{E}\left[ \frac{1}{K} \sum_{m\in \Tilde{\mathcal{S}}_r(\boldsymbol{p},\boldsymbol{q})} \frac{1}{1-q_m}  (-\eta_3\boldsymbol{\Delta}_{m}^{r,\boldsymbol{\xi}}) \right]\\
    &=-\eta_3\mathbb{E}\left[ \frac{1}{K} \sum_{m\in \mathcal{S}_r(\boldsymbol{p})} \frac{1}{1-q_m}  \boldsymbol{\Delta}_{m}^{r,\boldsymbol{\xi}} \cdot Z_m^r \right]\\
    &=-\eta_3\mathbb{E}\left[ \frac{1}{K} \sum_{m\in \mathcal{S}_r(\boldsymbol{p})} \boldsymbol{\Delta}_{m}^{r,\boldsymbol{\xi}} \right]
    =-\eta_3 K\cdot \frac{1}{K} \sum_{m=1}^M \omega_m \boldsymbol{\Delta}_{m}^{r,\boldsymbol{\xi}} \\
    &= -\underbrace{\eta_3 \sum_{m=1}^M \omega_m \lVert\boldsymbol{a}_m
    \rVert_1}_{\eta_{\mathrm{eff},3}}  \cdot \underbrace{\sum_{m=1}^M\frac{\omega_m \lVert\boldsymbol{a}_m
    \rVert_1}{\sum_{m=1}^M \omega_m \lVert\boldsymbol{a}_m
    \rVert_1}}_{T_{\mathrm{eff},3}} \cdot  \nabla\Bar{\boldsymbol{F}}_m^{r,\boldsymbol{\xi}}.
    \label{eq: c-a-fedavg_agg}
\end{align}
\end{subequations}

In \texttt{FedVarp}, let $r_m'$ denote the latest previous round in which PS receives the local model update from client $m$, we have the aggregation at the server as follows,
\begin{align}
    \mathbb{E}\left[ \frac{1}{K} \left(\sum_{m\in \Tilde{\mathcal{S}}_r(\boldsymbol{p},\boldsymbol{q})} (-\eta_4\boldsymbol{\Delta}_{m}^{r,\boldsymbol{\xi}}) + \sum_{m\in \mathcal{S}_r(\boldsymbol{p})\backslash\Tilde{\mathcal{S}}_r(\boldsymbol{p},\boldsymbol{q})} (-\eta_4\boldsymbol{\Delta}_{m}^{r_m',\boldsymbol{\xi}}) \right) \right],
    \label{eq: fedvarp_agg}
\end{align}
if $\mathbb{E}\left[ \boldsymbol{\Delta}_{m}^{r,\boldsymbol{\xi}} \right]=\mathbb{E}\left[ \boldsymbol{\Delta}_{m}^{r_m',\boldsymbol{\xi}}\right]$, then we have \eqref{eq: fedvarp_agg} further equal to
\begin{align}
    \eqref{eq: fedvarp_agg}
    =-\eta_4\mathbb{E}\left[ \frac{1}{K} \sum_{m\in \mathcal{S}_r(\boldsymbol{p})} \boldsymbol{\Delta}_{m}^{r,\boldsymbol{\xi}} \right].
\end{align}
This is identical to \eqref{eq: c-a-fedavg_agg}, so $\eta_4$ can be set the same to $\eta_3$. 

In \texttt{FedNova}, only the aggregation weights are reset to $\omega_m$, i.e., \texttt{FedNova} only changes the direction of the aggregated vectors, the step length is unaffected. So we have
\begin{subequations}
\begin{align}
    &\eta_{\mathrm{eff},5}=\eta_5\sum_{m=1}^M\omega_m(1-q_m),\\
    &T_{\mathrm{eff},5}=\sum_{m=1}^M \frac{\omega_m(1-q_m)}{\sum_{m=1}^M\omega_m(1-q_m)}\lVert \boldsymbol{a}_m \rVert_1.
\end{align}
\end{subequations}

In OS, the effective learning rate $\eta_{\mathrm{eff},6}$ and effective length of accumulation $T_{\mathrm{eff},6}$ are the same in \texttt{FedAvg}. In \texttt{FedAU}, it's hard to compute the effective step size, so the hyper-parameter is set the same in \texttt{FedAvg}.  

If $\eta_1$ is preset, then $\eta_2$, $\eta_3$, $\eta_4$, $\eta_5$ can be chosen accordingly such that $\eta_{\mathrm{eff},1}T_{\mathrm{eff},1}=\eta_{\mathrm{eff},2}T_{\mathrm{eff},2}=\eta_{\mathrm{eff},3}T_{\mathrm{eff},3}=\eta_{\mathrm{eff},4}T_{\mathrm{eff},4}=\eta_{\mathrm{eff},5}T_{\mathrm{eff},5}$. In \texttt{FedAU} and OS, the learning rate is the same as in \texttt{FedAvg}. 

\onecolumn
\subsection{Pseudo-Code of \texttt{FedACS}}\label{appx: FedACS}
The complete Pseudo-Code of \texttt{FedACS} is shown in Algorithm \ref{alg: FedACS} below. 
\begin{algorithm}
\caption{\texttt{FedACS}: Federated Heterogeneity-Aware Client Sampling} \label{alg: FedACS}
\begin{algorithmic}[1]
\Require $\boldsymbol{X}_0$, $\eta$, $R$, $\{\omega_m\}_{m=1}^M$, $\boldsymbol{a}_m^{(r)}$, $q_m^{(r)}$
\Ensure $\boldsymbol{X}_R$
\For{round $r = 1, \dots, R$}
    \State {Compute the sampling probability $p_m^{(r)}=\frac{\frac{ \omega_m}{(1-q_m^{(r)})\lVert\boldsymbol{a}_m^{(r)}
    \rVert_1}}{\sum_{m=1}^M \frac{\omega_m}{(1-q_m^{(r)})\lVert\boldsymbol{a}_m^{(r)}
    \rVert_1}}$}
     \State Let $\boldsymbol{p}^{(r)}=[p_1^{(r)},\cdots,p_M^{(r)}]$
     \State Sample $K$ clients independently with replacement according to $\boldsymbol{p}^{(r)}$ and form $\mathcal{S}_r(\boldsymbol{p})$
    \State // \textbf{\textit{Client Side}}:
    \For{each worker $m \in \mathcal{S}_r(\boldsymbol{p})$, in parallel}
        \State $\boldsymbol{X}_{m,r}^{0} = \boldsymbol{X}_r$
        \For{$t = 1, \dots, T_m^{(r)}$}
            \State Compute $\nabla\boldsymbol{F}_m(\boldsymbol{X}_{m,r}^{t-1}\vert \xi_{m,r}^{t})$
            \State Local update: $\boldsymbol{X}_{m,r}^{t} = \boldsymbol{X}_{m,r}^{t-1} - \eta \cdot a_{m,t}^{(r)} \cdot\nabla\boldsymbol{F}_m(\boldsymbol{X}_{m,r}^{t-1}\vert \xi_{m,r}^{t})$
        \EndFor
        \State Let $\boldsymbol{a}_m^{(r)}=[a_{m,1}^{(r)}, \cdots,  a_{m,T_m}^{(r)}] $
        \State Let $\nabla\boldsymbol{F}_m^{r,\boldsymbol{\xi}}=[\nabla\boldsymbol{F}_m(\boldsymbol{X}_{m,r}^{0}\vert \xi_{m,r}^{1}), \cdots, \nabla\boldsymbol{F}_m(\boldsymbol{X}_{m,r}^{ T_m-1}\vert \xi_{m,r}^{ T_m})]$
        \State Send the local model update $-\eta \boldsymbol{\Delta}_{m}^{r,\boldsymbol{\xi}}=-\eta \nabla\boldsymbol{F}_m^{r,\boldsymbol{\xi}} \cdot\boldsymbol{a}_m^{(r)}$ to the server
    \EndFor
    \State // \textbf{\textit{Transmission Between Clients and the Server}}:
    \For{link $m \in \mathcal{S}_r(\boldsymbol{p})$, in parallel}
     \State $Z_m^r\sim\mathrm{Bernoulli}(1-q_m^{(r)})$
    \EndFor
    \State Let $\boldsymbol{q}^{(r)}=[q_1^{(r)}, \cdots, q_M^{(r)}]$
    \State Let $\Tilde{\mathcal{S}}_r(\boldsymbol{p},\boldsymbol{q})=\{m: m\in \mathcal{S}_r(\boldsymbol{p}), Z_m^r=1 \}$
    \State // \textbf{\textit{Server Side}}:
    \State Receive $-\eta \boldsymbol{\Delta}_{m}^{r,\boldsymbol{\xi}}, m \in \Tilde{\mathcal{S}}_r(\boldsymbol{p},\boldsymbol{q})$
    \State Server update: $\boldsymbol{X}_{r+1}=\boldsymbol{X}_{r} +\frac{1}{K} \sum_{m\in \Tilde{\mathcal{S}}_r(\boldsymbol{p},\boldsymbol{q})} (-\eta\boldsymbol{\Delta}_{m}^{r,\boldsymbol{\xi}})$
    \State Broadcast $\boldsymbol{X}_{r+1}$ to clients
\EndFor
\end{algorithmic}
\end{algorithm}
\section{Quadratic Model in General Heterogeneous Setting with IS}\label{Appx: Objective optimizers}
\textbf{Quadratic Model.} Consider a simple quadratic model where each local objective function is strongly convex and defined as follows:
\begin{align}
    F_m(\boldsymbol{X})=\frac{1}{2}\boldsymbol{X}^\top \boldsymbol{H}_m \boldsymbol{X}-\boldsymbol{E}_m^\top \boldsymbol{X}+\frac{1}{2}\boldsymbol{E}_m^\top \boldsymbol{H}_m^{-1} \boldsymbol{E}_m \geq 0,
\end{align}
where $\boldsymbol{H}_m\in \mathbb{R}^{d\times d}$ is an invertible matrix and $\boldsymbol{E}_m\in \mathbb{R}^d$ is an arbitrary vector. In this setting, the local optimum is $\boldsymbol{X}^\star_m=\boldsymbol{H}_m^{-1} \boldsymbol{E}_m$. W.L.O.G., we define the global objective function as a weighted average of all local objective functions, i.e.,
\begin{align}
    F(\boldsymbol{X})=\sum_{m=1}^M \omega_m F_m(\boldsymbol{X}) 
    =\frac{1}{2}\boldsymbol{X}^\top \Bar{\boldsymbol{H}} \boldsymbol{X}-\Bar{\boldsymbol{E}}^\top \boldsymbol{X}+\frac{1}{2}\sum_{m=1}^M \omega_m\boldsymbol{E}_m^\top \boldsymbol{H}_m^{-1} \boldsymbol{E}_m, 
\end{align}
where $\Bar{\boldsymbol{H}}=\sum_{m=1}^M \omega_m \boldsymbol{H}_m$, and $\Bar{\boldsymbol{E}}=\sum_{m=1}^M \omega_m \boldsymbol{E}_m$. As a result, the global minimum is $\boldsymbol{X}^\star=\Bar{\boldsymbol{H}}^{-1}\Bar{\boldsymbol{E}}$.Now we would like to study if the federated optimization algorithms can converge to this global optimum.  

\textbf{Local Update Rule.} The local update rule of \texttt{FedProx} on the $m$-th client can be written as
\begin{subequations}
    \begin{align}
    \boldsymbol{X}_{m,r}^{t+1}
    &=\boldsymbol{X}_{m,r}^{t}-\eta\left( \boldsymbol{H}_m\boldsymbol{X}_{m,r}^{t}-\boldsymbol{E}_m+\mu\left( \boldsymbol{X}_{m,r}^{t}-\boldsymbol{X}_{r} \right) \right)\\
    &=(\boldsymbol{I}-\eta\mu\boldsymbol{I}-\eta\boldsymbol{H}_m)\boldsymbol{X}_{m,r}^{t}
    +\eta \boldsymbol{E}_m+\eta\mu \boldsymbol{X}_{r}, 
    \label{eq: local_update_example}
\end{align}
\end{subequations}
where $\boldsymbol{X}_{m,r}^{t}$ is the local model acquired at the $t$-th iteration on $m$-th client in the $r$-th round, $\boldsymbol{X}_{r}$ is the global model at the $r$-th round, $\eta$ is the learning rate, $\mu$ is the tunable hyper-parameter in \texttt{FedProx}. When $\mu=0$, the algorithm becomes \texttt{FedAvg}. 

By minor re-arrangement, \eqref{eq: local_update_example} can be written into
\begin{align*}
\boldsymbol{X}_{m,r}^{t+1}-\boldsymbol{C}_{m,r}
=(\boldsymbol{I}-\eta\mu\boldsymbol{I}-\eta\boldsymbol{H}_m)\left(\boldsymbol{X}_{m,r}^{t}-\boldsymbol{C}_{m,r}\right),
\end{align*}
where $\boldsymbol{C}_{m,r}=(\boldsymbol{H}_m+\mu\boldsymbol{I})^{-1}(\boldsymbol{E}_m+\mu \boldsymbol{X}_{r})$.

After $T_m$ local iterations, we get
\begin{align*}
\boldsymbol{X}_{m,r}^{T_m}-\boldsymbol{C}_{m,r}
=(\boldsymbol{I}-\eta\mu\boldsymbol{I}-\eta\boldsymbol{H}_m)^{T_m}\left(\boldsymbol{X}_{r}-\boldsymbol{C}_{m,r}\right).
\end{align*}
After minor adjustments, 
\begin{align}
    \boldsymbol{X}_{m,r}^{T_m}-\boldsymbol{X}_{r}
=(\boldsymbol{I}-\eta\mu\boldsymbol{I}-\eta\boldsymbol{H}_m)^{T_m}\left(\boldsymbol{X}_{r}-\boldsymbol{C}_{m,r}\right)+\boldsymbol{C}_{m,r}-\boldsymbol{X}_{r},
\end{align}
i.e.,
\begin{subequations}
\begin{align}
    \boldsymbol{X}_{m,r}^{T_m}-\boldsymbol{X}_{r}
&=\left((\boldsymbol{I}-\eta\mu\boldsymbol{I}-\eta\boldsymbol{H}_m)^{T_m}-\boldsymbol{I}\right)\left(\boldsymbol{X}_{r}-\boldsymbol{C}_{m,r}\right)\\
&=\boldsymbol{K}_m(\eta, \mu)\left( \boldsymbol{E}_m-\boldsymbol{H}_m\boldsymbol{X}_{r} \right),
\end{align}
\end{subequations}
where $\boldsymbol{K}_m(\eta, \mu)= \left(\boldsymbol{I}-(\boldsymbol{I}-\eta\mu\boldsymbol{I}-\eta\boldsymbol{H}_m)^{T_m}\right)\left(\boldsymbol{H}_m+\mu \boldsymbol{I} \right)^{-1}$. More details can be found in \cite{wangTackling}. 

\textbf{Server Aggregation with Heterogeneous Communication and Computation.} When $K$ clients are sampled according to their importance $\omega_m$ to transmit their local models to the server through unreliable communication with broken probability $q_m$ each round. The server averages all the received local models, it holds that
\begin{subequations}
\begin{align}
    \boldsymbol{X}_{r+1}-\boldsymbol{X}_{r}
    &= \frac{1}{K} \sum_{m\in \Tilde{\mathcal{S}}_r(\boldsymbol{p},\boldsymbol{q})} \left( \boldsymbol{X}_{m,r}^{T_m} - \boldsymbol{X}_{r} \right)\\
    &= \frac{1}{K} \sum_{m\in \Tilde{\mathcal{S}}_r(\boldsymbol{p},\boldsymbol{q})} \boldsymbol{K}_m(\eta, \mu)\left( \boldsymbol{E}_m-\boldsymbol{H}_m\boldsymbol{X}_{r} \right).
\end{align}
\label{eq: server_agg_example}
\end{subequations}
For clarity of derivation, an indicator is defined as follows:
\begin{align}
    Z_{m,r}=\begin{cases}
        1, \;\;\; \text{if\;the\;communication\;succeeds}\\
        0, \;\;\; \text{if\;the\;communication\;fails}
    \end{cases},
\end{align}
then $Z_{m,r}\sim \mathrm{Ber}(1-q_m)$. Consequently, we have $\mathbb{E}[Z_{m,r}]=1-q_m$. 
Accordingly, the expectation of \eqref{eq: server_agg_example} is 
\begin{subequations}
\begin{align}
    \mathbb{E}\left[\boldsymbol{X}_{r+1}-\boldsymbol{X}_{r}\right]
    &= \mathbb{E}\left[ \frac{1}{K} \sum_{m\in \mathcal{S}_r(\boldsymbol{p})} \left( \boldsymbol{X}_{m,r}^{T_m} - \boldsymbol{X}_{r} \right)\cdot Z_{m,r}\right]\\
    &= \mathbb{E}\left[ \frac{1}{K} \sum_{m\in \mathcal{S}_r(\boldsymbol{p})} \left( \boldsymbol{X}_{m,r}^{T_m} - \boldsymbol{X}_{r} \right)\cdot Z_{m,r}\right]\\
    &=\sum_{m=1}^M \omega_m(1-q_m) \mathbb{E}\left[ \boldsymbol{X}_{m,r}^{T_m} - \boldsymbol{X}_{r} \right]\\
    &=\sum_{m=1}^M \omega_m(1-q_m) \boldsymbol{K}_m(\eta, \mu)\left( \boldsymbol{E}_m-\boldsymbol{H}_m\mathbb{E}\left[\boldsymbol{X}_{r}\right] \right).
\end{align}
\end{subequations}
This is equivalent to
\begin{subequations}
\begin{align}
    \mathbb{E}\left[\boldsymbol{X}_{r+1}\right]
    &=\left(\boldsymbol{I}-\sum_{m=1}^M \omega_m(1-q_m) \boldsymbol{K}_m(\eta, \mu)\boldsymbol{H}_m\right)\mathbb{E}\left[\boldsymbol{X}_{r}\right]
    + \sum_{m=1}^M \omega_m(1-q_m) \boldsymbol{K}_m(\eta, \mu) \boldsymbol{E}_m
\end{align}\label{eq: expectation_global_example}
\end{subequations}
By minor manipulation, \eqref{eq: expectation_global_example} is equivalent to 
\begin{align}
    \mathbb{E}\left[\boldsymbol{X}_{r+1}\right]-\Tilde{\boldsymbol{X}}^\star
    &=\left(\boldsymbol{I}-\sum_{m=1}^M \omega_m(1-q_m) \boldsymbol{K}_m(\eta, \mu)\boldsymbol{H}_m\right)\left(\mathbb{E}\left[\boldsymbol{X}_{r}\right]-\Tilde{\boldsymbol{X}}^\star
    \right),
\end{align}
where $\Tilde{\boldsymbol{X}}^\star=\left(\boldsymbol{I}-\sum_{m=1}^M \omega_m(1-q_m) \boldsymbol{K}_m(\eta, \mu)\boldsymbol{H}_m\right)^{-1}\left(\sum_{m=1}^M \omega_m(1-q_m) \boldsymbol{K}_m(\eta, \mu) \boldsymbol{E}_m\right)$. 

After $R$ rounds, this gives
\begin{align}
    \mathbb{E}\left[\boldsymbol{X}_{R}\right]-\Tilde{\boldsymbol{X}}^\star
    &=\left(\boldsymbol{I}-\sum_{m=1}^M \omega_m(1-q_m) \boldsymbol{K}_m(\eta, \mu)\boldsymbol{H}_m\right)^R\left(\mathbb{E}\left[\boldsymbol{X}_{0}\right]-\Tilde{\boldsymbol{X}}^\star
    \right).
\end{align}

Accordingly, when $\left\lVert \boldsymbol{I}-\sum_{m=1}^M \omega_m(1-q_m) \boldsymbol{K}_m(\eta, \mu)\boldsymbol{H}_m\right\rVert<1$, the iterative training process converges to
\begin{align}
    \lim_{R\rightarrow \infty}  \mathbb{E}\left[\boldsymbol{X}_{R}\right]=\Tilde{\boldsymbol{X}}^\star=\left(\boldsymbol{I}-\sum_{m=1}^M \omega_m(1-q_m) \boldsymbol{K}_m(\eta, \mu)\boldsymbol{H}_m\right)^{-1}\left(\sum_{m=1}^M \omega_m(1-q_m) \boldsymbol{K}_m(\eta, \mu) \boldsymbol{E}_m\right).
\end{align}

\textbf{Concrete Proof for Example \ref{example: 2}.} In our example, $p_1=p_2=\cdots=p_M=\omega_1=\omega_2=\cdots=\omega_M=\frac{1}{M}$, $\boldsymbol{H}_1=\boldsymbol{H}_2=\cdots=\boldsymbol{H}_M=\boldsymbol{I}$ and $\mu=0$. Then $\boldsymbol{K}_m(\eta, \mu)=1-(1-\eta)^{T_m}$. So we have
\begin{align}
\lim_{R\rightarrow \infty}  \mathbb{E}\left[\boldsymbol{X}_{R}\right]=
    \frac{\sum_{m=1}^M(1-q_m)\left(1-(1-\eta)^{T_m} \right)\boldsymbol{E}_m}{\sum_{m=1}^M(1-q_m)\left(1-(1-\eta)^{T_m} \right)}.
    \label{eq: converge_point_R}
\end{align}

{For each client, as $\eta \to 0$, we have
$1-(1-\eta)^{T_m} \to 0$. Therefore, both the numerator and denominator 
in \eqref{eq: converge_point_R} approach zero. Differentiating 
the numerator and denominator with respect to $\eta$, we have}
\begin{equation}
{\frac{\mathrm{d}}{\mathrm{d}\eta}
\left[1-(1-\eta)^{T_m}\right]
=
T_m(1-\eta)^{T_m-1}.}
\end{equation}
{According to L'H\^opital's rule, we have}
\begin{subequations}
\begin{align}
&{\lim_{\eta\rightarrow 0} \lim_{R\rightarrow \infty}  \mathbb{E}\left[\boldsymbol{X}_{R}\right]}\\
=&{\lim_{\eta\to0}
\frac{
\sum_{m=1}^{M}(1-q_m)\left[1-(1-\eta)^{T_m}\right]\boldsymbol{E}_m
}{
\sum_{m=1}^{M}(1-q_m)\left[1-(1-\eta)^{T_m}\right]
}}\\
=&
{\lim_{\eta\to0}
\frac{
\sum_{m=1}^{M}(1-q_m)T_m(1-\eta)^{T_m-1}\boldsymbol{E}_m
}{
\sum_{m=1}^{M}(1-q_m)T_m(1-\eta)^{T_m-1}
}}\\
=&
{\frac{
\sum_{m=1}^{M}(1-q_m)T_m\boldsymbol{E}_m
}{
\sum_{m=1}^{M}(1-q_m)T_m
}\triangleq\Tilde{\boldsymbol{X}}^\star.}
\label{eq: converge_point_fedavg_sgd}
\end{align}
\end{subequations}
\section{Examination of the Bounded Dissimilarity Assumptions} \label{appx: Equivalent Assumptions} 
In this section, we examine the bounded dissimilarity assumptions listed in Table \ref{table: assumps} along with Assumption \ref{assump: Dissimilarity}. 
\subsection{Equivalence Between Assumption 3 and Assumption 6}
By proving that Assumption 3 is both a sufficient condition for Assumption 6 (via direct proof) and a necessary condition (via converse proof), we establish the equivalence between Assumption 3 and Assumption 6. 
\begin{proof}[Direct Proof (3$\Rightarrow$6)]
Let Assumption 3 hold for $\beta_1$ and $\kappa_1$. This implies for $\beta_1>0$, it holds that 
\begin{align}
    \sum_{m=1}^M \omega_m\left\lVert \nabla F_m(x) \right\rVert^2\leq (1+\beta_1^2) \left\lVert\nabla F(x) \right\rVert^2+\kappa_1^2.  
\end{align}
By simple manipulation, we get
\begin{align}
    2\sum_{m=1}^M \omega_m\left\lVert \nabla F_m(x) \right\rVert^2+2\left\lVert\nabla F(x) \right\rVert^2\leq (4+2\beta_1^2) \left\lVert\nabla F(x) \right\rVert^2+2\kappa_1^2.  
\end{align}
By Young’s Inequality for vectors, we get
\begin{align}
    \sum_{m=1}^M \omega_m\left\lVert \nabla F_m(x)-\nabla F(x) \right\rVert^2\leq (4+2\beta_1^2) \left\lVert\nabla F(x) \right\rVert^2+2\kappa_1^2.  
\end{align}
Clearly, since $4+2\beta_1^2>0$ and $2\kappa_1^2>0$, Assumption 6 is satisfied.  
\end{proof}

\begin{proof}[Converse Proof (3$\Leftarrow$6)] 
Let Assumption 6 hold for $\beta_2$ and $\kappa_2$. By simple variation, we get   
\begin{align}
    \sum_{m=1}^M \omega_m\lVert \nabla F_m(x)-  \nabla F(x) \rVert^2 + \lVert  \nabla F(x) \rVert^2 
    \leq 
    (1+\beta_2^2)\lVert  \nabla F(x) \rVert^2+\kappa_2^2.
\end{align}
Due to Young’s Inequality for vectors, we further get
\begin{align}
    \sum_{m=1}^M \omega_m\lVert \nabla F_m(x) \rVert^2 
    \leq 
    2(1+\beta_2^2)\lVert  \nabla F(x) \rVert^2+2\kappa_2^2.
\end{align}
Clearly, $ 2(1+\beta_2^2)>1$ and $2\kappa_2^2>0$,Assumption 3 is satisfied.    
\end{proof}

\subsection{Equivalence Between Assumption 4 and Assumption 5}
Let Assumption 4 hold. Obviously, Assumption 5 is satisfied by setting $D^2=\sum_{m=1}^M D_m^2$. 

Let Assumption 5 hold. Due to the non-negativity of $\ell_2$-norm, the bounded sum of gradient norm (Assumption 5) inherently implies that each component in the sum is bounded, i.e., $\exists D_1^2\sim D_M^2$ satisfying $\sum_{m=1}^MD_m^2=D^2$ such that $\lVert \nabla F_m(x)-\nabla F(x) \rVert^2\leq D_m^2$. 

In summary, Assumption 4 is both a necessary and sufficient condition for Assumption 5, establishing their equivalence.
\subsection{Relation Between Assumption 4 and Assumption 7}
\begin{proof}[Direct Proof (4$\not\Rightarrow$7)]
By Young’s Inequality for vectors and Jensen’s Inequality, the norm in Assumption 4 can be further lower bounded by
\begin{subequations}
  \begin{align}
D_m\geq\left\lVert\nabla F_m(x)-\nabla F(x)\right\rVert
&\geq \left\lVert\nabla F_m(x)\right\rVert-\left\lVert \nabla F(x)\right\rVert\\
&= \left\lVert\nabla F_m(x)\right\rVert-\left\lVert \sum_{m=1}^M \omega_m \nabla F_m(x)\right\rVert\\
&\geq \left\lVert\nabla F_m(x)\right\rVert-\sum_{m=1}^M \omega_m\left\lVert \nabla F_m(x)\right\rVert\\
&=\left(1-\omega_m\right)\left\lVert\nabla F_m(x)\right\rVert-\sum_{k\neq m} \omega_k\left\lVert \nabla F_k(x)\right\rVert. \label{eq: A3_LB}
\end{align}  
\label{eq: A3_LBall}
\end{subequations}

If there are at least $1$ unbounded gradient norms, W.L.O.G., assume $\left\lVert \nabla F_1(x)\right\rVert\geq \cdots \geq \left\lVert \nabla F_M(x)\right\rVert$. For \eqref{eq: A3_LB} to be bounded for all $m\in [M]$, they must scales linearly, i.e., $\left\lVert \nabla F_k(x)\right\rVert=\zeta_{k}\left\lVert \nabla F_1(x)\right\rVert+z_{k}$, $0\leq\zeta_{k}\leq 1$, $0\leq z_{k}<+\infty$. For the gradient norm of finite values, $\zeta_{k}=0$. For the large gradient norm of the same amplitude, $\zeta_{k}=1$, and $z_{k}=0$. 
Additionally, it must hold that
\begin{subequations}
\begin{align}
&(1-\omega_1)\left\lVert \nabla F_1(x)\right\rVert
-\sum_{k=2}^M\omega_k\left(\zeta_k\left\lVert \nabla F_k(x)\right\rVert+z_k\right)<+\infty,\\
&(1-\omega_2)(\zeta_2\left\lVert\nabla F_1(x)\right\rVert+z_2)
-\omega_1 \left\lVert \nabla F_1(x)\right\rVert-\sum_{k=3}^M\omega_k\left(\zeta_k\left\lVert \nabla F_k(x)\right\rVert+z_k\right)<+\infty,\\
&\hspace{6cm}\vdots\\
&(1-\omega_M)(\zeta_M\left\lVert\nabla F_1(x)\right\rVert+z_M)
-\omega_1 \left\lVert \nabla F_1(x)\right\rVert-\sum_{k=2}^{M-1}\omega_k\left(\zeta_k\left\lVert \nabla F_k(x)\right\rVert+z_k\right)<+\infty.
\end{align}   
\end{subequations}
Consequently, 
\begin{subequations}
\begin{align}
&1-\omega_1
-\sum_{k=2}^M\omega_k\zeta_k=0,\\
&\zeta_2-\omega_1
-\sum_{k=2}^M\omega_k\zeta_k=0,\\
&\hspace{1.7cm}\vdots\\
&\zeta_M-\omega_1
-\sum_{k=2}^M\omega_k\zeta_k=0.   
\end{align}\label{eq: 2_sol}    
\end{subequations}

The only feasible solution to \eqref{eq: 2_sol} is $\zeta_2=\cdots=\zeta_M=1$. This implies that Assumption 4 holds in two cases: either the gradient norms are bounded, or they are unbounded but differ only by constant factors. In the IID setting, all gradient norms are identical, so Assumption 4 is satisfied even if the norms are unbounded. In the non-IID setting, if the gradients are unbounded, their norms must differ only by a constant to satisfy Assumption 4.
\end{proof}
\begin{proof}[Converse Proof (4$\Leftarrow$7)] By repeatedly using Young’s Inequality for vectors for $M-1$ times, we get
\begin{subequations}
\begin{align}
    \left\lVert \nabla F_m(x)-\nabla F(x)\right\rVert^2
    &= \left\lVert (1-\omega_m) \nabla F_m(x)-\sum_{k\neq m} \omega_k\nabla F_k(x)\right\rVert^2\\
    &\leq \sum_{m=1}^M c_k  \left\lVert \nabla F_m(x) \right\rVert^2,  \label{eq: upp_a3}
\end{align}
\end{subequations}
with $c_k<2^{M-1}$. Assumption 7 implies that \eqref{eq: upp_a3} is bounded. Consequently, $\left\lVert \nabla F_m(x)-\nabla F(x)\right\rVert^2$ is also bounded, i.e., Assumption 7 implies Assumption 4.   

In summary, Assumption 7 is stricter than Assumption 4, i.e., Assumption 7 $\succeq$Assumption 4.  
\end{proof}
\subsection{Relation Between Assumption 3, 6 and Assumption 4, 5}
\begin{proof}[Direct Proof (4$\Rightarrow$6)]
Let Assumption 4 hold. We have
\begin{align}
\sum_{m=1}^M \omega_m \left\lVert \nabla F_m(x)-\nabla F(x) \right\rVert^2\leq\sum_{m=1}^M \omega_m D_m^2.
\end{align}
This is equivalent to $\beta_2^2=0$ and $\kappa_2^2=\sum_{m=1}^M \omega_m D_m^2$ in Assumption 6. That is, Assumption 4 implies Assumption 6.
\end{proof}
\begin{proof}[Converse Proof (4$\not\Leftarrow$6)]
Assumption 6 does not necessarily imply Assumption 4. Here, we give a counterexample to illustrate this fact. Let $\nabla F_m(x)=m\cdot\nabla F_1(x)$ and $\omega_m=\frac{1}{M}$, then $\nabla F(x)=\frac{M(M+1)}{2}\nabla F_1(x)$. The left side of Assumption 6 can be written as 
\begin{subequations}
    \begin{align}
    \sum_{m=1}^M \frac{1}{M}\left\lVert \nabla F_m(x)- \nabla F(x)\right\rVert^2
    &=\left(\frac{2m}{M+1}-1\right)^2 \left\lVert \nabla F(x)\right\rVert^2\\
    &\leq \left(\frac{M-1}{M+1}\right)^2 \left\lVert \nabla F(x)\right\rVert^2.
    \label{eq: bound Assump6}
\end{align}
\end{subequations}
From \eqref{eq: bound Assump6}, we can easily see that such choice of $\nabla F_m(x)$ and $\omega_m$ satisfying Assumption 6 by setting $\beta_2^2=\left(\frac{M-1}{M+1}\right)^2$ and $\kappa_2^2=0$. But $\left\lVert \nabla F_m(x)-\nabla F(x) \right\rVert^2$ can be arbitrarily large, violating Assumption 4. 

In summary, Assumption 4 is stricter than Assumption 6, i.e., Assumption 4 $\succeq$ Assumption 6. According to the equivalence among these assumptions, we have Assumption 4 $\Leftrightarrow$ Assumption 5 $\succeq$Assumption 3 $\Leftrightarrow$ Assumption 6.
\end{proof}
\subsection{Relation Between Assumption 7 and Assumption 8}
Due to unbiased estimation, it holds that 
\begin{align}
\underbrace{\vphantom{\left(\frac{a^{0.2}}{b}\right)}
\mathbb{E}\left[\left\lVert \nabla F_m(x\vert \xi)\right\rVert^2\right]}_{\text{Assumption 8}}
=
\underbrace{\vphantom{\left(\frac{a^{0.2}}{b}\right)} 
\left\lVert \mathbb{E}\left[\nabla F_m(x)\right]\right\rVert^2}_{\text{Assumption 7}}
+
\underbrace{\vphantom{\left(\frac{a^{0.2}}{b}\right)}
\mathbb{E}\left[\left\lVert \nabla F_m(x\vert \xi)-\nabla F_m(x) \right\rVert^2\right]}_{\text{Assumption 2}}. 
\end{align}
If the left-hand side is bounded, then both terms on the right-hand side must also be bounded, and vice versa. Therefore, Assumption 7, combined with Assumption 2, is equivalent to Assumption 8, i.e.,Assumption 7 $\cap$ Assumption 2 $\Leftrightarrow$ Assumption 8.

\section{Additional Details on the Ablation Studies}\label{Appx: ablation}
\subsection{Performance Comparison in Static Heterogeneous FL}\label{appx: static_heter_fl}
\subsubsection{Test Accuracy Comparison in Static Heterogeneous FL} In Figure \ref{fig:static_exp}, the results demonstrate that \texttt{FedACS} consistently outperforms \texttt{FedAvg}, \texttt{c-a-FedAvg}, \texttt{FedVarp}, \texttt{FedNova}, and OS achieving improvements in test accuracy of $11\%$-$34\%$ on MNIST, $4\%$-$12\%$ on CIFAR-10, and $4\%$-$14\%$ on CINIC-10, respectively. This advantage stems from its ability to address both statistical and structural objective inconsistencies, enabling it to handle heterogeneous communication and computation effectively. As a result, it can converge to the correct global optimum, unlike other benchmark methods. The \texttt{FedAU} is unstable, the accumulation of large gradients after consecutive transmission failures destroys its stability and convergence. 
\begin{figure}[!htb]
    \centering
    \begin{minipage}[b]{0.32\textwidth}
    \centering
    \includegraphics[width=1\linewidth]{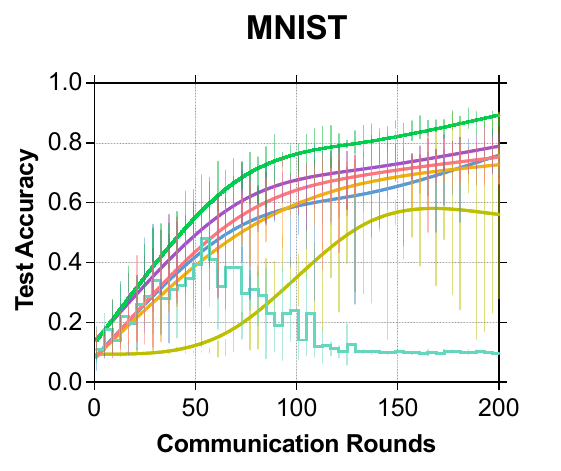}
    \end{minipage}
    \hfill
    \begin{minipage}[b]{0.32\textwidth}
    \centering
    \includegraphics[width=1\linewidth]{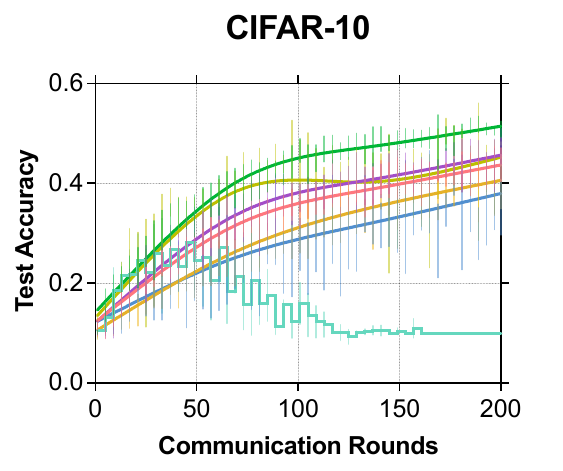}
    \end{minipage}
       \hfill
    \begin{minipage}[b]{0.32\textwidth}
    \centering
    \includegraphics[width=1\linewidth]{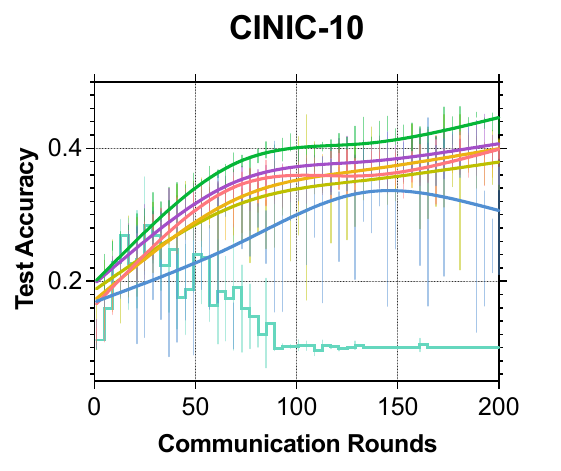}
    \end{minipage}
    \vspace{-0mm}
    \begin{minipage}[b]{1\textwidth}
    \centering
    \vspace{0mm}
    \includegraphics[width=0.9\linewidth]{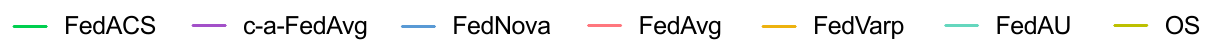}
    \end{minipage}
    \vspace{-5mm}
    \caption{\small \textbf{Test Accuracy of algorithms over various datasets in static heterogeneous FL.} The average test accuracy over multiple runs is smoothed and plotted for each algorithm. Vertical bars indicate the variance of raw data at the corresponding communication round, reflecting the stability of each algorithm. Vertical bars indicate the variance of the distance to the optimum. 
    } \label{fig:static_exp}
    \vspace{-2mm}
\end{figure}

\subsubsection{Efficiency Comparison in Static Heterogeneous FL.}\label{appx: eff_static}
As shown in Table~\ref{table: efficiency_static}, \texttt{FedACS} consistently reaches target accuracy thresholds in fewer communication rounds compared to \texttt{FedAvg}, \texttt{c-a-FedAvg}, \texttt{FedVarp}, and \texttt{FedNova}, achieving reductions of $7\%$-$61\%$ on MNIST, $64\%$-$118\%$ on CIFAR-10, and $52\%$-$131\%$ on CINIC-10. In contrast, \texttt{FedAU} and OS exhibit instability in consistently reaching these thresholds. This superior efficiency is attributed to \texttt{FedACS}'s capability to guide the global model toward the correct optimization objective. When all methods are calibrated to use the same effective step size, \texttt{FedACS} converges more accurately toward the true global optimum and thereby faster, whereas the baseline methods converge toward biased stationary points.

Beyond communication efficiency, \texttt{FedACS} also demonstrates substantially lower computation time to reach the same accuracy thresholds. It outperforms \texttt{FedAvg}, \texttt{c-a-FedAvg}, \texttt{FedVarp}, and \texttt{FedNova} with reductions in computation time of $17\%$-$83\%$ on MNIST, $77\%$-$152\%$ on CIFAR-10, and $43\%$-$133\%$ on CINIC-10. This improvement is not solely due to enhanced convergence direction but also results from its adaptive client selection strategy. Specifically, \texttt{FedACS} tends to favor clients with fewer local epochs, provided that those with more local epochs do not suffer from significantly unreliable communication. This implicit preference reduces computational load while preserving convergence accuracy.

To conclude, \texttt{FedACS} achieves communication and computation efficiency through a combination of precise convergence direction and intelligent client selection.

\begin{table}[!ht]
 \centering
 \vspace{-0.4em}
 \caption{\textbf{Training Efficiency of algorithms over various datasets in static heterogeneous FL.} We report the average number of communication rounds and total computation time required to reach the predefined accuracy thresholds for the first time across different datasets.
 } 
 \label{table: efficiency_static}
 \vspace{-0em}
    \resizebox{1.\textwidth}{!}{
  \begin{tabular}{l c c c c c c c c c c}
   \toprule
   \multirow{2}{*}{Algorithm} & \multicolumn{2}{c}{MNIST }& \multicolumn{2}{c}{CIFAR-10} & \multicolumn{2}{c}{CINIC-10} \\
   \cmidrule(lr){2-3} \cmidrule(lr){4-5}  \cmidrule(lr){6-7} 
            & Rounds for 70\% & Time (s) for 70\% 
            & Rounds for 40\% & Time (s) for 40\%
            & Rounds for 35\% & Time (s) for 35\%     \\
   \midrule
  \rowcolor{lightgray} \texttt{FedACS}     & \textbf{64\small {\transparent{0.5} (1.0$\times$)}} &\textbf{4307.7\small {\transparent{0.5} (1.0$\times$)}}  & \textbf{68 \small {\transparent{0.5} (1.0$\times$)}}& \textbf{4706.3 \small {\transparent{0.5} (1.0$\times$)}} & \textbf{36}\small {\transparent{0.5} (1.0$\times$)} & \textbf{12837.2\small {\transparent{0.5} (1.0$\times$)}}\\
   \texttt{FedAvg}     & 92\small {\transparent{0.5} (1.44$\times$)}  & 6458.7 \small {\transparent{0.5} (1.50$\times$)} &  112 \small {\transparent{0.5} (1.64$\times$)} &8609.0 \small {\transparent{0.5} (1.83$\times$)} & 57 \small {\transparent{0.5} (1.58$\times$)} &19999.3 \small {\transparent{0.5} (1.56$\times$)} \\
    \texttt{c-a-FedAvg}    &68 \small {\transparent{0.5} (1.07$\times$)}  & 5039.2 \small {\transparent{0.5} (1.17$\times$)} &  101\small {\transparent{0.5} (1.49$\times$)} & 8221.3 \small {\transparent{0.5} (1.77$\times$)} &55 \small {\transparent{0.5} (1.52$\times$)} & 18334.6 \small {\transparent{0.5} (1.43$\times$)}\\ 
    \texttt{FedVarp}         & 92\small {\transparent{0.5} (1.44$\times$)}  & 6846.9 \small {\transparent{0.5} (1.59$\times$)} & 148 \small {\transparent{0.5} (2.18$\times$)} & 11862.5 \small {\transparent{0.5} (2.52$\times$)} & 73 \small {\transparent{0.5} (2.03$\times$)} &25479.9 \small {\transparent{0.5} (1.98$\times$)}  \\
    \texttt{FedNova}         & 103\small {\transparent{0.5} (1.61$\times$)}  & 7893.4 \small {\transparent{0.5} (1.83$\times$)} &  117 \small {\transparent{0.5} (1.72$\times$)} & 8536.8 \small {\transparent{0.5} (1.81$\times$)} & 83 \small {\transparent{0.5} (2.31$\times$)} &29968.5\small {\transparent{0.5} (2.33$\times$)} \\ 
    \texttt{FedAU}         & --  & -- &  -- & -- & -- & -- \\ 
    OS        & --  & -- &   72\small {\transparent{0.5} (1.06$\times$)} & 6307.7\small {\transparent{0.5} (1.34$\times$)} & 47 \small {\transparent{0.5} (1.31$\times$)} &13953.2 \small {\transparent{0.5} (1.09$\times$)} \\
   \bottomrule
  \end{tabular}
  }
  \vskip.2\baselineskip
\begin{tabular}{p{.95\textwidth}}
\begin{footnotesize}    
A dash (–) indicates that the specified threshold was not reached in any of the runs.
\end{footnotesize}
\end{tabular}
  \vspace{-0.6em}
\end{table}

\subsection{Performance Comparison Under Varying Data Heterogeneity}
\subsubsection{Test Accuracy Under Varying Data Imbalance} In Figure \ref{fig: vary_data}, the results show that \texttt{FedACS} consistently outperforms \texttt{FedAvg}, \texttt{c-a-FedAvg}, \texttt{FedVarp}, \texttt{FedNova}, and OS, achieving improvements in test accuracy of $3.5\%$-$13.4\%$, $4\%$-$14.3\%$, $3\%$-$12\%$ on CIFAR-10 under varying levels of data imbalance, where $\alpha=0.05$, $\alpha=0.2$ and $\alpha=0.35$, respectively. 
\begin{figure}[!htb]
    \centering
    \begin{minipage}[b]{0.32\textwidth}
    \centering
    \includegraphics[width=1\linewidth]{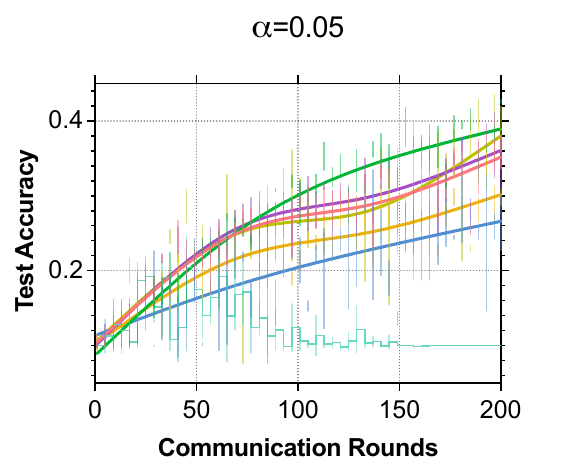}
    \end{minipage}
    \hfill
    \begin{minipage}[b]{0.32\textwidth}
    \centering
    \includegraphics[width=1\linewidth]{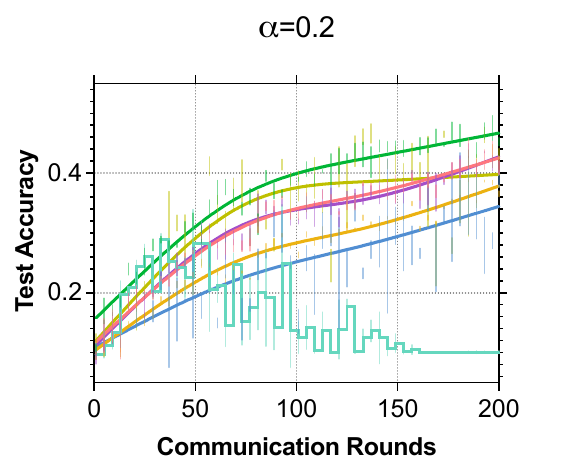}
    \end{minipage}
       \hfill
    \begin{minipage}[b]{0.32\textwidth}
    \centering
    \includegraphics[width=1\linewidth]{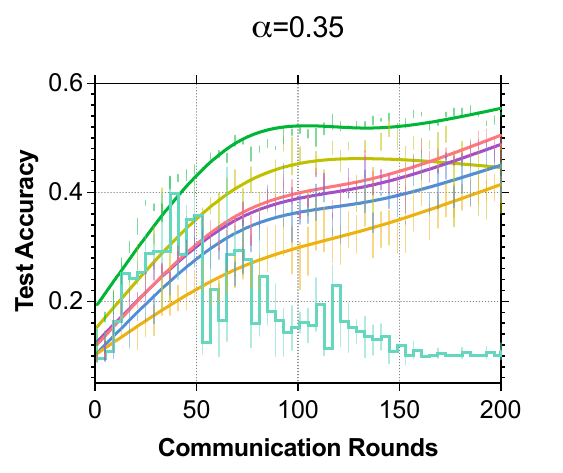}
    \end{minipage}
    \vspace{-0mm}
    \begin{minipage}[b]{1\textwidth}
    \centering
    \vspace{0mm}
    \includegraphics[width=0.9\linewidth]{Fig/Experiments/labels.pdf}
    \end{minipage}
    \vspace{-5mm}
    \caption{\small \textbf{Test accuracy of algorithms on CIFAR-10 under varying levels of data imbalance in dynamic heterogeneous FL.} The average test accuracy over multiple runs is smoothed and plotted for each algorithm. Vertical bars indicate the variance of raw data at the corresponding communication round, reflecting the stability of each algorithm. Vertical bars indicate the variance of the distance to the optimum. 
    } \label{fig: vary_data}
    \vspace{-2mm}
\end{figure}

\subsubsection{Efficiency Comparison  Under Varying Data Imbalance.}
As shown in Table~\ref{table: vary_data}, \texttt{FedACS} consistently reaches target accuracy thresholds in fewer communication rounds compared to \texttt{FedAvg}, \texttt{c-a-FedAvg}, \texttt{FedVarp}, and \texttt{FedNova}, achieving reductions of $33\%$-$39\%$, $72\%$-$206\%$, and $144\%$-$342\%$ on CIFAR-10 under varying levels of data imbalance, where $\alpha=0.05$, $\alpha=0.2$ and $\alpha=0.35$, respectively. 
Beyond communication efficiency, \texttt{FedACS} also demonstrates substantially lower computation time to reach the same accuracy thresholds. It outperforms \texttt{FedAvg}, \texttt{c-a-FedAvg}, \texttt{FedVarp}, and \texttt{FedNova} with reductions in computation time of $45\%$-$44\%$, and $87\%$-$273\%$, $162\%$-$422\%$ on CIFAR-10 under varying levels of data imbalance, where $\alpha=0.05$, $\alpha=0.2$ and $\alpha=0.35$, respectively.

\textbf{In summary}, \texttt{FedACS} substantially outperforms state-of-the-art methods in terms of test accuracy, communication cost, and computational efficiency across varying degrees of data imbalance.

\subsection{Performance Comparison Under Varying Sampling Ratio}
\subsubsection{Test Accuracy under Varying Client Sampling Ratio.} In Figure \ref{fig: vary_frac}, the results show that \texttt{FedACS} consistently outperforms \texttt{FedAvg}, \texttt{c-a-FedAvg}, \texttt{FedVarp}, \texttt{FedNova}, and OS, achieving improvements in test accuracy of $6\%$-$13.8\%$, $2\%$-$12\%$, $6\%$-$14\%$ on CIFAR-10 under varying client Sampling Ratio, where $\mathrm{frac}=0.3$, $\mathrm{frac}=0.5$, and $\mathrm{frac}=0.4$, respectively. 
\begin{figure}[!htb]
    \centering
    \begin{minipage}[b]{0.32\textwidth}
    \centering
    \includegraphics[width=1\linewidth]{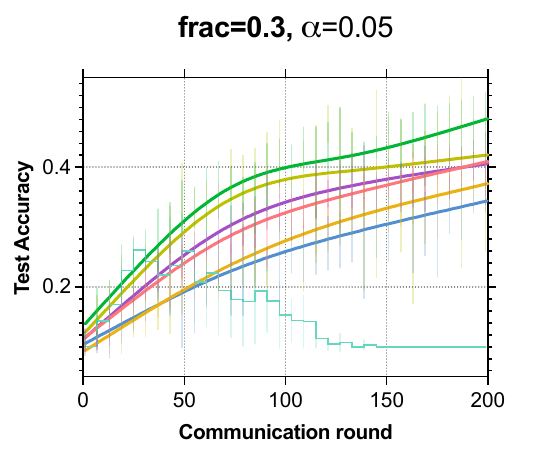}
    \end{minipage}
    \hfill
    \begin{minipage}[b]{0.32\textwidth}
    \centering
    \includegraphics[width=1\linewidth]{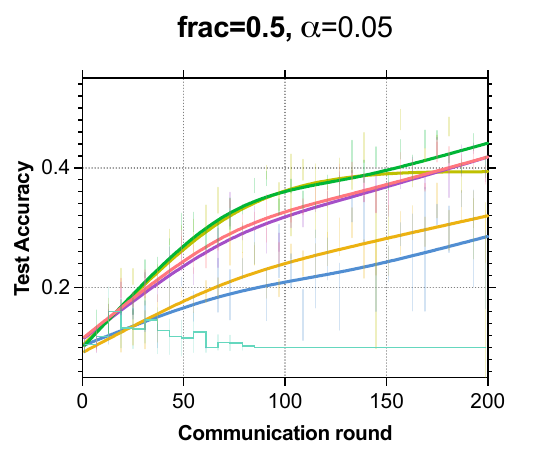}
    \end{minipage}
       \hfill
    \begin{minipage}[b]{0.32\textwidth}
    \centering
    \includegraphics[width=1\linewidth]{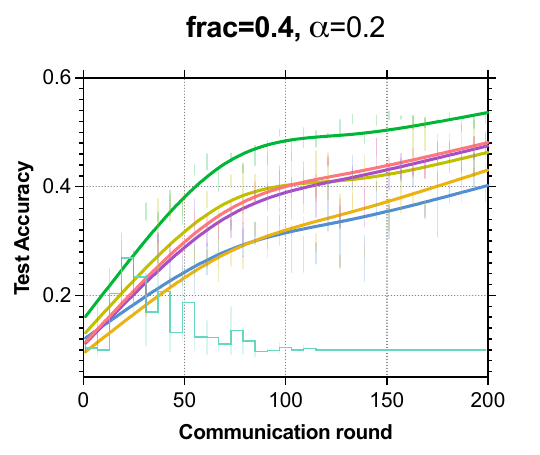}
    \end{minipage}
    \vspace{-0mm}
    \begin{minipage}[b]{1\textwidth}
    \centering
    \vspace{0mm}
    \includegraphics[width=0.9\linewidth]{Fig/Experiments/labels.pdf}
    \end{minipage}
    \vspace{-5mm}
    \caption{\small \textbf{Test accuracy of algorithms on CIFAR-10 under varying participation rates in dynamic heterogeneous FL.} The average test accuracy over multiple runs is smoothed and plotted for each algorithm. Vertical bars indicate the variance of raw data at the corresponding communication round, reflecting the stability of each algorithm. Vertical bars indicate the variance of the distance to the optimum. 
    } \label{fig: vary_frac}
    \vspace{-2mm}
\end{figure}

\subsubsection{Efficiency Comparison Varying Client Sampling Ratio.}
As shown in Table~\ref{table: vary_frac}, \texttt{FedACS} consistently reaches target accuracy thresholds in fewer communication rounds compared to \texttt{FedAvg}, \texttt{c-a-FedAvg}, \texttt{FedVarp}, and \texttt{FedNova}, achieving reductions of $33\%$-$39\%$, $29\%$-$34\%$, and $100\%$-$220\%$ 
on CIFAR-10 under varying client Sampling Ratio, where $\mathrm{frac}=0.3$, $\mathrm{frac}=0.5$, and $\mathrm{frac}=0.4$, respectively. 
Beyond communication efficiency, \texttt{FedACS} also demonstrates substantially lower computation time to reach the same accuracy thresholds. It outperforms \texttt{FedAvg}, \texttt{c-a-FedAvg}, \texttt{FedVarp}, and \texttt{FedNova} with reductions in computation time of $44\%$-$45\%$, and $89\%$-$91\%$, $118\%$-$256\%$ on CIFAR-10 under varying client Sampling Ratio, where $\mathrm{frac}=0.3$, $\mathrm{frac}=0.5$, and $\mathrm{frac}=0.4$, respectively.

\textbf{In summary}, \texttt{FedACS} substantially outperforms state-of-the-art methods in terms of test accuracy, communication cost, and computational efficiency across varying client participation rates.

\section{Universality of Consistent Communication-Computation Co-Design}\label{Appx: comm-comp universal}
This section presents generally adopted participation schemes and their associated anonymous aggregation rules. Based on these, we verify the universality of the derived Communication-Computation Co-Design.  
\subsection{Anonymous  Aggregation Rules with Unreliable Communication}
In general, anonymous aggregation rules cannot inherently accommodate asymmetric networks. An optimal aggregation rule should, at a minimum, avoid introducing additional bias, i.e., it should remain unbiased in symmetric networks.
\subsubsection{Importance Sampling (IS)}
In IS, clients $m\in \mathcal{S}_r(\boldsymbol{p})$ are sampled with replacement according to the learning weights, i.e., $\boldsymbol{p}=[\omega_1, \cdots, \omega_M]$, the anonymous  aggregation rule is accordingly given by 
\begin{align}
    \boldsymbol{X}_{r+1}=\boldsymbol{X}_{r} + \frac{1}{K} \sum_{m\in \Tilde{\mathcal{S}}_r(\boldsymbol{p},\boldsymbol{q})} (-\eta\boldsymbol{\Delta}_{m}^{r,\boldsymbol{\xi}}).
    \label{eq: agg_rule_IS}
\end{align}
\subsubsection{Uniform sampling (US)}
In US, the clients $m\in \mathcal{S}_r(\boldsymbol{p})$ are sampled uniformly without replacement, the anonymous  aggregation rule is accordingly given by
\begin{align}
    \boldsymbol{X}_{r+1}=\boldsymbol{X}_{r} + \frac{K}{\lvert \Tilde{\mathcal{S}}_r(\boldsymbol{p},\boldsymbol{q}) \rvert}\sum_{m\in \Tilde{\mathcal{S}}_r(\boldsymbol{p},\boldsymbol{q})} \omega_m (-\eta\boldsymbol{\Delta}_{m}^{r,\boldsymbol{\xi}}).
    \label{eq: agg_rule_US}
\end{align}
\subsubsection{Optimal sampling (OS)}
In OS, clients $m\in \mathcal{S}_r(\boldsymbol{p})$ are sampled with replacement according to sampling probability $\boldsymbol{p}$ proportional to the norm of gradients. The anonymous  aggregation rule is accordingly given by 
\begin{align}
    \boldsymbol{X}_{r+1}=\boldsymbol{X}_{r} + \frac{1}{K} \sum_{m\in \Tilde{\mathcal{S}}_r(\boldsymbol{p},\boldsymbol{q})} \frac{\omega_m}{p_m}(-\eta\boldsymbol{\Delta}_{m}^{r,\boldsymbol{\xi}}).
    \label{eq: agg_rule_OS}
\end{align}
\subsubsection{Full Participation} 
Under full participation of clients, the anonymous  aggregation rule is accordingly given by 
\begin{align}
    \boldsymbol{X}_{r+1}=\boldsymbol{X}_{r} + \sum_{m\in \Tilde{\mathcal{S}}_r(\boldsymbol{p},\boldsymbol{q})} \omega_m(-\eta\boldsymbol{\Delta}_{m}^{r,\boldsymbol{\xi}}).
    \label{eq: agg_rule_full}
\end{align}
The anonymous aggregation rules in \eqref{eq: agg_rule_IS}--\eqref{eq: agg_rule_full} are unbiased in homogeneous networks, as can be verified by substituting $q_m=q$ and $T_m=T$ into \eqref{eq:surro_deriva}--\eqref{eq:surro_fp}.
\subsection{Surrogate Functions}\label{Appx: Surrogate Functions}
This section derives the surrogate functions optimized under different participation schemes. 
\subsubsection{Expectation of Aggregation in IS}
When IS is adopted, the expectation of the server aggregation in \eqref{eq: agg_rule_IS} is given by
\begin{subequations}
\begin{align}
    \mathbb{E}\left[ \frac{1}{K} \sum_{m\in \Tilde{\mathcal{S}}_r(\boldsymbol{p},\boldsymbol{q})} (-\eta\boldsymbol{\Delta}_{m}^{r,\boldsymbol{\xi}}) \right]
    &= \mathbb{E}_{\boldsymbol{p}}\left[ \frac{1}{K} \mathbb{E}_{\boldsymbol{q}}\left[ \sum_{m\in \Tilde{\mathcal{S}}_r(\boldsymbol{p},\boldsymbol{q})} (-\eta\boldsymbol{\Delta}_{m}^{r,\boldsymbol{\xi}}) Z_m^r\right] \right]\\
    &= \mathbb{E}_{\boldsymbol{p}}\left[ \frac{1}{K}  \sum_{m\in \mathcal{S}_r(\boldsymbol{p})} (-\eta\boldsymbol{\Delta}_{m}^{r,\boldsymbol{\xi}}) \mathbb{E}_{q_m}\left[Z_m^r\right] \right]\\
    &=-\eta K \cdot \frac{1}{K}  \sum_{m=1}^M \omega_m (1-q_m) \boldsymbol{\Delta}_{m}^{r,\boldsymbol{\xi}}\\
    &=-\underbrace{\eta \sum_{m=1}^M \omega_m (1-q_m)}_{\Tilde{\eta}_{\mathrm{eff}}}  \sum_{m=1}^M \underbrace{\frac{\omega_m (1-q_m)}{\sum_{m=1}^M \omega_m (1-q_m)}}_{\Tilde{\gamma}_m} \boldsymbol{\Delta}_{m}^{r,\boldsymbol{\xi}}\\
    &=-\Tilde{\eta}_{\mathrm{eff}} \cdot \underbrace{\sum_{m=1}^M \Tilde{\gamma}_m \lVert \boldsymbol{a}_m\rVert_1}_{\Tilde{T}_{\mathrm{eff}}} \cdot \sum_{m=1}^M \underbrace{\frac{\Tilde{\gamma}_m \lVert \boldsymbol{a}_m\rVert_1}{\sum_{m=1}^M \Tilde{\gamma}_m \lVert \boldsymbol{a}_m\rVert_1 }}_{\Tilde{\Omega}_m} \cdot\underbrace{\frac{\nabla\boldsymbol{F}_m^{r,\boldsymbol{\xi}}\cdot\boldsymbol{a}_m}{\lVert \boldsymbol{a}_m\rVert_1}}_{\nabla\Bar{\boldsymbol{F}}_m^{r,\boldsymbol{\xi}}}
\end{align}\label{eq:surro_deriva} 
\end{subequations}
In the derivation, we utilize that client sampling is independent of unreliable connectivity. 
\subsubsection{Expectation of Aggregation in US}
When US is adopted without replacement, the expectation of the server aggregation in \eqref{eq: agg_rule_US} is given by
\begin{subequations}
\begin{align}
   \mathbb{E}\left[ \frac{M}{\lvert \mathcal{S}_r(\boldsymbol{p}) \rvert}\sum_{m\in \Tilde{\mathcal{S}}_r(\boldsymbol{p},\boldsymbol{q})} \omega_m (-\eta\boldsymbol{\Delta}_{m}^{r,\boldsymbol{\xi}}) \right]
   &= -\eta \frac{M}{\lvert \mathcal{S}_r(\boldsymbol{p}) \rvert} \sum_{m=1}^M \frac{\lvert \mathcal{S}_r(\boldsymbol{p}) \rvert}{M}  \omega_m (1-q_m) \boldsymbol{\Delta}_{m}^{r,\boldsymbol{\xi}}\\
   &=-\Tilde{\eta}_{\mathrm{eff}}\Tilde{T}_{\mathrm{eff}} \sum_{m=1}^M \Tilde{\Omega}_m \nabla\Bar{\boldsymbol{F}}_m^{r,\boldsymbol{\xi}}
\end{align}\label{eq:surro_us}
\end{subequations}
In the derivation, we utilize that in the US without replacement, the probability of each client being sampled is $\frac{\lvert \mathcal{S}_r(\boldsymbol{p}) \rvert}{M}$ when sampling $\lvert \mathcal{S}_r(\boldsymbol{p}) \rvert$ clients.
\subsubsection{Expectation of Aggregation in OS}
When OS or its variational algorithms are adopted, the expectation of the server aggregation in \eqref{eq: agg_rule_OS} is given by
\begin{subequations}
\begin{align}
    \mathbb{E}\left[ \frac{1}{K} \sum_{m\in \Tilde{\mathcal{S}}_r(\boldsymbol{p},\boldsymbol{q})} \frac{\omega_m}{p_m}(-\eta\boldsymbol{\Delta}_{m}^{r,\boldsymbol{\xi}}) \right]
    &= \mathbb{E}_{\boldsymbol{p}}\left[ \frac{1}{K}  \sum_{m\in \mathcal{S}_r(\boldsymbol{p})} \frac{\omega_m}{p_m}(-\eta\boldsymbol{\Delta}_{m}^{r,\boldsymbol{\xi}}) \mathbb{E}_{q_m}\left[Z_m^r\right] \right]\\
    &=-\eta K\cdot \frac{1}{K} \sum_{m=1}^M  p_m\cdot  \frac{\omega_m}{p_m} (1-q_m)  \boldsymbol{\Delta}_{m}^{r,\boldsymbol{\xi}}\\
    &=-\Tilde{\eta}_{\mathrm{eff}}\Tilde{T}_{\mathrm{eff}} \sum_{m=1}^M \Tilde{\Omega}_m \nabla\Bar{\boldsymbol{F}}_m^{r,\boldsymbol{\xi}} 
\end{align}\label{eq:surro_os}   
\end{subequations}
In the derivation, we utilize that client sampling is independent of unreliable connectivity.
\subsubsection{Expectation of Aggregation in Full Participation} 
\begin{subequations}
\begin{align}
     \mathbb{E}\left[\sum_{m\in \Tilde{\mathcal{S}}_r(\boldsymbol{p},\boldsymbol{q})} \omega_m(-\eta\boldsymbol{\Delta}_{m}^{r,\boldsymbol{\xi}})\right]
     &=-\eta\sum_{m=1}^M \omega_m(1-q_m) \boldsymbol{\Delta}_{m}^{r,\boldsymbol{\xi}}\\
     &=-\Tilde{\eta}_{\mathrm{eff}}\Tilde{T}_{\mathrm{eff}} \sum_{m=1}^M \Tilde{\Omega}_m \nabla\Bar{\boldsymbol{F}}_m^{r,\boldsymbol{\xi}}
\end{align}\label{eq:surro_fp}
\end{subequations}

From \eqref{eq:surro_deriva}--\eqref{eq:surro_fp}, the surrogate function in IS, US, OS, and full participation are identical and is
\begin{align}
    \Bar{F}(\boldsymbol{X})= \sum_{m=1}^M \Tilde{\Omega}_m F_m(\boldsymbol{X}).  
\end{align} 
Substituting \eqref{eq: comm-comp design} yields $\tilde{\Omega}_m=\omega_m$, which completes the proof of universality of the proposed consistent communication-computation co-design. 

\end{appendices}

\end{document}